\pdfoutput=1
\documentclass[11pt]{article}
\def\arXiv{1} 

\usepackage[top=1in, right=1in, left=1in, bottom=1in]{geometry}

\usepackage[numbers,square]{natbib}

\usepackage{amsthm}
\usepackage[hidelinks]{hyperref}

\usepackage{thmtools}
\usepackage{thm-restate}

\theoremstyle{plain}
\newtheorem{theorem}{Theorem}
\newtheorem*{theorem*}{Theorem}
\newtheorem{lemma}{Lemma}
\newtheorem{proposition}{Proposition}
\newtheorem*{lemma*}{Lemma}
\newtheorem{corollary}[theorem]{Corollary}

\newtheorem{assumption}{Assumption}

\usepackage{enumitem}

\usepackage{url}            %
\usepackage{booktabs}       %
\usepackage{amsfonts}       %
\usepackage{nicefrac}       %
\usepackage{microtype}      %
\usepackage{xcolor}
\usepackage{tabularx}
\usepackage[noend]{algpseudocode}
\usepackage{stackengine}
\usepackage{comment}
\usepackage{caption}
\usepackage{titlesec}

\usepackage{subcaption}
\newcommand{\notarxiv}[1]{foo}
\newcommand{\arxiv}[1]{bar}
\ifdefined \arXiv
\renewcommand{\arxiv}[1]{#1}%
\renewcommand{\notarxiv}[1]{\ignorespaces}%
\else%
\renewcommand{\arxiv}[1]{\ignorespaces}%
\renewcommand{\notarxiv}[1]{#1}%
\fi%

\usepackage{enumitem}
\usepackage{amssymb}
\usepackage{amsbsy}
\usepackage{amsmath}
\usepackage{amsthm}
\usepackage{amsfonts}
\usepackage{latexsym}
\usepackage{graphicx}
\usepackage{color}
\usepackage{xifthen}
\usepackage{xspace}
\usepackage{mathtools}

\usepackage{xparse}

\usepackage{algorithm}
\usepackage[noend]{algpseudocode}
\floatname{algorithm}{Meta-Algorithm}

\usepackage{multirow}
\usepackage[normalem]{ulem}
\usepackage{accents}

\usepackage{array}
\usepackage{ragged2e}

\usepackage[T1]{fontenc}
\usepackage{thmtools}

\usepackage{cleveref}
\Crefname{algorithm}{Meta-Algorithm}{Program-codes}

\usepackage[textsize=scriptsize,textwidth=3cm]{todonotes}

\usepackage{etoolbox}
\newtoggle{restatements}

\newtoggle{heavyplots}

\usepackage{xfrac}

\newcommand{\wt}[1]{\widetilde{#1}}  %
\newcommand{\norm}[1]{\left\|{#1}\right\|} %
\newcommand{\lone}[1]{\norm{#1}_1} %
\newcommand{\norms}[1]{\|{#1}\|} %
\newcommand{\R}{\mathbb{R}} %
\newcommand{\E}{\mathbb{E}} %
\renewcommand{\P}{\mathbb{P}}	%
\providecommand{\argmax}{\mathop{\rm argmax}} %

\providecommand{\sign}{\mathop{\rm sign}}

\newcommand{\half}{\frac{1}{2}}

\newcommand{\defeq}{\coloneqq}
\newcommand{\eqdef}{\eqqcolon}

\newcommand{\hess}{\nabla^2}

\newcommand{\eps}{\varepsilon}
\newcommand{\supind}[1]{^{\left({#1}\right)}}

\newcommand{\algo}{\mathsf{A}} \newcommand{\wun}{w} \newcommand{\Dkl}[2]{D_{\text{KL}}(#1\parallel#2)}

\newcommand{\ballp}{\mathcal{B}_\epsilon^p}
\newcommand{\balltwo}{\mathcal{B}_\epsilon^2}

\newcommand{\xadv}{x'}

\newcommand{\distrib}{P_{\mathsf{x,y}}}
\newcommand{\distribx}{P_{\mathsf{x}}}
\newcommand{\nlab}{n}
\newcommand{\nunlab}{\tilde{n}}
\newcommand{\intest}{\hat{\theta}_{\textup{intermediate}}}
\newcommand{\estim}{\hat{\theta}}
\newcommand{\finalest}{\estim_{\textup{final}}}
\newcommand{\labest}{\intest}
\newcommand{\semisupest}{\finalest}
\newcommand{\xun}{\tilde{x}} 
\newcommand{\Xun}{\tilde{X}}
\newcommand{\yun}{\tilde{y}}
\newcommand{\Yun}{\tilde{Y}}

\newcommand{\nn}{n_0}

\newcommand{\learner}{\algo}

\newcommand{\err}{\mathsf{err}_{\textup{standard}}}
\NewDocumentCommand{\errrob}{ O{\infty} O{\epsilon} 
}{\mathsf{err}_{\textup{robust}}^{#1,#2}}
\newcommand{\loss}{L_{\textup{standard}}}
\newcommand{\lossrob}{L_{\textup{robust}}}
\newcommand{\lossreg}{L_{\textup{reg}}}

\newcommand{\pgd}[3]{\texttt{PG}(#1, #2, #3)}
\newcommand{\pgdno}{\texttt{PG}}
\newcommand{\cw}{\texttt{CW}}
\newcommand{\rst}{\texttt{RST}}
\newcommand{\rstgen}[2]{\texttt{RST}_\texttt{#1}(\texttt{#2})}
\newcommand{\rststno}{\rst_\texttt{stab}}
\newcommand{\rstatno}{\rst_\texttt{adv}}
\newcommand{\rstat}[1]{\rst_\texttt{adv}(\texttt{#1})}
\newcommand{\rstst}[1]{\rst_\texttt{stab}(\texttt{#1})}
\newcommand{\sst}[1]{\texttt{SST}(\texttt{#1})}
\newcommand{\baseline}[2]{\texttt{Baseline}_\texttt{#1}(\texttt{#2})}
\newcommand{\pgdours}{\texttt{PG}_\texttt{Ours}}
\newcommand{\pgdtrades}{\texttt{PG}_\texttt{TRADES}}
\newcommand{\pgdmadry}{\texttt{PG}_\texttt{Madry}}

\newcommand{\wrn}[2]{\text{Wide ResNet #1-#2}}
\newcommand{\semisup}[2]{\texttt{rVAT}_{\texttt{#1}}(\texttt{#2})}
\newcommand{\semisupno}{\texttt{rVAT}}

\newcommand{\dev}[1]{{\footnotesize ~$\pm$ #1}}

\newcommand{\yair}[1]{{ \color{blue} Yair: #1}}
\newcommand{\aditi}[1]{{\bf \color{orange} Aditi: #1}}
\newcommand{\ludwig}[1]{\todo[color=green!40]{Ludwig: #1}}
\newcommand{\yairside}[1]{\todo[color=blue!10]{Yair: #1}}
\newcommand{\aditiside}[1]{\todo[color=orange!20]{Aditi: #1}}
\newcommand{\jcd}[1]{{\bf \color{blue} JCD: #1}}
\newcommand{\jcdcomment}[1]{\todo[color=yellow!10]{JCD: #1}}
\newcommand{\pl}[1]{\textcolor{red}{[PL: #1]}}
\newcommand{\plside}[1]{\todo[color=red!20]{Percy: #1}}

\renewcommand{\yair}[1]{}
\renewcommand{\aditi}[1]{}
\renewcommand{\ludwig}[1]{}
\renewcommand{\yairside}[1]{}
\renewcommand{\aditiside}[1]{}
\renewcommand{\jcd}[1]{}
\renewcommand{\jcdcomment}[1]{}
\renewcommand{\pl}[1]{}
\renewcommand{\plside}[1]{}

\usepackage[normalem]{ulem}

\newcommand{\scrcmd}{\textsuperscript}
\newcommand{\scr}[1]{\scrcmd{#1}}

\newcommand\sN{\ensuremath{\mathcal{N}}}

\newcommand\sX{\ensuremath{\mathcal{X}}}
\newcommand\sY{\ensuremath{\mathcal{Y}}}

\newcommand\refeqn[1]{(\ref{eqn:#1})}

\ifthenelse{\isundefined{\definition}}{\newtheorem{definition}{Definition}}{}
\ifthenelse{\isundefined{\assumption}}{\newtheorem{assumption}{Assumption}}{}
\ifthenelse{\isundefined{\hypothesis}}{\newtheorem{hypothesis}{Hypothesis}}{}
\ifthenelse{\isundefined{\proposition}}{\newtheorem{proposition}{Proposition}}{}
\ifthenelse{\isundefined{\theorem}}{\newtheorem{theorem}{Theorem}}{}
\ifthenelse{\isundefined{\lemma}}{\newtheorem{lemma}{Lemma}}{}
\ifthenelse{\isundefined{\corollary}}{\newtheorem{corollary}{Corollary}}{}
\ifthenelse{\isundefined{\alg}}{\newtheorem{alg}{Algorithm}}{}
\ifthenelse{\isundefined{\example}}{\newtheorem{example}{Example}}{}
\titlespacing*{\paragraph}{0pt}{0.35\baselineskip}{1em}

\title{Unlabeled Data Improves Adversarial Robustness}

\makeatletter
\newcommand{\printfnsymbol}[1]{%
	\textsuperscript{\@fnsymbol{#1}}%
}
\makeatother

\newcommand\blfootnote[1]{%
	\begingroup
	\renewcommand\thefootnote{}\footnote{#1}%
	\addtocounter{footnote}{-1}%
	\endgroup
}

\author
{
	  Yair Carmon\thanks{ ~Equal contribution.}\\
	  Stanford University\\
	  \texttt{yairc@stanford.edu}\\
	  \and
	  Aditi Raghunathan\printfnsymbol{1}\\
	  Stanford University\\
	  \texttt{aditir@stanford.edu}\\
	  \and
	  Ludwig Schmidt\\
	  UC Berkeley\\
	  \texttt{ludwig@berkeley.edu}\\
	  \and
	  Percy Liang\\
	  Stanford University\\
	  \texttt{pliang@cs.stanford.edu}\\
	  \and
	  John C.\ Duchi\\
	  Stanford University\\
	  \texttt{jduchi@stanford.edu}\\
}
\date{}

\begin{document}

\maketitle

\maketitle

\blfootnote{ Code and data are available on GitHub at 
\url{https://github.com/yaircarmon/semisup-adv} and on CodaLab at
	\url{https://bit.ly/349WsAC}.}

\begin{abstract}
We demonstrate, theoretically and empirically, that adversarial robustness 
can significantly benefit from semisupervised learning.  Theoretically, we 
revisit the simple Gaussian model of \citet{schmidt2018adversarially} that 
shows a sample complexity gap between standard and robust 
classification.  We prove that unlabeled data bridges this gap: a simple 
semisupervised learning procedure (self-training) achieves high robust 
accuracy using the same number of labels required for achieving high 
standard accuracy.
Empirically, we augment CIFAR-10 with 500K unlabeled images sourced 
from 80 Million Tiny Images and use robust self-training to outperform 
state-of-the-art robust accuracies by over 5 points in (i) $\ell_\infty$ 
robustness against several strong attacks via adversarial training and (ii) 
certified $\ell_2$ and $\ell_\infty$ robustness via randomized smoothing.  
On SVHN, adding the dataset's own extra training set with the labels 
removed provides gains of 4 to 10 points, within 1 point of the gain from 
using the extra labels. 
\end{abstract}
\section{Introduction}

The past few years have seen an intense research interest in making models
robust to adversarial examples \citep{szegedy2014intriguing, 
biggio2013evasion,biggio2018wild}.  
Yet despite a wide range of proposed
defenses, the state-of-the-art in adversarial robustness is far from satisfactory. Recent work points towards sample complexity as a
possible reason for the small gains in robustness: \citet{schmidt2018adversarially}
show that in a simple model, learning a classifier with non-trivial adversarially robust
accuracy requires substantially more samples than achieving good ``standard''
accuracy.
Furthermore, recent
empirical work obtains promising gains in robustness via transfer learning
of a robust classifier from a larger labeled  dataset~\citep{hendrycks2019pretraining}.  While both theory and
experiments suggest that more training data leads to greater robustness,
following this suggestion can be difficult due to the cost of gathering
additional data and especially obtaining high-quality labels.

To alleviate the need for carefully labeled
data, in this paper we study adversarial robustness through the lens of
 semisupervised learning. 
Our approach is motivated by two basic
observations.  First, adversarial robustness essentially asks that
predictors be stable around naturally occurring inputs.  Learning to satisfy
such a stability constraint should not inherently require labels.  Second, the
added requirement of robustness fundamentally alters the regime where
semi-supervision is useful.  
Prior work on semisupervised learning mostly focuses
on improving the standard accuracy by leveraging unlabeled data. 
However, in
our adversarial setting the labeled data alone already produce accurate (but
not robust) classifiers.  We can use such classifiers on the unlabeled data
and obtain useful \emph{pseudo-labels}, which directly suggests the use of
\emph{self-training}---one of the oldest frameworks for
semisupervised 
learning~\citep{scudder1965probability,chapelle2006semisupervised}, 
which applies a supervised training method on the 
pseudo-labeled data. 
We provide theoretical and experimental evidence that self-training is effective for  adversarial robustness.

The first part of our paper is theoretical and considers the simple 
$d$-dimensional Gaussian model \citep{schmidt2018adversarially} with 
$\ell_\infty$-perturbations of magnitude $\epsilon$.
We scale the model so that $\nn$ labeled examples allow for learning a classifier with nontrivial standard accuracy, and roughly $\nn\cdot  \epsilon^2\sqrt{d/\nn}$ examples are necessary for attaining any nontrivial robust accuracy. This implies a sample complexity gap in the high-dimensional regime $d \gg \nn  \epsilon^{-4}$. In this regime, we prove that self training with $O(\nn\cdot  \epsilon^2\sqrt{d/\nn})$ unlabeled data and just $\nn$ labels achieves high robust accuracy.
Our analysis provides a refined perspective on the sample complexity barrier in this model: the increased sample requirement is exclusively on unlabeled data.

Our theoretical findings motivate the second, empirical part of our paper, 
where we test the effect of unlabeled data 
and self-training on standard adversarial robustness benchmarks. We 
propose and experiment with robust self-training (RST), a natural 
extension of self-training that uses standard supervised training to obtain 
pseudo-labels and then feeds the pseudo-labeled data into a supervised 
training algorithm that targets adversarial robustness.
We use TRADES \citep{zhang2019theoretically} for \emph{heuristic} $\ell_\infty$-robustness, and stability training \citep{zheng2016improving} combined with randomized smoothing \citep{cohen2019certified} for \emph{certified} $\ell_2$-robustness.

For CIFAR-10~\cite{krizhevsky2009learningmultiple}, we obtain 500K unlabeled images by mining the 80 Million Tiny Images dataset~\citep{torralba2008million} with an image classifier.
Using RST on the CIFAR-10 training set augmented with the additional unlabeled data, we outperform state-of-the-art \emph{heuristic} $\ell_\infty$-robustness against strong iterative attacks by $7\%$.
In terms of \emph{certified} $\ell_2$-robustness, RST outperforms our fully supervised baseline by $5\%$ and beats previous state-of-the-art numbers by $10\%$.
Finally, we also match the state-of-the-art certified $\ell_\infty$-robustness, while improving on the corresponding standard accuracy by over $16\%$. We show that some natural alternatives such as virtual adversarial training~\cite{miyato2018virtual} and aggressive data augmentation do not perform as well as RST. We also study the sensitivity of RST to varying data volume and relevance.

Experiments with SVHN show similar gains in robustness with RST on semisupervised data. 
Here, we apply RST by removing the labels from the 531K extra training data and see $4$--$10\%$ increases in robust accuracies compared to the baseline that only uses the labeled 73K training set. Swapping the pseudo-labels for the true SVHN extra labels increases these accuracies by at most 1\%. This confirms that the majority of the benefit from extra data comes from the inputs and not the labels.

In independent and concurrent work,~\citet{uesato2019are, najafi2019robustness} and~\citet{zhai2019adversarially} also explore semisupervised learning for adversarial robustness. See \Cref{sec:related} for a comparison.

Before proceeding to the details of our theoretical results in Section~\ref{sec:theory}, we briefly introduce relevant background in Section~\ref{sec:setup}.
Sections~\ref{sec:approach} and~\ref{sec:experiments} then describe our 
adversarial self-training approach and provide comprehensive experiments 
on CIFAR-10 and SVHN. We survey related work in Section~\ref{sec:related} 
and conclude in Section~\ref{sec:conclusion}.

\section{Setup}
\label{sec:setup}

\paragraph{Semi-supervised classification task.} 
We consider the task of mapping input 
$x \in \sX\subseteq\R^d$ to label $y \in \sY$. Let $\distrib$ denote the underlying distribution of $(x,y)$ pairs, and let $\distribx$ denote its marginal on $\sX$.
Given training data consisting of (i) labeled examples $(X, Y) = (x_1, y_1), \hdots (x_{\nlab}, 
y_{\nlab}) \sim \distrib$ and (ii) unlabeled examples $\Xun=\xun_1, \xun_2, \hdots 
\xun_{\nunlab} \sim \distribx$,
the goal is to learn a classifier $f_\theta : \sX 
\to \sY$ in a model family parameterized by $\theta \in \Theta$. 

\paragraph{Error metrics.} 

The standard quality metric for classifier $f_\theta$ is its error probability,
\begin{align}
  \label{eqn:standarderror}
  \err(f_\theta) \defeq \P_{(x, y) \sim \distrib} \big(f_\theta(x)\ne y\big).
\end{align}
We also evaluate classifiers on their performance on \emph{adversarially perturbed inputs}. 
In this work, we consider perturbations in a $\ell_p$ norm ball of radius $\epsilon$ around the input, and define the corresponding robust error probability,
\begin{equation}
  \label{eqn:robusterror}
  \errrob[p](f_\theta) \defeq \P_{(x, y) \sim \distrib} \big(\exists \xadv \in \ballp(x), f_\theta(\xadv)\ne y \big)
  ~\mbox{for}~\ballp(x) \defeq \{ \xadv\in\sX \mid \norm{\xadv-x}_p \leq \epsilon \}. 
\end{equation}
In this paper we study $p=2$ and $p=\infty$. 
We say that a classifier $f_\theta$ has \emph{certified} $\ell_p$ accuracy 
$\xi$ when we can \emph{prove} that $\errrob[p](f_\theta) \le 1-\xi$.

\paragraph{Self-training.} 
Consider a supervised learning algorithm $\algo$ that maps a dataset $(X, Y)$ to parameter $\theta$. %
\emph{Self-training} is the straightforward extension of $\algo$ to a semisupervised setting, and consists of the following two steps. 
First, obtain an intermediate model $\intest = \algo(X, Y)$, and use it to generate \emph{pseudo-labels} $\yun_i = f_{\intest}(\xun_i)$ for $i\in[\nunlab]$. 
Second, combine the data and pseudo-labels to obtain a final model $\finalest = \algo( [X, \Xun], [Y, \Yun])$.
\section{Theoretical results}
\label{sec:theory}

In this section, we consider a simple high-dimensional model studied in~\cite{schmidt2018adversarially}, which is the only known formal example of an information-theoretic sample complexity gap between standard and robust classification. For this model, we demonstrate the value of unlabeled data---a simple self-training procedure achieves high robust accuracy, when achieving non-trivial robust accuracy using the labeled data alone is impossible.  

\paragraph{Gaussian model.} We consider a binary classification task where $\sX=\R^d$, $\sY = \{-1,1\}$, $y$ uniform on $\sY$ and $x|y\sim\sN(y \mu , \sigma^2 I)$ for a vector $\mu\in\R^d$ and coordinate noise variance $\sigma^2 > 0$.
We are interested in the standard error~\refeqn{standarderror} and robust error $\errrob$~\refeqn{robusterror} for $\ell_\infty$ perturbations of size $\epsilon$.

\paragraph{Parameter setting.}
We choose the model parameters to meet the following desiderata: (i) there 
exists a classifier that achieves very high robust and standard accuracies, 
(ii) using $\nn$ examples we can learn a classifier with non-trivial standard 
accuracy and (iii) we require much more than $\nn$ examples to learn a 
classifier with nontrivial robust accuracy. As shown 
in~\citep{schmidt2018adversarially}, the following parameter setting meets 
the desiderata,
\begin{equation}\label{eqn:scaling}
\epsilon\in (0,\tfrac{1}{2}),~~\norm{\mu}^2=d~~\text{and}~~
\frac{\norm{\mu}^2}{\sigma^2}= 
\sqrt{\frac{d}{\nn}} \gg \frac{1}{\epsilon^2}.
\end{equation}
When interpreting this setting it is useful to think of $\epsilon$ as fixed
and of $d/\nn$ as a large number, i.e.\ a highly overparameterized
regime.

\subsection{Supervised learning in the Gaussian model}

We briefly recapitulate the sample complexity gap described in~\citep{schmidt2018adversarially} for the fully supervised setting.

\paragraph{Learning a simple linear classifier.}
We consider linear classifiers of the form $f_\theta = \sign(\theta^\top x)$.
Given $\nlab$ labeled data $(x_1, y_1), \ldots, (x_{\nlab}, y_{\nlab})\stackrel{\text{iid}}{\sim}\distrib$, we form the following simple classifier
\begin{align}
  \label{eqn:labeled-estimator}
  \estim_{\nlab} \defeq \frac{1}{\nlab}\sum_{i=1}^{\nlab} y_i x_i.
\end{align}

We achieve nontrivial standard accuracy using $\nn$ examples; see
\Cref{sec:theory-sup} for proof of the following (as well as detailed rates of convergence).
\begin{restatable}{proposition}{restateSupSamples}
  \label{prop:supervised-samples}There exists a universal constant $r$ 
  such that for all $\epsilon^2 \sqrt{d/\nn} \ge r$,
  \[
  n\ge\nn \; \Rightarrow \; \E_{\estim_{n}}\err\left(f_{\estim_{n}}\right)\le\frac{1}{3}
  \;\; \text{ and }\;\;
  n\ge\nn \cdot 4 \epsilon^{2}\sqrt{\frac{d}{\nn}}
  \; \Rightarrow \; \E_{\estim_{n}}\errrob\left(f_{\estim_{n}}\right)\le 10^{-3}.
  \]
\end{restatable}
Moreover, as the following theorem states, no learning algorithm can produce a classifier with nontrivial robust error without observing  $\wt{\Omega}(\nn \cdot \epsilon^2\sqrt{d/\nn})$ examples. Thus, a sample complexity gap forms as $d$ grows.
\begin{restatable}[\citet{schmidt2018adversarially}]{theorem}{restateLowerBound}
  \label{thm:lower-bound}
  Let $\learner_\nlab$ be any learning rule  
  mapping a dataset $S\in (\sX \times \sY)^\nlab$ to classifier 
  $\learner_\nlab[S]$. Then,
  \begin{align}
    \label{eqn:robust}
    \nlab \le \nn \frac{\epsilon^2 \sqrt{d/\nn}}{8 \log d}
    \; \Rightarrow \; 
    \E \,\errrob(\learner_\nlab[S]) \ge \half (1-d^{-1}), 
  \end{align}
  where the expectation is with respect to the random draw of $S\sim 
  \distrib^{\nlab}$ as well as possible randomization in 
  $\learner_{\nlab}$. 
\end{restatable}

\subsection{Semi-supervised learning in the Gaussian model}
We now consider the semisupervised setting with $\nlab$ labeled examples and $\nunlab$ additional unlabeled examples. We apply the self-training  methodology described in \Cref{sec:setup} on the simple learning rule~\eqref{eqn:labeled-estimator}; our intermediate classifier is $\intest\defeq \estim_{\nlab} = \frac{1}{\nlab}\sum_{i=1}^{\nlab} y_i x_i$,  and we generate pseudo-labels $\yun_i \defeq f_{\labest}(\tilde{x}_i) = \sign(\xun_i^\top \labest)$ for $i=1,\ldots,\nunlab$. We then learning rule~\refeqn{labeled-estimator} to obtain our final semisupervised classifier
$\semisupest \defeq \frac{1}{\nunlab}\sum_{i=1}^{\nunlab} \tilde{y}_i \tilde{x}_i$. 
The following theorem guarantees that $\semisupest$ achieves high robust accuracy.
\begin{restatable}{theorem}{restateSemisupSamples}
  \label{thm:semisupervised-samples}There exists a universal constant 
	$\tilde{r}$ such that for $\epsilon^2 \sqrt{d/\nn}\ge\tilde{r}$, $\nlab\ge \nn$ labeled data and additional $\nunlab$ unlabeled data,
	\[
	\nunlab\ge\nn \cdot 288 \epsilon^{2}\sqrt{\frac{d}{\nn}}
	\; \Rightarrow \;
	\E_{\finalest}\errrob\left(f_{\finalest}\right)\le 10^{-3}.
	\]
\end{restatable}

Therefore, compared to the fully supervised case, the self-training classifier 
requires only a constant factor more input examples, and roughly a factor 
$\epsilon^2\sqrt{d/\nn}$ fewer labels. 
We prove~\Cref{thm:semisupervised-samples} 
in~\Cref{sec:theory-semisup}, where we also precisely characterize the 
rates of convergence of the robust error; the outline of our argument is as 
follows. We have $\finalest = (\frac{1}{\nunlab}\sum_{i=1}^{\nunlab} \yun_i 
y_i) \mu + \frac{1}{\nunlab}\sum_{i=1}^{\nunlab} \yun_i \varepsilon_i$ 
where $\varepsilon_i\sim \mathcal{N}(0,\sigma^2 I)$ is the noise in 
example $i$. We show (in~\Cref{sec:theory-semisup}) that with high 
probability $\frac{1}{\nunlab}\sum_{i=1}^{\nunlab} \yun_i y_i \ge 
\frac{1}{6}$ while the variance of $\frac{1}{\nunlab}\sum_{i=1}^{\nunlab} 
\yun_i \varepsilon_i$ goes to zero as $\nunlab$ grows, and therefore the 
angle between $\finalest$ and $\mu$ goes to zero. Substituting into a 
closed-form expression for $\errrob(f_{\finalest})$ 
(Eq.~\eqref{eq:robust-err-closed-form} 
in~\Cref{sec:app-theory-closed-from}) 
gives the desired upper bound.
We remark that other learning techniques, such as EM and PCA, can also leverage unlabeled data in this model. The self-training procedure we describe is similar to 2 steps of EM~\citep{dasgupta07em}. 

\subsection{Semisupervised learning with irrelevant unlabeled data}\label{sec:irrelevant}

In \Cref{sec:app-irrelevant} we study a setting where only $\alpha\nunlab$ 
of the unlabeled data are relevant to the task, where we model the relevant 
data as before, and the irrelevant data as having no signal component, i.e., 
with $y$ uniform on $\{-1,1\}$ and $x\sim\mathcal{N}(0,\sigma^2 I)$ 
independent of $y$.
We show that for any fixed $\alpha$, high robust accuracy is still possible, 
but the required number of \emph{relevant} examples grows by a factor of 
$1/\alpha$ compared to the amount of unlabeled examples require to 
achieve the same robust accuracy when all the data is relevant. 
This demonstrates that irrelevant data can significantly hinder self-training, 
but does not stop it completely.

\section{Semi-supervised learning of robust neural networks} 
\label{sec:approach}
Existing adversarially robust training methods are designed for the supervised setting. In this section, we use these methods to leverage additional unlabeled data by adapting the self-training framework described in \Cref{sec:setup}.

\begin{algorithm} \algrenewcommand\algorithmicprocedure{} \caption{Robust self-training} \label{alg:robustself} \hspace*{\algorithmicindent} \textbf{Input:} Labeled data $(x_1, y_1, \ldots, x_{\nlab}, y_{\nlab})$ and unlabeled data $(\xun_1, \ldots, \xun_{\nunlab})$
	
	\vspace{3pt}
	
	\hspace*{\algorithmicindent} \textbf{Parameters:} Standard loss $\loss$, robust loss $\lossrob$ and unlabeled weight $\wun$ \begin{algorithmic}[1] %
		\State Learn $\intest$ by minimizing  $\sum\limits_{i=1}^{\nlab} \loss(\theta, x_i, y_i)$
		
		\State Generate pseudo-labels $\yun_i = f_{\intest}(\xun_i)$ for $i = 1, 2, \ldots \nunlab$ %
		\State Learn $\finalest$ by minimizing $\sum\limits_{i=1}^\nlab \lossrob(\theta, x_i, y_i) + \wun \sum\limits_{i=1}^{\nunlab} \lossrob(\theta, \xun_i, \yun_i)$ %
\end{algorithmic} \end{algorithm}

  \Cref{alg:robustself} summarizes robust-self training. In contrast to 
  standard self-training, we use a different supervised learning method in 
  each stage, since the intermediate and the final classifiers have different 
  goals.
  In particular, the only goal of $\intest$ is to generate high quality 
  pseudo-labels for the (non-adversarial) unlabeled data. Therefore, we 
  perform standard training in the first stage, and robust training in the 
  second. The hyperparameter $\wun$ allows us to upweight the labeled 
  data, which in some cases may be more relevant to the task (e.g., when 
  the unlabeled data comes form a different distribution), and will usually 
  have more accurate labels.

\subsection{Instantiating robust self-training}
\label{sec:losses}
  Both stages of robust self-training perform supervised learning, allowing 
  us to borrow ideas from the literature on supervised standard and robust 
  training.
We consider neural networks of the form $f_\theta(x) = \argmax_{y\in\sY} p_\theta(y \mid x)$, where $p_\theta(\,\cdot \mid x)$ is a probability distribution over the class labels.

\paragraph{Standard loss.} As in common,
we use the multi-class logarithmic loss for standard supervised learning,
\begin{equation*} 
  \loss(\theta, x, y)
  = -\log p_\theta(y\mid x).
 \end{equation*} 

\paragraph{Robust loss.} For the supervised robust loss, we use a robustness-promoting regularization term proposed in~\citep{zhang2019theoretically} and closely related to earlier proposals in~\citep{zheng2016improving, miyato2018virtual,kannan2018adversarial}. 
The robust loss is
\begin{flalign} 
\label{eq:trades} 
&\lossrob(\theta, x, y) = 
\loss(\theta, x, y) + \beta \lossreg(\theta, x),\\
&\quad\quad~\mbox{where}~~\lossreg(\theta, x)\defeq
\max_{\xadv \in\ballp(x)} \Dkl{p_\theta(\cdot \mid x)}{p_\theta(\cdot \mid \xadv)}. \nonumber
\end{flalign} 
The regularization term\footnote{
  \citet{zhang2019theoretically} write the regularization term 
  $\Dkl{p_\theta(\cdot\mid \xadv)}{p_\theta(\cdot \mid x)}$, i.e.\ with 
  $p_\theta(\cdot\mid\xadv)$ rather than $p_\theta(\cdot\mid x)$ taking 
  role of the label, but their open source implementation follows 
  \eqref{eq:trades}. 
} 
$\lossreg$ forces predictions to remain stable within $\ballp(x)$, and the hyperparameter $\beta$ balances the robustness and accuracy objectives. We consider two approximations for the maximization in $\lossreg$.

\begin{enumerate}[leftmargin = 12pt]
\item \textbf{Adversarial training: a heuristic defense via approximate maximization.}
  
We focus on $\ell_\infty$ perturbations and use the projected gradient method to approximate the regularization term of \eqref{eq:trades},
  \begin{align}\label{eq:trades-adv}
    \lossreg^\text{adv}(\theta, x) \defeq   \Dkl{p_\theta(\cdot \mid x)}{p_\theta(\cdot \mid \xadv_\texttt{PG}[x])}, %
  \end{align}
  where $\xadv_\texttt{PG}[x]$ is obtained via projected gradient ascent on  $r(x')=\Dkl{p_\theta(\cdot \mid x)}{p_\theta(\cdot \mid \xadv)}$. Empirically, performing approximate maximization during training is effective in finding classifiers that are robust to a wide range of attacks~\citep{madry2018towards}. 
  
\item \textbf{Stability training: a certified $\ell_2$ defense via randomized smoothing.}

Alternatively, we consider stability training~\citep{zheng2016improving,li2018second}, where we replace maximization over small perturbations with much larger additive random noise drawn from $\sN(0,\sigma^2 I)$,
\begin{equation} \label{eq:trades-noise} 
\lossreg^\text{stab}(\theta, x) \defeq  \E_{\xadv \sim \sN(x, \sigma^2 I)} \Dkl{p_\theta(\cdot \mid x)}{p_\theta(\cdot \mid \xadv)}. 
\end{equation}
Let $f_\theta$ be the classifier obtained by minimizing $\loss+\beta \lossrob^\text{stab}$. At test time, we use the following \emph{smoothed} classifier. 
\begin{equation} \label{eq:smoothing}g_\theta(x) \defeq \argmax_{y\in\sY} \;\; %
q_\theta(y \mid x),~~\mbox{where}~~ q_\theta(y \mid x) \defeq \P_{\xadv \sim \sN(x, \sigma^2 I)} ( f_\theta(\xadv) = y).
\end{equation} 
Improving on previous work~\citep{lecuyer2019certified, li2018second},  \citet{cohen2019certified} prove that robustness of $f_\theta$ to large random perturbations (the goal of stability training) implies \emph{certified} $\ell_2$ adversarial robustness of the smoothed classifier $g_\theta$. 
\end{enumerate}

\section{Experiments}\label{sec:experiments}
In this section, we empirically evaluate robust self-training  (\texttt{RST}) 
and show that it leads to \emph{consistent and substantial} improvement in 
robust accuracy, on both CIFAR-10~\citep{krizhevsky2009learningmultiple} 
and SVHN~\citep{netzer2011reading} and with both adversarial ($\rstatno$) 
and stability training ($\rststno$). For CIFAR-10, we mine unlabeled data 
from 80 Million Tiny Images and study in depth the strengths and 
limitations of \texttt{RST}. For SVHN, we simulate unlabeled data by 
removing labels and show that with \texttt{RST} the harm of removing the 
labels is small. This indicates that most of the gain comes from additional  
inputs rather than additional labels. Our experiments build on open source 
code from ~\cite{zhang2019theoretically, cohen2019certified}; we release 
our data and code at 
\url{https://github.com/yaircarmon/semisup-adv} and on CodaLab at  
\url{https://bit.ly/349WsAC}. 

\paragraph{Evaluating heuristic defenses.}
We evaluate  $\rstatno$  and other heuristic defenses on their performance against the strongest known $\ell_\infty$ attacks, namely the projected gradient method \citep{madry2018towards}, denoted \texttt{PG} and the Carlini-Wagner attack~\citep{carlini2017towards} denoted \texttt{CW}.

\paragraph{Evaluating certified defenses.}
For $\rststno$ and other models trained against random noise, we evaluate \emph{certified} robust accuracy of the \emph{smoothed} classifier against $\ell_2$ attacks. We perform the certification using the randomized smoothing protocol described in~\cite{cohen2019certified}, with parameters $N_0 = 100$, $N=10^4$, $\alpha=10^{-3}$ and noise variance $\sigma=0.25$. 

\paragraph{Evaluating variability.} We repeat training 3 times and report accuracy as X\dev{Y}, with X the median across runs and Y half the difference between the minimum and maximum.

\subsection{CIFAR-10}

\subsubsection{Sourcing unlabeled data}\label{sec:sourcing}
To obtain unlabeled data distributed similarly to the CIFAR-10 images, we 
use the 80 Million Tiny Images (80M-TI) dataset~\citep{torralba2008million}, of which CIFAR-10 is a manually labeled subset. However, most images in 80M-TI do not correspond to CIFAR-10 image categories.
To select relevant images, we train an 11-way classifier to distinguish CIFAR-10 classes and an 11$^\text{th}$ ``non-CIFAR-10'' class using a $\wrn{28}{10}$ model~\citep{zagoruyko2016wide} (the same as in our experiments below). 
For each class, we select additional 50K images from 80M-TI using the trained model's predicted scores\footnote{We exclude any image close  to the CIFAR-10 test set; see \Cref{app:sourcing} for detail.}---this is our 500K images unlabeled which we add to the 50K CIFAR-10 training set when performing \texttt{RST}.
We provide a detailed description of the data sourcing process in \Cref{app:sourcing}.

\newcommand{\tocheck}[1]{{\color{red}#1}}

\newcolumntype{C}[1]{>{\Centering\arraybackslash}p{#1\linewidth}}
\newcolumntype{L}[1]{>{\RaggedRight\arraybackslash}p{#1\linewidth}}
\newcolumntype{R}[1]{>{\RaggedLeft\arraybackslash}p{#1\linewidth}}
\newcommand\mC[2]{\multicolumn{1}{C{#1}}{#2}} %

\begin{table}	
	\centering
	\begin{tabular}{lllll|l||l}
		\toprule
		Model
		& $\pgdmadry$%
		& $\pgdtrades$%
		& $\pgdours$ 
		&  $\cw$~\cite{carlini2017towards}
		& Best attack 
		& No attack\\
		\midrule
		$\rstat{50K+500K}$ & 
		{63.1}  & {63.1} & 62.5 & 
		{64.9} & \textbf{62.5} {\footnotesize $\pm$0.1}& \textbf{89.7} {\footnotesize $\pm$0.1} \\
		TRADES~\cite{zhang2019theoretically} &
		{55.8} & 56.6 & 55.4 & 
		{65.0} & 55.4 & 84.9\\
		Adv.\ pre-training~\cite{hendrycks2019pretraining} & 
		57.4 & 58.2 & 57.7 &
		-     & 57.4$^\dagger$ & {87.1} \\
		Madry et al.~\cite{madry2018towards} & 
		45.8 & - & - & 
		{47.8} & 45.8 & {87.3}\\
		Standard self-training & 
		- & 0.3 & 0 & 
		- & 0 & 96.4 \\
		\bottomrule    
	\end{tabular}
	\caption{\textbf{Heuristic defense.} CIFAR-10 test accuracy under 
	different optimization-based $\ell_\infty$ attacks of magnitude 
	$\epsilon=8/255$. Robust self-training (RST) with 500K unlabeled Tiny 
	Images outperforms the state-of-the-art robust models in terms of 
  robustness as well as standard accuracy (no attack). Standard self-training with the 
	same data does not provide robustness.
		$\dagger$: A projected gradient attack with 1K restarts reduces the 
		accuracy of this model to 52.9\%, evaluated on 10\% of the test set 
		\cite{hendrycks2019pretraining}. 
	}
	\label{table:adv_table}
\end{table}

\begin{figure}
	\centering

		\newcolumntype{C}[1]{>{\Centering\arraybackslash}p{#1\linewidth}}
		\newcolumntype{L}[1]{>{\RaggedRight\arraybackslash}p{#1\linewidth}}
		\newcolumntype{R}[1]{>{\RaggedLeft\arraybackslash}p{#1\linewidth}}
	\begin{minipage}[b]{0.4\textwidth}
		\centering
		\hspace{-0pt}
		\notarxiv{\includegraphics[height=3.4cm]{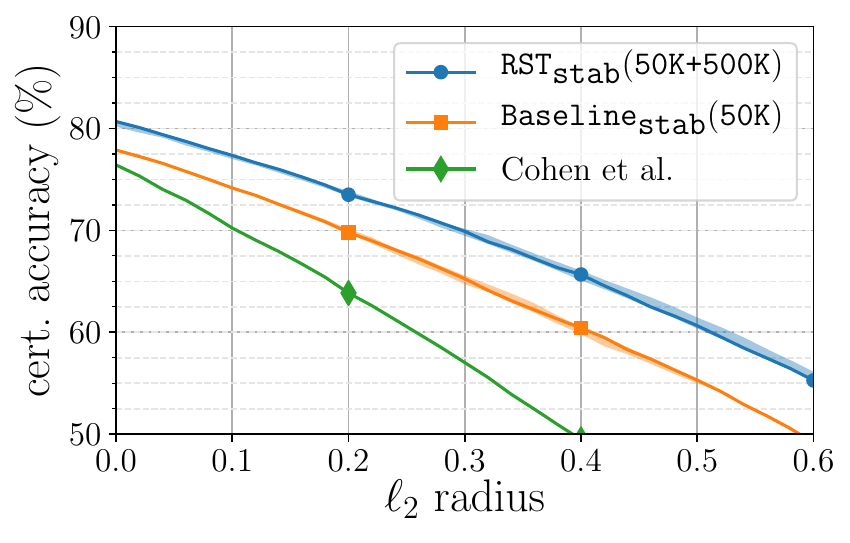}}%
		\arxiv{\includegraphics[height=4.1cm]{figures/main_stability_acc_vs_rad_adj.pdf}}
		\\[-3pt]
		\footnotesize\textbf{(a)}
	\end{minipage}
	\begin{minipage}[b]{0.59\textwidth}
	\centering
	\begin{tabular}{L{0.44}L{0.2}L{0.175}}
		\toprule
		Model & $\ell_\infty$ acc.\ at $\epsilon=\frac{2}{255}$ & Standard acc. \\ 
		\midrule
		$\rstst{50K+500K}$ & \textbf{63.8}\dev{0.5} & \textbf{80.7}\dev{0.3} \\
		$\baseline{stab}{50K}$ & 58.6\dev{0.4} & 77.9\dev{0.1} \\
		{\arxiv{\footnotesize}\notarxiv{\small}{Wong et al.  (single)}}~\cite{wong2018scaling} & 53.9 & 68.3 \\
		{\arxiv{\footnotesize}\notarxiv{\small}{Wong et al.  (ensemble)}}~\cite{wong2018scaling} & \textbf{63.6} & 64.1 \\
		IBP~\cite{gowal2018effectiveness}& 50.0 & 70.2 \\
		\bottomrule
	\end{tabular}
	\\[3pt]
	\footnotesize\textbf{(b)}
	\end{minipage}
	\caption{\textbf{Certified defense.} Guaranteed CIFAR-10 test accuracy 
	under all $\ell_2$ and $\ell_\infty$ attacks.
	Stability-based robust self-training  with 500K unlabeled Tiny 
	Images ($\rstst{50K+500K}$) outperforms stability training with only 
	labeled data ($\baseline{stab}{50K}$). 
	 \textbf{(a)} Accuracy vs.\ 
	$\ell_2$ radius, certified via randomized 
  smoothing~\cite{cohen2019certified}. Shaded regions 
   indicate variation 
	across 3 runs. Accuracy at $\ell_2$ radius 0.435 implies accuracy at 
	$\ell_\infty$ radius 2/255. 
	\textbf{(b)} The implied $\ell_\infty$ certified accuracy is comparable to 
	the state-of-the-art in methods that directly target $\ell_\infty$ 
	robustness.
  }
	\label{fig:main}
\end{figure}

\subsubsection{Benefit of unlabeled data}\label{sec:benefit}

We perform robust self-training using the unlabeled data described above. 
We use a Wide ResNet 28-10 architecture for both the intermediate pseudo-label generator  and final robust model. For adversarial training, we compute $x_\texttt{PG}$ exactly as in~\citep{zhang2019theoretically} with $\epsilon=8/255$, and denote the resulting model as $\rstat{50K+500K}$. For stability training, we set the additive noise variance to to $\sigma=0.25$ and denote the result $\rstst{50K+500K}$. We provide training details in \Cref{sec:app-hyper}. 

\paragraph{Robustness of $\rstat{50K+500K}$ against strong attacks.} 
In Table~\ref{table:adv_table}, we report the accuracy of  $\rstat{50K+500K}$ and the best models in the literature against various strong attacks at $\epsilon=8/255$ (see \Cref{sec:app-attack} for details). $\pgdtrades$ and $\pgdmadry$ correspond to the attacks used in~\citep{zhang2019theoretically} and~\citep{madry2018towards}  respectively, and we apply the Carlini-Wagner attack $\cw$~\cite{carlini2017towards} on $1{,}000$ random test examples, where we use the implementation~\cite{papernot2018cleverhans} that performs search over attack hyperparameters. We also tune a PG attack against $\rstat{50K+500K}$ (to maximally reduce its accuracy), which we denote  $\pgdours$ (see \Cref{sec:app-attack} for details). 

$\rstat{50K+500K}$ gains 7\% over TRADES~\cite{zhang2019theoretically}, which we can directly attribute to the unlabeled data (see \Cref{sec:app-hyper-comp}). In~\Cref{sec:app-radii} we also show this gain holds over different attack radii. 
The model  of~\citet{hendrycks2019pretraining} is based on ImageNet adversarial pretraining and is less directly comparable to ours due to the difference in external data and training method. Finally, we perform  standard self-training using the unlabeled data, which offers a moderate 0.4\% improvement in standard accuracy over the intermediate model but is not adversarially robust (see~\Cref{sec:app-sst}).

\paragraph{Certified robustness of $\rstst{50K+500K}$.} 
Figure~\ref{fig:main}a shows the certified robust accuracy as a function of $\ell_2$ perturbation radius for different models. We compare $\rstat{50K+500K}$ with~\citep{cohen2019certified}, which has the highest reported certified accuracy, and $\baseline{stab}{50K}$, a model that we trained using only the CIFAR-10 training set and the same training configuration as $\rstst{50K+500K}$. $\rstst{50K+500K}$ improves on our $\baseline{stab}{50K}$ by 3--5\%. The gains of $\baseline{stab}{50K}$ over the previous state-of-the-art are due to a combination of better architecture, hyperparameters, and training objective (see \Cref{sec:app-compare-stab}). 
The certified $\ell_2$ accuracy is strong enough to imply state-of-the-art 
certified $\ell_\infty$ robustness via elementary norm bounds. In 
Figure~\ref{fig:main}b we compare $\rstst{50K+500K}$ to the 
state-of-the-art in certified $\ell_\infty$ robustness, showing a a 10\% 
improvement over single models, and performance on par with the cascade 
approach of~\cite{wong2018scaling}. We also outperform the cascade 
model's standard accuracy by $16\%$. 

\subsubsection{Comparison to alternatives and ablations studies}

\paragraph{Consistency-based semisupervised learning (\Cref{sec:app-semisup}).}
Virtual adversarial training (VAT), a state-of-the-art method for (standard) semisupervised training of neural network~\citep{miyato2018virtual, oliver2018realistic}, is easily adapted to the adversarially-robust setting. 
We train models using adversarial- and stability-flavored adaptations of VAT, and compare them to their robust self-training counterparts.
 We find that the VAT approach offers only limited benefit over fully-supervised robust training, and that robust self-training offers 3--6\% higher accuracy.

\paragraph{Data augmentation (\Cref{sec:app-augment}).}
In the low-data/standard accuracy regime, strong data augmentation is competitive against and complementary to semisupervised learning~\citep{cubuk2019autoaugment, xie2019unsupervised}, as it effectively increases the sample size by generating different plausible inputs. 
It is therefore natural to compare state-of-the-art data augmentation (on the labeled data only) to robust self-training.
We consider two popular schemes: Cutout~\citep{devries2017improved} 
and AutoAugment~\citep{cubuk2019autoaugment}. While they provide 
significant benefit to standard accuracy, both augmentation schemes 
provide essentially no improvements when we add them to our fully 
supervised baselines.

\paragraph{Relevance of unlabeled data (\Cref{sec:app-relevance}).}  The
theoretical analysis in \Cref{sec:theory} suggests that self-training
performance may degrade significantly in the presence of irrelevant
unlabeled data; other semisupervised learning methods share this
sensitivity~\cite{oliver2018realistic}. 
In order to measure the effect on robust self-training, we mix out unlabeled data sets with different amounts of random images from 80M-TI and compare the performance of resulting models. We find that stability training is more
sensitive than adversarial training, and that both methods still yield noticeable
robustness gains, even with only 50\% relevant data.

\paragraph{Amount of unlabeled data (\Cref{sec:app-amount}).}
We perform robust self-training with varying amounts of unlabeled data 
and observe that 100K unlabeled data 
provide roughly half the gain provided by 500K unlabeled data, indicating 
diminishing returns as data amount grows. However, as we report in 
Appendix~\ref{sec:app-amount}, hyperparameter tuning issues make it  
difficult to assess how performance trends with data amount.

\paragraph{Amount of labeled data (\Cref{sec:app-label-amount}).}
Finally, to explore the complementary question of the effect of varying the 
amount of labels available for pseudo-label generation, we strip the labels 
of all but $\nn$ CIFAR-10 images, and combine the remainder with our 
500K unlabeled data. We observe that $\nn=8$K labels suffice to to exceed 
the robust accuracy of the (50K labels) fully-supervised baselines for 
both adversarial 
training and the $\pgdours$ attack, and certified robustness via stability 
training.
\subsection{Street View House Numbers (SVHN)}\label{sec:svhn}
The SVHN dataset~\citep{netzer2011reading} is naturally split into a core training set of about 73K images and an `extra' training set with about 531K easier images. 
In our experiments, we compare three settings:
 (i) robust training on the core training set only, denoted $\baseline{*}{73K}$, 
 (ii) robust self-training with the core training set and the extra training images, denoted $\rstgen{*}{73K+531K}$, and 
 (iii) robust training on all the SVHN training data, denoted $\baseline{*}{604K}$. 
 As in CIFAR-10, we experiment with both adversarial and stability training, so $*$ stands for either $\texttt{adv}$ or $\texttt{stab}$. 

Beyond validating the benefit of additional data, our SVHN experiments  measure the loss inherent in using pseudo-labels in lieu of true labels. 
Figure~\ref{fig:svhn} summarizes the results: the unlabeled provides 
significant gains in robust accuracy, and the accuracy drop due to using 
pseudo-labels is below 1\%. This reaffirms our intuition that in regimes of 
interest, \emph{perfect labels are not crucial} for improving robustness.
We give a detailed account of our SVHN experiments in 
\Cref{sec:app-svhn}, where we also compare our results to the literature.

\begin{figure}[t]
  \centering
	\notarxiv{  \begin{subfigure}{0.4\textwidth}
      \begin{tabular}{ccc}
        \toprule
        Model & $\pgdours$ & No attack\\ 
        \midrule
        $\baseline{at}{73K}$ & 75.3\dev{0.4} & 94.7\dev{0.2}  \\
        $\rstat{73K+531K}$ & 86.0\dev{0.1} & 97.1\dev{0.1}  \\
        $\baseline{at}{604K}$ & 86.4\dev{0.2} & 97.5\dev{0.1}  \\
        \bottomrule
      \end{tabular}
  \end{subfigure}
  ~
    \begin{subfigure}{0.58\textwidth}
      \flushright
      \includegraphics[width=0.75\columnwidth]{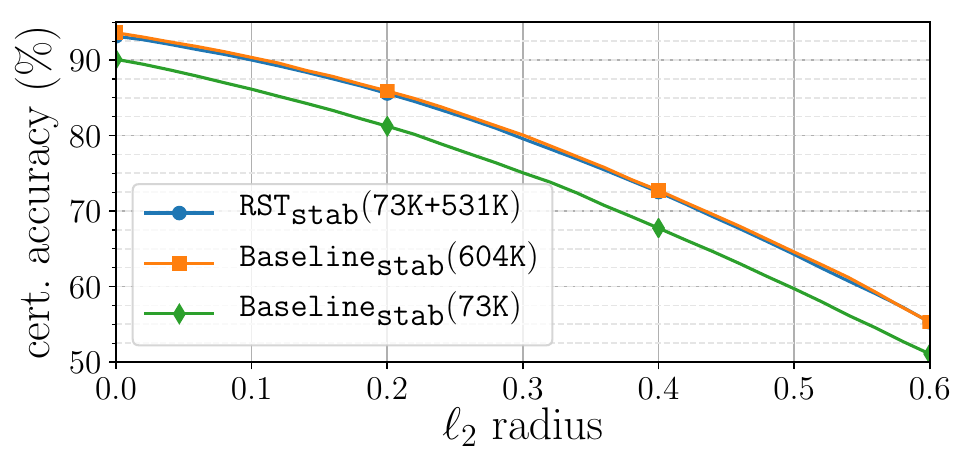}
    \end{subfigure}
	\captionlistentry[table]{A table beside a figure}}%
	\arxiv{
	~~~
	\begin{minipage}[b]{0.45\textwidth}
		\centering
		\begin{tabular}{ccc}
		\toprule
		Model & $\pgdours$ & No attack\\ 
		\midrule
		$\baseline{adv}{73K}$ & 75.3\dev{0.4} & 94.7\dev{0.2}  \\
		$\rstat{73K+531K}$ & 86.0\dev{0.1} & 97.1\dev{0.1}  \\
		$\baseline{adv}{604K}$ & 86.4\dev{0.2} & 97.5\dev{0.1}  \\
		\bottomrule
	\end{tabular}
	\vspace{20pt}
	\end{minipage}
	~~~
	\begin{minipage}[b]{0.45\textwidth}
		\centering
	\includegraphics[height=3.3cm]{figures/svhn_stability_acc_vs_rad_adj.pdf}
	\\[-3pt]
	\end{minipage}
	\vspace{-6pt}}%

      \caption{
      SVHN test accuracy for robust training without the extra data, with unlabeled extra (self-training), and with the labeled extra data. Left: Adversarial training and accuracies under $\ell_\infty$ attack with $\epsilon=4/255$. Right: Stability training and certified $\ell_2$ accuracies as a function of perturbation radius. Most of the gains from extra data comes from the unlabeled inputs.
  }
      \label{fig:svhn}
\end{figure}

\section{Related work}\label{sec:related}

\paragraph{Semisupervised learning.}
The literature on semisupervised learning dates back to beginning of 
machine 
learning~\citep{scudder1965probability,chapelle2006semisupervised}. 
A recent family of approaches operate by enforcing consistency in the 
model's predictions under various perturbations of the unlabeled 
data~\citep{miyato2018virtual, xie2019unsupervised}, or over the course 
of  training
\citep{tarvainen2017mean, sajjadi2016regularization, laine2017temporal}.
While self-training has shown some gains in standard
accuracy~\citep{lee2013pseudo}, the consistency-based approaches perform 
significantly better on popular semisupervised learning 
benchmarks~\citep{oliver2018realistic}.
In contrast, our paper considers the very different regime of adversarial 
robustness, and we observe that robust self-training offers significant 
gains in robustness over fully-supervised methods. Moreover, it seems to 
outperform consistency-based regularization 
 (VAT; see Section~\ref{sec:app-semisup}).
We note that there are many additional approaches to semisupervised 
learning, including transductive SVMs, graph-based methods, and 
generative modeling~\citep{chapelle2006semisupervised, zhu2003semi}.

\paragraph{Self-training for domain adaptation.}
Self-training is gaining prominence in the related 
setting of \emph{unsupervised domain adaptation} (UDA). There, the 
unlabeled data is from a ``target'' distribution, which is different from the 
``source'' distribution that generates labeled data. Several recent 
approaches~\citep[cf.][]{long2013transfer, inoue2018cross} are based on 
approximating class-conditional distributions of the target domain via 
self-training, and then learning feature transformations that match these 
conditional distributions across the source and target domains. Another line 
of work~\citep{zou2018unsupervised, zou2019confidence} is based on 
iterative self-training coupled with refinements such as class balance or 
confidence regularization. Adversarial robustness and UDA share the 
similar goal of 
learning models that perform well under some kind of distribution shift; in UDA we access the target distribution 
through unlabeled data while in adversarial robustness, we characterize 
target distributions via perturbations. The fact that self-training is 
effective in both cases suggests it may apply to distribution shift 
robustness more broadly.

\paragraph{Training robust classifiers.} The discovery of adversarial 
examples~\citep{szegedy2014intriguing, 
biggio2013evasion,biggio2018wild} prompted a flurry of 
``defenses'' and ``attacks.'' While several defenses were broken by 
subsequent attacks~\citep{carlini2017towards,athalye2018obfuscated, 
carlini2017adversarial}, the general approach of adversarial training 
\citep{madry2018towards, sinha2018certifiable,zhang2019theoretically} 
empirically seems to offer gains in robustness. Other lines of work attain 
\emph{certified} robustness, though often at a cost to empirical robustness 
compared to heuristics~\cite{raghunathan2018certified, 
wong2018provable, raghunathan2018sdp, wong2018scaling, 
gowal2018effectiveness}. 
Recent work by \citet{hendrycks2019pretraining} shows that even when  
pre-training has limited value for standard accuracy on benchmarks, 
adversarial pre-training is effective.
We complement this work by showing 
that a similar conclusion holds for semisupervised learning (both 
practically and theoretically in a stylized model), and extends to certified 
robustness as well.

\paragraph{Sample complexity upper bounds.} Recent works~\cite{yin2019rademacher,khim2018adversarial,attias2019improved} 
study adversarial robustness from a learning-theoretic perspective, and in a 
number of simplified settings develop generalization bounds using extensions of Rademacher complexity. In some cases these upper bounds are demonstrably larger than their standard counterparts, suggesting there may be statistical barriers to robust learning.

\paragraph{Barriers to robustness.} 
\citet{schmidt2018adversarially} show a sample complexity barrier to robustness in a stylized setting. 
We observed that in this model, unlabeled data is as useful for robustness as labeled data. 
This observation led us to experiment with robust semisupervised learning.
Recent work also suggests other barriers to robustness: 
\citet{montasser2019vc} show settings where improper learning and 
surrogate losses are crucial in addition to more samples; 
\citet{bubeck2019adversarial} and~\citet{degwekar2019computational} 
show possible computational barriers; \citet{gilmer2018adversarial} show a 
high-dimensional model where robustness is a consequence of any 
non-zero standard error, while~\citet{raghunathan2019hurt, 
tsipras2019robustness, fawzi2018analysis} show settings where robust and 
standard errors are at odds.
Studying ways to overcome these additional theoretical barriers may translate to more progress in practice.

\paragraph{Semisupervised learning for adversarial robustness.}
Independently and concurrently with our work,~\citet{zhai2019adversarially, najafi2019robustness} and~\citet{uesato2019are} also study the use of unlabeled data in the adversarial setting. We briefly describe each work in turn, and then contrast all three with ours.

\citet{zhai2019adversarially} study the Gaussian model of~\citep{schmidt2018adversarially} and show a PCA-based procedure that successfully leverages unlabeled data to obtain adversarial robustness. They propose a training procedure that at every step treats the current model's predictions as true labels, and experiment on CIFAR-10. Their experiments include the standard semisupervised setting where some labels are removed, as well as the transductive setting where the test set is added to the training set without labels.

\citet{najafi2019robustness} extend the distributionally robust optimization perspective of~\citep{sinha2018certifiable} to a semisupervised setting. They propose a training objective that replaces pseudo-labels with soft labels weighted according to an adversarial loss, and report results on MNIST, CIFAR-10, and SVHN with some training labels removed. The experiments in \citep{zhai2019adversarially, najafi2019robustness} do not augment CIFAR-10 with new unlabeled data and do not improve the state-of-the-art in adversarial robustness.

The work of~\citet{uesato2019are} is the closest to ours---they also study self-training in the Gaussian model and propose a version of robust self-training which they apply on CIFAR-10 augmented with Tiny Images. Using the additional data they obtain new state-of-the-art results in heuristic defenses, comparable to ours. As our papers are very similar, we provide a detailed comparison in~\Cref{sec:app-comparison}.

Our paper offers a number of perspectives that complement~\citep{uesato2019are, zhai2019adversarially, najafi2019robustness}. First, in addition to heuristic defenses, we show gains in certified robustness where we have a guarantee on robustness against \emph{all} possible attacks. Second, we study the impact of irrelevant unlabeled data theoretically (\Cref{sec:irrelevant}) and empirically (\Cref{sec:app-relevance}). Finally, we provide additional experimental studies of data augmentation and of the impact of unlabeled data amount when using all labels from CIFAR-10.

\section{Conclusion}\label{sec:conclusion}
We show that unlabeled data closes a sample complexity gap in a stylized 
model and that robust self-training (RST) is consistently beneficial on two 
image classification benchmarks. Our findings open up a number of 
avenues for further research. Theoretically, is sufficient unlabeled data a 
universal cure for sample complexity gaps between standard and 
adversarially robust learning?
Practically, what is the best way to leverage unlabeled data for 
robustness, and can semisupervised learning similarly benefit alternative 
(non-adversarial) notions of robustness? As the scale of data grows, 
computational capacities increase, and machine learning moves beyond 
minimizing average error, we expect unlabeled data to provide continued 
benefit.

\newpage

\paragraph{Reproducibility.} Code, data, and experiments are available 
on GitHub at \url{https://github.com/yaircarmon/semisup-adv} and on 
CodaLab at \url{https://bit.ly/349WsAC}.

\arxiv{\section*{Acknowledgments}}
\notarxiv{\subsubsection*{Acknowledgments}} %
The authors would like to thank an anonymous reviewer for proposing the 
label amount experiment in~\Cref{sec:app-label-amount}. 
YC was supported by the Stanford Graduate Fellowship. 
AR was supported by the Google Fellowship and Open Philanthropy AI 
Fellowship. 
PL was supported by the Open Philanthropy Project Award. 
JCD was supported by the NSF CAREER award 1553086, the Sloan Foundation and ONR-YIP N00014-19-1-2288.
 \setlength{\bibsep}{5pt plus 1pt minus 1pt}

\bibliographystyle{abbrvnat}

\newpage

\part*{Supplementary Material}

\appendix

\section{Theoretical results}\label{sec:app-theory}

This appendix contains the full proofs for the results in~\Cref{sec:theory}, 
as well as explicit bounds for the robust error of the self-training estimator.

We remark that the results of this section easily extend
to the case where there is class imbalance:
The upper bounds in~\Cref{prop:supervised-samples} 
and~\Cref{thm:semisupervised-samples} hold regardless of the label 
distribution, while the lower bound in~\Cref{thm:lower-bound} changes 
from $\half(1-d^{-1})$ to $p(1-d^{-1})$ where $p$
is the proportion of the smaller class; the only change to the proof 
in~\cite{schmidt2018adversarially} is 
a modification of the lower bound on $\Psi$ in page 29 of the arxiv version.

\subsection{Error probabilities in closed 
form}\label{sec:app-theory-closed-from}

We recall our model $x\sim\mathcal{N}\left(y\mu,\sigma^{2}I\right)$
with $y$ uniform on $\left\{ -1,1\right\} $ and $\mu\in\R^{n}$.
Consider a linear classifier $f_{\theta}\left(x\right)=\mathrm{sign}\left(x^{\top}\theta\right)$.
Then the standard error probability is
\begin{equation}
\err\left(f_{\theta}\right)=\P\left(y\cdot x^{\top}\theta<0\right)=\P\left(\mathcal{N}\left(\frac{\mu^{\top}\theta}{\sigma\left\Vert \theta\right\Vert },1\right)<0\right)\eqdef Q\left(\frac{\mu^{\top}\theta}{\sigma\left\Vert \theta\right\Vert }\right)\label{eq:err-closed-form}
\end{equation}
where 
\[
Q\left(x\right)=\frac{1}{\sqrt{2\pi}}\int_{x}^{\infty}e^{-t^{2}/2}dt
\]
is the Gaussian error function. For linear classifier $f_{\theta}$,
input $x$ and label $y$, the strongest adversarial perturbation
of $x$ with $\ell_{\infty}$ norm $\epsilon$ moves each coordinate
of $x$ by $-\epsilon\mathrm{sign}\left(y\theta\right)$. The robust
error probability is therefore
\begin{align}
\errrob\left(f_{\theta}\right) & =\P\left(\inf_{\left\Vert \nu\right\Vert _{\infty}\le\epsilon}\left\{ y\cdot\left(x+\nu\right)^{\top}\theta\right\} <0\right)\nonumber \\
 & =\P\left(y\cdot x^{\top}\theta-\epsilon\left\Vert \theta\right\Vert _{1}<0\right)=\P\left(\mathcal{N}\left(\mu^{\top}\theta,\left(\sigma\left\Vert \theta\right\Vert \right){}^{2}\right)<\epsilon\left\Vert \theta\right\Vert _{1}\right)\nonumber \\
 & =Q\left(\frac{\mu^{\top}\theta}{\sigma\left\Vert \theta\right\Vert }-\frac{\epsilon\left\Vert \theta\right\Vert _{1}}{\sigma\left\Vert \theta\right\Vert }\right)\le Q\left(\frac{\mu^{\top}\theta}{\sigma\left\Vert \theta\right\Vert }-\frac{\epsilon\sqrt{d}}{\sigma}\right).\label{eq:robust-err-closed-form}
\end{align}
In this model, standard and robust accuracies {align} in the
sense that any highly accurate standard classifier, with $\frac{\mu^{\top}\theta}{\left\Vert \theta\right\Vert }>\epsilon\sqrt{d}$,
will necessarily also be robust. Moreover, for dense $\mu$ (with
$\lone{\mu}/\norm{\mu} =\Omega(\sqrt{d})$),
good linear estimators will typically be dense as well, in which case
$\frac{\mu^{\top}\theta}{\sigma\left\Vert \theta\right\Vert }$ determines
both standard and robust accuracies. Our analysis will consequently
focus on understanding the quantity $\frac{\mu^{\top}\theta}{\sigma\left\Vert \theta\right\Vert }$. 

\subsubsection{Optimal standard accuracy and parameter setting}

We note that for a given problem instance, the classifier that minimizes
the standard error is simply $\theta^{\star}=\mu$. Its standard error
is
\[
\err\left(f_{\theta^{\star}}\right)=Q\left(\frac{\left\Vert \mu\right\Vert }{\sigma}\right)\le e^{-\left\Vert \mu\right\Vert ^{2}/2\sigma^{2}}.
\]
Recall our parameter setting,
\begin{equation}
\epsilon\le\frac{1}{2},\ \sigma=\left(n_{0}d\right)^{1/4},\ \text{and}\ \left\Vert \mu\right\Vert ^{2}=d.\label{eq:param-set}
\end{equation}
Under this setting, $\frac{\left\Vert \mu\right\Vert }{\sigma}=\left(\frac{d}{\nn}\right)^{1/4}$
and we have
\[
\err\left(f_{\theta^{\star}}\right)=Q\left(\left(\frac{d}{\nn}\right)^{1/4}\right)\le e^{-\frac{1}{2}\sqrt{d/\nn}}\ \text{and}\ \errrob\left(f_{\theta^{\star}}\right)\le Q\left(\left(1-\epsilon\right)\left(\frac{d}{\nn}\right)^{1/4}\right)\le e^{-\frac{1}{8}\sqrt{d/\nn}}.
\]
Therefore, in the regime $d/\nn\gg1$, the classifier $\theta^{\star}$
achieves essentially perfect accuracies, both standard and robust.
We will show that estimating $\theta$ from $\nn$ labeled data and
a large number ($\approx\sqrt{d/\nn}$) of unlabeled data allows us to
approach the performance of $\theta^{\star}$, without prior knowledge
of $\mu$. 

\subsection{Performance of supervised estimator}\label{sec:theory-sup}

Given labeled data set $\left(x_{1},y_{1}\right),\ldots,\left(x_{\nlab},y_{\nlab}\right)$
we consider the linear classifier given by
\[
\estim_{\nlab}=\frac{1}{\nlab}\sum_{i=1}^{\nlab}y_{i}x_{i}.
\]
In the following lemma we give a tight concentration bound for $\mu^{\top}\estim_{\nlab}/\left(\sigma\left\Vert \estim_{\nlab}\right\Vert \right)$,
which determines the standard and robust error probabilities of $f_{\estim_{\nlab}}$
via equations (\ref{eq:err-closed-form}) and (\ref{eq:robust-err-closed-form})
respectively
\begin{lemma}
\label{lem:supervised-perf}There exist numerical constants $c_{0},c_{1},c_{2}$
such that under parameter setting (\ref{eq:param-set}) and $d/n_{0}>c_{0}$,
\[
\frac{\mu^{\top}\estim_{\nlab}}{\sigma\left\Vert \estim_{\nlab}\right\Vert }\ge\left(\sqrt{\frac{n_{0}}{d}}+\frac{\nn}{n}\left(1+c_{1}\left(\frac{\nn}{d}\right)^{1/8}\right)\right)^{-1/2}\ \text{with probability }\ge1-e^{-c_{2}\left(d/\nn\right)^{1/4}\min\left\{ n,\left(d/\nn\right)^{1/4}\right\} }.
\]
\end{lemma}
\begin{proof}
We have
\[
\estim_{\nlab}\sim\mathcal{N}\left(\mu,\frac{\sigma^{2}}{n}I\right)\text{ so that \ensuremath{\delta\defeq\estim_{\nlab}-\mu\sim\mathcal{N}\left(0,\frac{\sigma^{2}}{n}I\right)}}.
\]
To lower bound the random variable $\frac{\mu^{\top}\estim_{\nlab}}{\left\Vert \estim_{\nlab}\right\Vert }$
we consider its squared inverse, and decompose it as follows
\begin{align*}
\frac{\left\Vert \estim_{\nlab}\right\Vert ^{2}}{\left(\mu^{\top}\estim_{\nlab}\right)^{2}} & =\frac{\left\Vert \delta+\mu\right\Vert ^{2}}{\left(\left\Vert \mu\right\Vert ^{2}+\mu^{\top}\delta\right)^{2}}=\frac{1}{\left\Vert \mu\right\Vert ^{2}}+\frac{\left\Vert \delta\right\Vert ^{2}-\frac{1}{\norm{\mu}^2}\left(\mu^{\top}\delta\right)^{2}}{\left(\left\Vert \mu\right\Vert ^{2}+\mu^{\top}\delta\right)^{2}}\\
 & \le\frac{1}{\left\Vert \mu\right\Vert ^{2}}+\frac{\left\Vert \delta\right\Vert ^{2}}{\left(\left\Vert \mu\right\Vert ^{2}+\mu^{\top}\delta\right)^{2}}
\end{align*}
To obtain concentration bounds, we note that
\[
\left\Vert \delta\right\Vert ^{2}\sim\frac{\sigma^{2}}{\nlab}\chi_{d}^{2}\ \text{ and }\ \frac{\mu^{\top}\delta}{\left\Vert \mu\right\Vert }\sim\mathcal{N}\left(0,\frac{\sigma^{2}}{n}\right).
\]
Therefore, standard concentration results give
\begin{equation}
\P\left(\left\Vert \delta\right\Vert ^{2}\ge\frac{\sigma^{2}d}{n}\left(1+\frac{1}{\sigma}\right)\right)\le e^{-d/8\sigma^{2}}\ \text{ and }\ \P\left(\frac{\mu^{\top}\delta}{\left\Vert \mu\right\Vert }\ge\left(\sigma\left\Vert \mu\right\Vert \right)^{1/2}\right)\le2e^{-\frac{1}{2}n\left\Vert \mu\right\Vert /\sigma}.\label{eq:supervised-concentration}
\end{equation}
	Assuming that the two events $\left\Vert \delta\right\Vert ^{2}\le\frac{\sigma^{2}d}{n}\left(1+\frac{1}{\sigma}\right)$
and $\left|\mu^{\top}\delta\right|\le  \sigma^{1/2}\left\Vert \mu\right\Vert^{3/2} $
hold, we have
\[
\frac{\left\Vert \estim_{\nlab}\right\Vert ^{2}}{\left(\mu^{\top}\estim_{\nlab}\right)^{2}}
\le
\frac{1}{\left\Vert \mu\right\Vert ^{2}}+\frac{\frac{\sigma^{2}d}{n}\left(1+\frac{1}{\sigma}\right)}{\left\Vert \mu\right\Vert ^{4}\left(1-\left(\sigma/\left\Vert \mu\right\Vert \right)^{1/2}\right)^{2}}.
\]
Substituting the parameter setting setting (\ref{eq:param-set}),
we have that for $d/n_{0}$ sufficiently large,
\begin{align*}
\frac{\sigma^{2}\left\Vert \estim_{\nlab}\right\Vert ^{2}}{\left(\mu^{\top}\estim_{\nlab}\right)^{2}} & \le
\sqrt{\frac{\nn}{d}} + \frac{\frac{\nn d^2}{n}\left(1+\left(\nn d\right)^{-1/4}\right)}{d^{2}\left(1-\left(\nn/d\right)^{1/8}\right)^{2}}\le\sqrt{\frac{\nn}{d}}+\frac{n_{0}}{n}\left(1+c_{1}\left(\nn/d\right)^{1/8}\right)
\end{align*}
for some numerical constant $c_{1}.$ For this to imply the bound
stated in the lemma we also need $\mu^{\top}\estim_{\nlab}\ge0$ to
hold, but this is already implied by
\[
\mu^{\top}\estim_{\nlab}=\left\Vert \mu\right\Vert ^{2}+\mu^{\top}\delta\ge\left\Vert \mu\right\Vert ^{2}\left(1-\left(\sigma/\left\Vert \mu\right\Vert \right)^{-1/2}\right)\ge d\left(1-\left(\frac{\nn}{d}\right)^{1/8}\right)>0.
\]
 Substituting the parameters settings into the concentration bounds
(\ref{eq:supervised-concentration}), we have by the union bound that
the desired upper bound fails to hold with probability at most
\[
e^{-d/8\sqrt{\nn d}}+2e^{-n\sqrt{d}/2\left(\nn d\right)^{1/4}}\le e^{-c_{2}\left(d/\nn\right)^{1/4}\min\left\{ n,\left(d/\nn\right)^{1/4}\right\} }
\]
for another numerical constant $c_{2}$ and $d/\nn>1$. 
\end{proof}
As an immediate corollary to Lemma \ref{lem:supervised-perf}, we
obtain the sample complexity upper bounds cited in the main text.
\restateSupSamples*
\begin{proof}
For the case $n\ge\nn$ we take $r$ sufficiently large such that
by Lemma \ref{lem:supervised-perf} we have
\[
\frac{\mu^{\top}\estim_{\nlab}}{\sigma\left\Vert \estim_{\nlab}\right\Vert }\ge\frac{1}{\sqrt{2\left(\frac{\nn}{n}+\sqrt{\frac{n_{0}}{d}}\right)}}\ge\frac{1}{2}\ \text{with probability}\ge1-e^{-c_{2}\sqrt{d/\nn}}
\]
for an appropriate $c_{2}$. Therefore by the expression (\ref{eq:err-closed-form})
for the standard error probability (and the fact that it is never
more than 1), we have
\[
\E_{\estim_{n}}\err\left(f_{\estim_{n}}\right)\le Q\left(\frac{1}{2}\right)+e^{-c_{2}\left(d/\nn\right)^{1/8}}\le\frac{1}{3}
\]
for appropriate $r$. Similarly, for the case $n\ge\nn\cdot 4 \epsilon^2 \sqrt{\frac{d}{\nn}}$ we apply
Lemma \ref{lem:supervised-perf} combined with $\epsilon<\frac{1}{2}$
to write
\[
\frac{\mu^{\top}\estim_{\nlab}}{\sigma\left\Vert \estim_{\nlab}\right\Vert }\ge\frac{1}{\sqrt{2\left(\frac{\nn}{n}+\sqrt{\frac{n_{0}}{d}}\right)}}\ge\frac{1}{\sqrt{2\left(\frac{\nn}{4\epsilon^{2}\sqrt{\nn d}}+\frac{1}{4\epsilon^{2}}\sqrt{\frac{n_{0}}{d}}\right)}}=\sqrt{2}\epsilon\left(\frac{d}{\nn}\right)^{1/4}
\]
with probability $\ge1-e^{-c_{2}\left(d/\nn\right)^{1/4}\min\left\{ n,\left(d/\nn\right)^{1/4}\right\} }$.
Therefore, using the expression (\ref{eq:robust-err-closed-form})
and $\sigma=\left(\nn d\right)^{1/4}$, we have (using $n\ge\epsilon^{2}\left(d/\nn\right)^{1/4}$)
\[
\E_{\estim_{n}}\errrob\left(f_{\estim_{n}}\right)\le Q\left(\left[\sqrt{2}-1\right]\epsilon(d/\nn)^{1/4}\right)+e^{-\epsilon^{2}c_{2}\sqrt{d/\nn}}\le 10^{-3},
\]
for sufficiently large $r$.
\end{proof}

\subsection{Lower bound}

We now briefly explain how to translate the sample complexity lower
bound of \citet{schmidt2018adversarially} into our parameter setting.
\restateLowerBound*
\begin{proof}
The setting of our theorem is identical to that of Theorem 11 in \citet{schmidt2018adversarially},
which shows that
\[
\E\,\errrob(\learner_{\nlab}[S])\ge\frac{1}{2}\P\left(\left\Vert \mathcal{N}\left(0,I\right)\right\Vert _{\infty}\le\epsilon\sqrt{1+\frac{\sigma^{2}}{n}}\right).
\]
Using $\sigma^{2}=\sqrt{\nn d}$, $\nlab\le\frac{\epsilon^{2}\sqrt{\nn d}}{8\log d}$
implies $\epsilon\sqrt{1+\frac{\sigma^{2}}{n}}\ge\sqrt{8\log d}$
and therefore
\[
\E\,\errrob(\learner_{\nlab}[S])\ge\frac{1}{2}\P\left(\left\Vert \mathcal{N}\left(0,I\right)\right\Vert _{\infty}\le\sqrt{8\log d}\right).
\]
Moreover
\[
\P\left(\left\Vert \mathcal{N}\left(0,I\right)\right\Vert _{\infty}\le\sqrt{8\log d}\right)=\left(1-Q\left(\sqrt{8\log d}\right)\right)^{d}\ge\left(1-e^{-4\log d}\right)^{d}\ge1-\frac{1}{d}.
\]
\end{proof}

\subsection{Performance of semisupervised estimator}\label{sec:theory-semisup}

We now consider the semisupervised setting\textemdash our primary
object of study in this paper. We consider the self-training estimator
that in the first stage uses $\nlab\ge \nn$ labeled examples to construct
\[
\labest\defeq\estim_{\nlab},
\]
and then uses it to produce pseudo-labels 
\[
\yun_{i}=\mathrm{sign}\left(\xun_{i}^{\top}\labest\right)
\]
 for the $\nunlab$ unlabeled data points $\xun_{1},\ldots,\xun_{\nunlab}$.
In the second and final stage of self-training, we employ the same
simple learning rule on the pseudo-labeled data and construct
\[
\finalest\defeq\frac{1}{\nunlab}\sum_{i=1}^{\nunlab}\yun_{i}\xun_{i}.
\]
The following result shows a high-probability
bound on $\frac{\mu^{\top}\semisupest}{\sigma\left\Vert \semisupest\right\Vert }$,
analogous to the one obtained for the fully supervised estimator in
Lemma \ref{lem:supervised-perf} (with different constant factors).
\begin{lemma}
\label{lem:semisupervised-perf}
There exist numerical constants $\tilde{c}_{0},\tilde{c}_{1},\tilde{c}_{2}>0$
such that under parameter setting (\ref{eq:param-set}) and $d/n_{0}>\tilde{c}_{0}$,
\[
\frac{\mu^{\top}\finalest}{\sigma\left\Vert \finalest\right\Vert }\ge\left(\sqrt{\frac{\nn}{d}}+\frac{72\nn}{\nunlab}\left(1+\tilde{c}_{1}\left(\frac{\nn}{d}\right)^{-1/4}\right)\right)^{-1/2}
\]
with probability $\ge1-e^{-\tilde{c}_{2}\min\left\{ \nunlab,\nn\left(d/\nn\right)^{1/4},\sqrt{d/\nn}\right\}}$.
\end{lemma}
\begin{proof}
The proof follows a similar argument to the one used to prove Lemma
\ref{lem:supervised-perf}, except now we have to to take care of
the fact that the noise component in $\finalest$ is not entirely
Gaussian. Let $b_{i}$ be the indicator that the $i$th pseudo-label
is incorrect, so that $\xun_{i}\sim\mathcal{N}\left(\left(1-2b_{i}\right)\yun_{i}\mu,\sigma^{2}I\right)$,
and let 
\[
\gamma\defeq\frac{1}{\nunlab}\sum_{i=1}^{\nunlab}\left(1-2b_{i}\right)\in[-1,1].
\]
We may write the final estimator as
\[
\semisupest=\frac{1}{\nunlab}\sum_{i=1}^{\nunlab}\yun_{i}\xun_{i}=\gamma\mu+\frac{1}{\nunlab}\sum_{i=1}^{\nunlab}\yun_{i}\eps_{i}
\]
where $\eps_{i}\sim\mathcal{N}\left(0,\sigma^{2}I\right)$ independent
of each other. Defining
\[
\tilde{\delta}\defeq\semisupest-\gamma\mu
\]
we have the decomposition and bound
\begin{align}
\frac{\left\Vert \semisupest\right\Vert ^{2}}{\left(\mu^{\top}\semisupest\right)^{2}} & =\frac{\left\Vert \tilde{\delta}+\gamma\mu\right\Vert ^{2}}{\left(\gamma\left\Vert \mu\right\Vert ^{2}+\mu^{\top}\tilde{\delta}\right)^{2}}=\frac{1}{\left\Vert \mu\right\Vert ^{2}}+\frac{\left\Vert \tilde{\delta}+\gamma\mu\right\Vert ^{2}-\frac{1}{\left\Vert \mu\right\Vert ^{2}}\left(\gamma\left\Vert \mu\right\Vert ^{2}+\mu^{\top}\tilde{\delta}\right)^{2}}{\left(\gamma\left\Vert \mu\right\Vert ^{2}+\mu^{\top}\tilde{\delta}\right)^{2}}\nonumber \\
 & =\frac{1}{\left\Vert \mu\right\Vert ^{2}}+\frac{\|\tilde{\delta}\|^{2}-\frac{1}{\left\Vert \mu\right\Vert ^{2}}\left(\mu^{\top}\tilde{\delta}\right)^{2}}{\left(\gamma\left\Vert \mu\right\Vert ^{2}+\mu^{\top}\tilde{\delta}\right)}\le\frac{1}{\left\Vert \mu\right\Vert ^{2}}+\frac{\|\tilde{\delta}\|^{2}}{\left\Vert \mu\right\Vert ^{4}\left(\gamma+\frac{1}{\left\Vert \mu\right\Vert ^{2}}\mu^{\top}\tilde{\delta}\right)^{2}}.\label{eq:semisup-deomp}
\end{align}

To write down concentration bounds for $\|\tilde{\delta}\|^{2}$ and
$\mu^{\top}\tilde{\delta}$ we must address their non-Gaussianity.
To do so, choose a coordinate system such that the first coordinate
is in the direction of $\labest$, and let $v\supind i$ denote the
$i$th entry of vector $v$ in this coordinate system. Then
\[
\yun_{i}=\mathrm{sign}\left(\xun_{i}\supind 1\right)=\mathrm{sign}\left(\mu\supind 1+\eps_{i}\supind 1\right).
\]
Consequently, $\eps_{i}\supind j$ is independent of $\yun_{i}$ for
all $i$ and $j\ge2$, so that $\yun_{i}\eps_{i}\supind j\sim\mathcal{N}\left(0,\sigma^{2}\right)$
and $\frac{1}{\nunlab}\sum_{i=1}^{\nunlab}\yun_{i}\eps_{i}\supind j\sim\mathcal{N}\left(0,\sigma^{2}/\nunlab\right)$
and
\[
\sum_{j=2}^{d}\left(\frac{1}{\nunlab}\sum_{i=1}^{\nunlab}\yun_{i}\eps_{i}\supind j\right)^{2}\sim\frac{\sigma^{2}}{\nunlab}\chi_{d-1}^{2}.
\]
 Moreover, we have by Cauchy\textendash Schwarz
\[
\left(\frac{1}{\nunlab}\sum_{i=1}^{\nunlab}\yun_{i}\eps_{i}\supind 1\right)^{2}\le\frac{1}{\nunlab^{2}}\left(\sum_{i=1}^{\nunlab}\yun_{i}^{2}\right)\left(\sum_{i=1}^{\nunlab}\left[\eps_{i}\supind 1\right]^{2}\right)=\frac{1}{\nunlab}\sum_{i=1}^{\nunlab}\left[\eps_{i}\supind 1\right]^{2}\sim\frac{\sigma^{2}}{\nunlab}\chi_{\nunlab}^{2}.
\]
Therefore, since $\|\tilde{\delta}\|^{2}=\sum_{j=1}^{d}\left(\frac{1}{\nunlab}\sum_{i=1}^{\nunlab}\yun_{i}\eps_{i}\supind j\right)^{2}$,
we have by the union bound
\begin{equation}
\P\left(\|\tilde{\delta}\|^{2}\ge2\frac{\sigma^{2}}{\nunlab}\left(d-1+\nunlab\right)\right)\le\P\left(\chi_{\nunlab}^{2}\ge2\nunlab\right)+\P\left(\chi_{d-1}^{2}\ge2\left(d-1\right)\right)\le e^{-\nunlab/8}+e^{-\left(d-1\right)/8}.\label{eq:delta-til-norm-bound}
\end{equation}
The same technique also yields a crude bound on $\mu^{\top}\tilde{\delta}=\frac{1}{\nunlab}\sum_{i=1}^{\nunlab}\yun_{i}\mu^{\top}\eps_{i}$.
Namely, we have
\[
\left(\mu^{\top}\tilde{\delta}\right)^{2}\le\frac{1}{\nunlab^{2}}\left(\sum_{i=1}^{\nunlab}\yun_{i}^{2}\right)\left(\sum_{i=1}^{\nunlab}\left(\mu^{\top}\eps_{i}\right)^{2}\right)=\frac{1}{\nunlab}\sum_{i=1}^{\nunlab}\left(\mu^{\top}\eps_{i}\right)^{2}\sim\frac{\sigma^{2}\left\Vert \mu\right\Vert ^{2}}{\nunlab}\chi_{\nunlab}^{2}
\]
and therefore 
\[
\P\left(\left|\mu^{\top}\tilde{\delta}\right|\ge\sqrt{2}\sigma\left\Vert \mu\right\Vert \right)=\P\left(\left|\mu^{\top}\tilde{\delta}\right|^{2}\ge2\sigma^{2}\left\Vert \mu\right\Vert ^{2}\right)\le e^{-\nunlab/8}.
\]
Finally, we need to argue that $\gamma$ is not too small. Recall
that $\gamma=\frac{1}{\nunlab}\sum_{i=1}^{\nunlab}\left(1-2b_{i}\right)$ where
$b_{i}$ is the indicator that $\yun_{i}$ is incorrect and therefore
\[
\E\left[\gamma\mid\labest\right]=1-2\err(f_{\labest}),
\]
so we expect $\gamma$ to be reasonably large as long as $\err(f_{\labest})<\frac{1}{2}$.
Indeed,
\begin{align*}
\P\left(\gamma<\frac{1}{6}\right) & =\P\left(\frac{1}{\nunlab}\sum_{i=1}^{\nunlab}\left(1-2b_{i}\right)<\frac{1}{6}\right)\\
 & \le\P\left(\err(f_{\labest})>\frac{1}{3}\right)+\P\left(\frac{1}{\nunlab}\sum_{i=1}^{\nunlab}b_{i}<\frac{5}{12}\mid\err(f_{\labest})\le\frac{1}{3}\right).
\end{align*}
Note that
\[
\frac{1}{3}\ge Q\left(\frac{1}{2}\right)\ge Q\left(\left[2\left(1+\sqrt{\nn/d}\right)\right]^{-1/2}\right)
\]
Therefore, by Lemma \ref{lem:supervised-perf}, for sufficiently large
$d/\nn$,
\[
\P\left(\err(f_{\labest})>\frac{1}{3}\right)\le e^{-c\cdot\min\left\{ \sqrt{d/\nn},\nn\left(d/\nn\right)^{1/4}\right\} }
\]
for some constant $c$. Moreover, by Bernoulli concentration (Hoeffding's
inequality) we have that
\[
\P\left(\frac{1}{\nunlab}\sum_{i=1}^{\nunlab}b_{i}<\frac{5}{12}\mid\err(f_{\labest})\le\frac{1}{3}\right)\le e^{-2\nunlab\left(\frac{5}{12}-\frac{1}{3}\right)^{2}}=e^{-\nunlab/72}.
\]
Define the event,
\[
\mathcal{E}=\left\{ \|\tilde{\delta}\|^{2}\ge2\frac{\sigma^{2}}{\nunlab}\left(d+\nunlab\right),\ \left|\mu^{\top}\tilde{\delta}\right|\le\sqrt{2}\sigma\left\Vert \mu\right\Vert \ \text{ and }\gamma\ge\frac{1}{6}\right\} ;
\]
by the preceding discussion,
\[
\P\left(\mathcal{E}^{C}\right)\le2e^{-\nunlab/8}+e^{-\left(d-1\right)/8}+e^{-c\cdot\min\left\{ \sqrt{d/\nn},\nn\left(d/\nn\right)^{1/4}\right\} }+e^{-\nunlab/72}\le e^{-\tilde{c}_{2}\min\left\{ \nunlab,\sqrt{d/\nn},\nn\left(d/\nn\right)^{1/4}\right\} }.
\]
Moreover, by the bound (\ref{eq:semisup-deomp}), $\mathcal{E}$ implies
\begin{align*}
\frac{\left\Vert \semisupest\right\Vert ^{2}}{\left(\mu^{\top}\semisupest\right)^{2}} & \le\frac{1}{\left\Vert \mu\right\Vert ^{2}}+\frac{2\sigma^{2}\left(d+\nunlab\right)}{\nunlab\left\Vert \mu\right\Vert ^{4}\left(\frac{1}{6}-\frac{\sqrt{2}\sigma}{\left\Vert \mu\right\Vert }\right)^{2}}.
\end{align*}
Substituting $\sigma=\left(\nn d\right)^{1/4}$ and $\left\Vert \mu\right\Vert =\sqrt{d}$
 and multiplying by $\sigma^{2}$ gives
\begin{align*}
\frac{\sigma^{2}\left\Vert \semisupest\right\Vert ^{2}}{\left(\mu^{\top}\semisupest\right)^{2}} & \le\sqrt{\frac{\nn}{d}}+\frac{2\left(\nn d\right)\left(d+\nunlab\right)}{\nunlab d^{2}\left(\frac{1}{6}-\sqrt{2}\left(\frac{\nn}{d}\right)^{1/4}\right)^{2}}\\
 & \le\sqrt{\frac{\nn}{d}}+\frac{72\nn}{\nunlab}\left(1+\tilde{c}_{1}\left(\frac{\nn}{d}\right)^{-1/4}\right)
\end{align*}
for appropriate $\tilde{c}_{1}$ and sufficiently large $d/\nn$.
As argued in Lemma \ref{lem:supervised-perf}, the event $\mathcal{E}$
already implies $\mu^{\top}\semisupest\ge0$, and therefore the result
follows.
\end{proof}
Lemma \ref{lem:semisupervised-perf} immediately gives a sample
complexity upper bound for the self-training classifier $\finalest$
trained with $\nlab$ labeled data and $\nunlab$ unlabeled data.
\restateSemisupSamples*
\begin{proof}
We take $\tilde{r}$ sufficiently large so that by Lemma \ref{lem:semisupervised-perf}
we have, using $\sigma=\left(\nn d\right)^{1/4}$ and $\epsilon<\frac{1}{2},$
\[
\frac{\mu^{\top}\estim_{\nlab}}{\sigma\left\Vert \estim_{\nlab}\right\Vert }\ge\frac{1}{\sqrt{2\left(\frac{72\nn}{n}+\sqrt{\frac{n_{0}}{d}}\right)}}\ge\frac{1}{\sqrt{2\left(\frac{\nn}{4\epsilon^{2}\sqrt{\nn d}}+\frac{1}{4\epsilon^{2}}\sqrt{\frac{n_{0}}{d}}\right)}}=\sqrt{2}\epsilon\left(\frac{d}{\nn}\right)^{1/4}
\]
with probability $\ge1-e^{-\tilde{c}_{2}\min\left\{ \nunlab,\nn\left(d/\nn\right)^{1/4},\sqrt{d/\nn}\right\} }\ge1-e^{-\epsilon^{2}\tilde{c}_{2}\left(d/\nn\right)^{1/4}}$.
Therefore, using the expression (\ref{eq:robust-err-closed-form})
and $\sigma=\left(\nn d\right)^{1/4}$, we have (using $n\ge\epsilon^{2}\left(d/\nn\right)^{1/4}$)
\[
\E_{\estim_{n}}\errrob\left(f_{\estim_{n}}\right)\le Q\left(\left[\sqrt{2}-1\right]\epsilon(d/\nn)^{1/4}\right)+e^{-\epsilon^{2}\tilde{c}_{2}\left(d/\nn\right)^{1/4}}\le 10^{-3},
\]
for sufficiently large $\tilde{r}$.
\end{proof}

\subsection{Performance in the presence of irrelevant data}\label{sec:app-irrelevant}

To model the presence of irrelevant data, we consider a slightly different
model where, for $\alpha\in\left(0,1\right)$, $\alpha\nunlab$ of
the unlabeled data are distributed as $\mathcal{N}\left(y_{i}\mu,\sigma^{2}I\right)$
as before, while the other $\left(1-\alpha\right)\nunlab$ unlabeled
data are drawn from $\mathcal{N}\left(0,\sigma^{2}I\right)$ (with no signal
component). We note that similar conclusions would hold if we let the irrelevant unlabeled data be drawn from $\mathcal{N}\left(\mu_{2},\sigma^{2}I\right)$
for some $\mu_{2}$ such that $\left|\mu^{\top}\mu_{2}\right|$ is
sufficiently small, for example $\mu_{2}\sim\mathcal{N}\left(0,I\right)$
independent of $\mu$. We take $\mu_{2}=0$ to simplify the presentation.

To understand the impact of irrelevant data we need to establish two
statements. First, we would like to show that adversarial robustness
is still possible given sufficiently large $\nunlab$, namely $\Omega\left(\epsilon^{2}\sqrt{\nn d}/\alpha^{2}\right);$
a factor $1/\alpha$ more relevant data then what our previous result
required. Second, we wish to show that this upper bound is tight.
That is, we would like to show that self-training with $\nn$ labeled
data and $O\left(\epsilon^{2}\sqrt{\nn d}/\alpha^{2}\right)$ $\alpha$-relevant
unlabeled data fails to achieve robustness. We make these statements
rigorous in the following.
\begin{theorem}
\label{thm:irrelevant}There exist numerical constants $c$ and $r$
such the following holds under parameter setting (\ref{eq:param-set})
, $\alpha$-fraction of relevant unlabeled data and $\min\{\epsilon^2/\log d, \alpha^2\}\sqrt{d/\nn}>r$.
First,
\[
\nunlab\ge\nn \cdot \frac{288 \epsilon^{2}}{\alpha^2}\sqrt{\frac{d}{\nn}}
\Rightarrow
\E_{\finalest}\errrob\left(f_{\finalest}\right)\le 10^{-3}.
\]
Second, there exists
$\mu\in\R^{d}$ for which
\[
\nunlab\le \nn \frac{c\cdot \epsilon^{2}}{\alpha^2}\sqrt{\frac{d}{\nn}}
\Rightarrow\E_{\finalest}\errrob\left(f_{\finalest}\right)\ge\frac{1}{2}\left(1-\frac{1}{d}\right).
\]
\end{theorem}
Examining the robust error probability (\ref{eq:robust-err-closed-form}),
establishing these results requires upper and lower bounds on the
quantity $\frac{\mu^{\top}\finalest}{\|\finalest\|}$ as well as a
lower bound on $\frac{\|\finalest\|_{1}}{\|\finalest\|}$. We begin
with the former, which is a two-sided version of Lemma \ref{lem:semisupervised-perf}.
\begin{lemma}
\label{lem:ir-semisup-perf}There exist numerical constants $\bar{c}_{0},\underline{c}_{1},\bar{c}_{1},\bar{c}_{2}$
such that under parameter setting (\ref{eq:param-set}), $\alpha$-fraction
of relevant unlabeled data and $d/n_{0}>\bar{c}_{0}/\alpha^{4}$,
\[
\left(\sqrt{\frac{\nn}{d}}+\frac{72\nn}{\alpha^{2}\nunlab}\left(1+\frac{\bar{c}_{1}}{\alpha}\left(\frac{\nn}{d}\right)^{-1/4}\right)\right)^{-1/2}\le\frac{\mu^{\top}\finalest}{\sigma\left\Vert \finalest\right\Vert }\le\left(\sqrt{\frac{\nn}{d}}+\frac{\nn}{2\alpha^{2}\nunlab}\left(1-\frac{\underline{c}_{1}}{\alpha}\left(\frac{\nn}{d}\right)^{-1/4}\right)\right)^{-1/2},
\]
with probability $\ge1-e^{-\bar{c}_{2}\min\left\{ \alpha \nunlab,\nn\left(d/\nn\right)^{1/4},\sqrt{d/\nn}\right\} }$.
\end{lemma}
The proof of Lemma \ref{lem:ir-semisup-perf} is technical and very
similar to the proof of Lemma \ref{lem:semisupervised-perf}, so we
defer it Section \ref{subsec:app-irrelevant-proofs-perf}. We remark that
in the regime $\nunlab\ge\alpha^{-2}$, a more careful concentration
argument would allow us to remove $\alpha$ from the condition $d/n_{0}>\bar{c}_{0}/\alpha^{4}$
and the high order terms of the form $\frac{c_{1}}{\alpha}\left(\frac{\nn}{d}\right)^{-1/4}$
in Lemma \ref{lem:ir-semisup-perf}.

Next, argue that\textemdash at least for certain values of $\mu$\textemdash the
self-training estimator $\finalest$ is dense in the sense that $\|\finalest\|_{1}/\|\finalest\|$
is within a constant of $\sqrt{d}$.
\begin{lemma}
\label{lem:sparsity-lb}Let $\mu$ be the all-ones vector. There exist
constants $k_{1},k_{2}$ such that under parameter setting (\ref{eq:param-set}),
$\alpha$-fraction of relevant unlabeled data and $d\ge\nunlab\ge30$,
\[
\frac{\|\finalest\|_{1}}{\|\finalest\|}\ge k_{1}\sqrt{d}\ \text{with probability }\ge1-e^{-k_{0}\min\left\{ \nunlab,d\right\} }.
\]
\end{lemma}
We prove Lemma \ref{lem:sparsity-lb} in Section \ref{subsec:app-irrelevant-proofs-sparsity}.
Armed with the necessary bounds, we prove Theorem \ref{thm:irrelevant}. 
\begin{proof}[Proof of Theorem \ref{thm:irrelevant}]
The case $\nunlab\ge\frac{288}{\alpha^{2}}\epsilon^{2}\sqrt{\nn d}$
follows from Lemma \ref{lem:ir-semisup-perf} using an argument identical
to the one used in the proof of Theorem \ref{prop:supervised-samples}.
To show the case $\nunlab\le\frac{c}{\alpha^{2}}\epsilon^{2}\sqrt{\nn d}$,
we take $r$ such that $\frac{\underline{c}_{1}}{\alpha}\left(\frac{\nn}{d}\right)^{-1/4}<\frac{1}{2}$
and apply the upper bound in Lemma \ref{lem:ir-semisup-perf} to obtain
\[
\frac{\mu^{\top}\finalest}{\sigma\left\Vert \finalest\right\Vert }\le2\alpha\sqrt{\nunlab/\nn}\le\sqrt{c}\epsilon\left(\frac{d}{\nn}\right)^{1/4}
\]
with probability $1-e^{-\bar{c}_{2}\min\left\{ \alpha\nunlab,\nn\left(d/\nn\right)^{1/4},\sqrt{d/\nn}\right\} }$.
Next by Lemma \ref{lem:sparsity-lb}, we have
\[
\frac{\|\finalest\|_{1}}{\|\finalest\|}\ge\frac{\sqrt{d}}{k_{1}}
\]
 with probability at least $1-e^{-k_{0}\min\left\{ \nunlab,d\right\} }.$
Therefore, taking $c\le\frac{1}{k_{1}^2}$, we have
and using $\sigma=\left(\nn d\right)^{1/4}$ and the expression (\ref{eq:robust-err-closed-form})
for the robust error probability, we have
\begin{align*}
\E_{\finalest}\errrob\left(f_{\finalest}\right) & \ge\frac{1}{2}\P\left(\frac{\mu^{\top}\finalest}{\sigma\left\Vert \finalest\right\Vert }-\frac{\|\finalest\|_{1}}{\sigma\|\finalest\|}\le0\right)\\
 & \ge\frac{1}{2}\left(1-e^{-\bar{c}_{2}\min\left\{ \alpha\nunlab,\nn\left(d/\nn\right)^{1/4},\sqrt{d/\nn}\right\} }-e^{-k_{0}\min\left\{ \nunlab,d\right\} }\right).
\end{align*}
Finally, we may assume without loss of generality $\alpha\nunlab\ge\frac{\epsilon^{2}\sqrt{\nn d}}{8\log d}-\nn$
because otherwise the result holds by Theorem \ref{thm:lower-bound}.
Using $\sqrt{d/\nn}\ge\epsilon^{-2}r\log d$ and taking $r$
sufficiently large, we have that $\alpha\nunlab\ge\frac{\epsilon^{2}\sqrt{\nn d}}{16\log d}$.
Therefore, $e^{-\bar{c}_{2}\min\left\{ \alpha\nunlab,\nn\left(d/\nn\right)^{1/4},\sqrt{d/\nn}\right\} }+e^{-k_{0}\min\left\{ \nunlab,d\right\} }\le\frac{1}{d}$
for sufficiently large $r$. 
\end{proof}

\subsubsection{\label{subsec:app-irrelevant-proofs-perf}Proof of Lemma \ref{lem:ir-semisup-perf}}

The proof is largely the same as the proof of Lemma \ref{lem:semisupervised-perf}.
We redefine $b_{i}$ to be the indicator of $\yun_{i}$ being incorrect
when $i$ is relevant, and $1/2$ when it is irrelevant. Then, $\xun_{i}\sim\mathcal{N}\left(\left(1-2b_{i}\right)\yun_{i}\mu,\sigma^{2}I\right)$,
and with 
\[
\gamma\defeq\frac{1}{\alpha\nunlab}\sum_{i=1}^{\nunlab}\left(1-2b_{i}\right)\in[-1,1]
\]
we may write the final classifier as
\[
\semisupest=\frac{1}{\nunlab}\sum_{i=1}^{\nunlab}\yun_{i}\xun_{i}=\alpha\gamma\mu+\frac{1}{\nunlab}\sum_{i=1}^{\nunlab}\yun_{i}\eps_{i}
\]
where $\eps_{i}\sim\mathcal{N}\left(0,\sigma^{2}I\right)$ independent
of each other. Defining
\[
\tilde{\delta}\defeq\semisupest-\alpha\gamma\mu
\]
we have the decomposition and bound
\begin{align}
\frac{\left\Vert \semisupest\right\Vert ^{2}}{\left(\mu^{\top}\semisupest\right)^{2}} & =\frac{1}{\left\Vert \mu\right\Vert ^{2}}+\frac{\|\tilde{\delta}\|^{2}-\frac{1}{\left\Vert \mu\right\Vert ^{2}}\left(\mu^{\top}\tilde{\delta}\right)^{2}}{\left(\alpha\gamma\left\Vert \mu\right\Vert ^{2}+\mu^{\top}\tilde{\delta}\right)^{2}}\le\frac{1}{\left\Vert \mu\right\Vert ^{2}}+\frac{\|\tilde{\delta}\|^{2}}{\left\Vert \mu\right\Vert ^{4}\left(\alpha\gamma+\frac{1}{\left\Vert \mu\right\Vert ^{2}}\mu^{\top}\tilde{\delta}\right)^{2}}.\label{eq:ir-semisup-deomp}
\end{align}
As argued in the proof of Lemma (\ref{lem:semisupervised-perf}),
\[
\P\left(\|\tilde{\delta}\|^{2}\ge2\frac{\sigma^{2}}{\nunlab}\left(d-1+\nunlab\right)\right)\le e^{-\nunlab/8}+e^{-\left(d-1\right)/8}.
\]
and
\[
\P\left(\left|\mu^{\top}\tilde{\delta}\right|\ge\sqrt{2}\sigma\left\Vert \mu\right\Vert \right)\le e^{-\nunlab/8}.
\]
Moreover, $\gamma$ is exactly the average of $1-2b_{i}$ over the
relevant data, and therefore, as argued in Lemma (\ref{lem:semisupervised-perf}),
\[
\P\left(\gamma<\frac{1}{6}\right)\le e^{-\alpha\nunlab/72}+e^{-c\cdot\min\left\{ \sqrt{d/\nn},\nn\left(d/\nn\right)^{1/4}\right\} }.
\]
Under the event $\mathcal{E}=\left\{ \|\tilde{\delta}\|^{2}\le2\frac{\sigma^{2}}{\nunlab}\left(d+\nunlab\right),\ \left|\mu^{\top}\tilde{\delta}\right|\le\sqrt{2}\sigma\left\Vert \mu\right\Vert \ \text{ and }\gamma\ge\frac{1}{6}\right\} $,
\[
\frac{\left\Vert \semisupest\right\Vert ^{2}}{\left(\mu^{\top}\semisupest\right)^{2}}\le\frac{1}{\left\Vert \mu\right\Vert ^{2}}+\frac{2\sigma^{2}\left(d+\nunlab\right)}{\alpha^{2}\nunlab\left\Vert \mu\right\Vert ^{4}\left(\frac{1}{6}-\frac{\sqrt{2}\sigma}{\alpha\left\Vert \mu\right\Vert }\right)^{2}}.
\]
Substituting $\left\Vert \mu\right\Vert ^{2}=d$ and $\sigma^{2}=\sqrt{\nn d}$
and multiplying by $\sigma^{2}$, we have
\[
\frac{\sigma^{2}\left\Vert \semisupest\right\Vert ^{2}}{\left(\mu^{\top}\semisupest\right)^{2}}\le\sqrt{\frac{\nn}{d}}+\frac{2\left(\nn d\right)\left(d+\nunlab\right)}{\alpha^{2}\nunlab d^{2}\left(\frac{1}{6}-\frac{\sqrt{2}}{\alpha}\left(\frac{\nn}{d}\right)^{1/4}\right)^{2}}\le\sqrt{\frac{\nn}{d}}+\frac{72\nn}{\alpha^{2}\nunlab}\left(1+\frac{\bar{c}_{1}}{\alpha}\left(\frac{\nn}{d}\right)^{1/4}\right)
\]
for appropriate $\bar{c}_{1}$ and $\alpha^{4}\left(d/\nn\right)$
sufficiently large, which also implies $\mu^{\top}\semisupest\ge0$.
To obtain the other direction of the bound, we note that in the coordinate
system where the first coordinate is in the direction of $\intest$,
\[
\|\tilde{\delta}\|^{2}\ge\sum_{j=2}^{d}\left(\frac{1}{\nunlab}\sum_{i=1}^{\nunlab}\yun_{i}\eps_{i}\supind j\right)^{2}\sim\frac{\sigma^{2}}{\nunlab}\chi_{d-1}^{2}
\]
and therefore
\[
\P\left(\|\tilde{\delta}\|^{2}\le\frac{1}{2}\frac{\sigma^{2}}{\nunlab}\left(d-1\right)\right)\le e^{-\nunlab/32}.
\]
Therefore, under $\mathcal{E}'=\left\{ \|\tilde{\delta}\|^{2}\ge\frac{1}{2}\frac{\sigma^{2}}{\nunlab}\left(d-1\right),\ \left|\mu^{\top}\tilde{\delta}\right|\le\sqrt{2}\sigma\left\Vert \mu\right\Vert \ \text{ and }\gamma\ge\frac{1}{6}\right\} $, substituting
into (\ref{eq:ir-semisup-deomp}) we have
\[
\frac{\left\Vert \semisupest\right\Vert ^{2}}{\left(\mu^{\top}\semisupest\right)^{2}}\ge\frac{1}{\left\Vert \mu\right\Vert ^{2}}+\frac{\sigma^{2}\left(d-1\right)}{2\alpha^{2}\nunlab\left\Vert \mu\right\Vert ^{4}\left(1+\frac{\sqrt{2}\sigma}{\alpha\left\Vert \mu\right\Vert }\right)^{2}}-\frac{2\sigma^{2}}{\alpha^{2}\nunlab\left\Vert \mu\right\Vert ^{4}\left(\frac{1}{6}-\frac{\sqrt{2}\sigma}{\alpha\left\Vert \mu\right\Vert }\right)^{2}}.
\]
Substituting $\left\Vert \mu\right\Vert ^{2}=d$ and $\sigma^{2}=\sqrt{\nn d}$
and multiplying by $\sigma^{2}$, we have
\begin{align*}
\frac{\sigma^{2}\left\Vert \semisupest\right\Vert ^{2}}{\left(\mu^{\top}\semisupest\right)^{2}} & \ge\sqrt{\frac{\nn}{d}}+\frac{\left(\nn d\right)\left(d-1\right)}{2\alpha^{2}\nunlab d^{2}\left(1+\frac{\sqrt{2}}{\alpha}\left(\frac{\nn}{d}\right)^{1/4}\right)^{2}}-\frac{2\nn d}{\alpha^{2}\nunlab d^2\left(\frac{1}{6}-\frac{\sqrt{2}}{\alpha}\left(\frac{\nn}{d}\right)^{1/4}\right)^{2}}\\
 & \ge\sqrt{\frac{\nn}{d}}+\frac{\nn}{2\alpha^{2}n}\left(1-\frac{\underline{c}_{1}}{\alpha}\left(\frac{\nn}{d}\right)^{1/4}\right)
\end{align*}
for appropriate $\underline{c}_{1}$ and $\alpha^{4}\left(d/\nn\right)$
sufficiently large. Both inequalities hold under $\mathcal{E}\cup\mathcal{E}'$,
which by the preceding discussion fails with probability at most
\[
e^{-\nunlab/32}+e^{-\alpha\nunlab/72}+e^{-c\cdot\min\left\{ \sqrt{d/\nn},\nn\left(d/\nn\right)^{1/4}\right\} }+2e^{-\nunlab/8}+e^{-\left(d-1\right)/8}\le e^{-\bar{c}_{2}\min\left\{ \alpha\nunlab,\nn\left(d/\nn\right)^{1/4},\sqrt{d/\nn}\right\} }
\]
for appropriate $\bar{c}_{2}$.

\subsubsection{\label{subsec:app-irrelevant-proofs-sparsity}Proof of Lemma \ref{lem:sparsity-lb}}

As in the proof of Lemma \ref{lem:ir-semisup-perf} we define $b_{i}$
to be the indicator of $\yun_{i}$ being incorrect when $i$ is relevant,
and $1/2$ otherwise. So that we may write
\[
\semisupest=\alpha\gamma\mu+\frac{1}{\nunlab}\sum_{i=1}^{\nunlab}\yun_{i}\eps_{i},\ \text{where }\gamma\defeq\frac{1}{\alpha\nunlab}\sum_{i=1}^{\nunlab}\left(1-2b_{i}\right)\ \text{and }\eps_{i}\sim\mathcal{N}\left(0,\sigma^{2}I\right).
\]
We further decompose $\eps_{i}$ to components orthogonal and parallel
to $\intest$, $\eps_{i}=\eps_{i}^{\perp}+\eps_{i}^{\parallel}$ ,
so that $\eps_{i}^{\perp}\sim\mathcal{N}\left(0,\sigma^{2}\left(1-\pi\pi^{T}\right)\right)$
and $\eps_{i}^{\parallel}\sim\mathcal{N}\left(0,\sigma^{2}\pi\pi^{T}\right)$,
where $\pi$ is a unit vector parallel to $\intest$. We note that
$\eps_{i}^{\perp}$ is independent of $\yun_{i}$ and of $\eps_{i}^{\parallel}$.
Let $\left[v\right]_{j}$ denote the $j$th coordinate of vector $v$
in the standard basis, such that
\[
\|\finalest\|_{1}=\sum_{j=1}^{d}|[\finalest]_{j}|=\sum_{j=1}^{d}\left|\alpha\gamma+\frac{1}{\nunlab}\sum_{i=1}^{\nunlab}\yun_{i}\left[\eps_{i}^{\perp}\right]_{j}+\frac{1}{\nunlab}\sum_{i=1}^{\nunlab}\yun_{i}\left[\eps_{i}^{\parallel}\right]_{j}\right|.
\]
We define
\[
J=\left\{ j\mid\left[\pi\right]_{j}^{2}\le\frac{2}{d}\right\} ,\text{ so that }\left|J\right|\ge\frac{d}{2}
\]
since $\pi$ is a unit vector. For every $j$, $\frac{1}{\nunlab}\sum_{i=1}^{\nunlab}\yun_{i}\left[\eps_{i}^{\perp}\right]_{j}\sim\mathcal{N}\left(0,\frac{\sigma^{2}}{\nunlab}\left(1-\left[\pi\right]_{j}^{2}\right)\right)$
and therefore for every $j\in J$
\[
\P\left(\frac{1}{\nunlab}\sum_{i=1}^{\nunlab}\yun_{i}\left[\eps_{i}^{\perp}\right]_{j}>\frac{\sigma}{\sqrt{\nunlab}}\left(1-\frac{2}{d}\right)\right)\ge Q\left(1\right)\ge\frac{1}{8}.
\]
Moreover, by Cauchy\textendash Schwarz
\[
\left(\frac{1}{\nunlab}\sum_{i=1}^{\nunlab}\yun_{i}\left[\eps_{i}^{\parallel}\right]_{j}\right)^{2}\le\frac{1}{\nunlab}\sum_{i=1}^{\nunlab}\left[\eps_{i}^{\parallel}\right]_{j}^{2}\sim\frac{\sigma^{2}\left[\pi\right]_{j}^{2}}{\nunlab}\chi_{\nunlab}
\]
and therefore for every $j\in J$ 
\[
\P\left(\left|\frac{1}{\nunlab}\sum_{i=1}^{\nunlab}\yun_{i}\left[\eps_{i}^{\parallel}\right]_{j}\right|>\frac{2\sigma}{\sqrt{\nunlab d}}\right)\le e^{-\nunlab/8}.
\]
Therefore, with probability at at least $Q\left(1\right)-e^{-\nunlab/8}\ge\frac{1}{10}$
for $\nunlab\ge30$,
\begin{equation}\label{eq:l1-lb}
\left|\alpha\gamma+\frac{1}{\nunlab}\sum_{i=1}^{\nunlab}\yun_{i}\left[\eps_{i}^{\perp}\right]_{j}+\frac{1}{\nunlab}\sum_{i=1}^{\nunlab}\yun_{i}\left[\eps_{i}^{\parallel}\right]_{j}\right|\ge\alpha\gamma+\frac{\sigma}{\sqrt{\nunlab}}\left(\sqrt{1-\frac{2}{d}}-\frac{2}{\sqrt{d}}\right)\ge\alpha\gamma+\frac{\sigma}{2\sqrt{\nunlab}}
\end{equation}
for $d\ge20.$ We implicitly assumed here $\gamma>0$, which we have previously argued to hold with high probability. However, an analogous bound holds if $\gamma\le0$ and so we don't need to take the this into account here.

Let $J'$ be the random set of coordinates for which
the inequality~\eqref{eq:l1-lb} holds; we have $\|\finalest\|_{1}\ge d\left(\alpha\gamma+\frac{\sigma^{2}}{2\nunlab}\right)\left|J'\right|\ge d\left(\alpha\gamma+\frac{\sigma^{2}}{2\nunlab}\right)\left|J'\cap J\right|$.
Moreover, by the above discussion $\left|J'\cap J\right|$ is binomial
with at least $d/2$ trials and success probability at least $1/10$.
Therefore
\[
\P\left(\left|J'\cap J\right|\le\frac{d}{20}-\frac{d}{40}\right)\le e^{-2\left(d/2\right)/20^{2}}=e^{-d/400}.
\]
And consequently we have
\[
\|\finalest\|_{1}\ge\frac{d}{40}\left(\alpha\gamma+\frac{\sigma^{2}}{2\nunlab}\right)\text{ with probability}\ge 1- e^{-d/400}.
\]
Next we need to argue about the Euclidean norm
\[
\|\finalest\|=\norms{\alpha\gamma\mu+\tilde{\delta}} \le\left\Vert \alpha\gamma\mu\right\Vert +\norms{\tilde{\delta}} =\alpha\gamma\sqrt{d}+\norms{\tilde{\delta}}
\]
with $\tilde{\delta}=\frac{1}{\nunlab}\sum_{i=1}^{\nunlab}\yun_{i}\eps_{i}$.
As argued in Eq. (\ref{eq:delta-til-norm-bound}) in the proof of
Lemma \ref{lem:semisupervised-perf},
\[
\P\left(\|\tilde{\delta}\|^{2}\ge2\frac{\sigma^{2}}{\nunlab}\left(d-1+\nunlab\right)\right)\le e^{-\nunlab/8}+e^{-\left(d-1\right)/8}.
\]

Under $\|\finalest\|_{1}\ge\frac{d}{40}\left(\alpha\gamma+\frac{\sigma^{2}}{2\nunlab}\right)$
and $\|\tilde{\delta}\|^{2}\le2\frac{\sigma^{2}}{\nunlab}\left(d-1+\nunlab\right)\le2\frac{\sigma^{2}}{\nunlab}\left(\sqrt{d}+\sqrt{\nunlab}\right)^{2}$
we have
\[
\frac{\|\finalest\|_{1}}{\|\finalest\|}\ge\frac{d}{40\sqrt{d}}\frac{\alpha\gamma+\frac{1}{2}\sigma\nunlab^{-1/2}}{\alpha\gamma+\sqrt{2}\sigma\nunlab^{-1/2}+\sqrt{2}\sigma d^{-1/2}}\ge\frac{\sqrt{d}}{k_{1}}
\]
for $d\ge\nunlab$ and $k_{1}=160\sqrt{2}$. By the preceding discussion,
this happens with probability at least $1-e^{-\nunlab/8}-e^{-\left(d-1\right)/8}-e^{-d/400}\ge1-e^{k_{1}\min\left\{ \nunlab,d\right\} }$.

\section{CIFAR-10 experimental setup}
\subsection{Training hyperparameters}\label{sec:app-hyper}

Here we describe the training hyperparameters used in our main CIFAR-10 experiments. Our additional experiments in \Cref{sec:app-cifar} use the same hyperparameters unless otherwise mentioned.

\paragraph{Architecture.} 
We use a Wide ResNet 28-10 architecture, as in~\citep{madry2018towards} and similarly to~\citep{zhang2019theoretically}, who use a 34-10 variant.

\paragraph{Robust self-training.} 
We set the regularization weight $\beta=6$ as in~\citep{zhang2019theoretically}. We implicitly set the unlabeled data weight to $\wun=50\text{K}/500\text{K} = 0.1$ by composing every batch from equal parts labeled and unlabeled data.

\paragraph{Adversarial self-training.} 
We compute $x_\texttt{PG}$ exactly as in~\citep{zhang2019theoretically}, with step size $0.007$, $10$ iterations and $\epsilon=8/255$. 

\paragraph{Stability training.} We set the additive noise variance to $\sigma=0.25$. We perform the certification using the randomized smoothing protocol described in~\cite{cohen2019certified}, with parameters $N_0 = 100$, $N=10^4$, $\alpha=10^{-3}$ and noise variance $\sigma=0.25$.

\paragraph{Input normalization.}  We scale each pixel in the input image by $1/255$, to be in $[0,1]$. 

\paragraph{Data augmentation.}  We perform the standard CIFAR-10 data augmentation: a random 4-pixel crop followed by a random horizontal flip.

\paragraph{Optimizer configuration.} We use the hyperparameters~\citep{cubuk2019autoaugment} prescribe for \wrn{28}{10} and CIFAR-10, except for batch size and number of epochs: initial learning rate $0.1$, cosine learning rate annealing~\cite{loshchilov2017sgdr} (without restarts), weight decay $5\cdot10^{-4}$ and SGD with Nesterov momentum $0.9$. To speed up robust training, we doubled the batch size from 128 and 256. (This way, every batch has 128 original CIFAR-10 images and 128 pseudo-labeled imaged). 

\paragraph{Number of gradient steps.}
Since we increase the dataset size by a factor of 10, we expect to require more steps for training to converge. However, due to the high cost of adversarial training (caused by the inner optimization loop), we restricted the training of $\rstat{50K+500K}$ to 39K  gradient steps. This corresponds to 100 CIFAR-10 epochs at batch size 128, which is standard. Stability training is much cheaper (involving only random sampling at each step), and we train $\rstst{50K+500K}$ for 156K gradient steps. Training for longer will likely enhance performance, but probably not dramatically. 

\paragraph{Pseudo-label generation.} 
We used the same model to generate the pseudo-labels in all of our CIFAR-10 robust-training experiments. They were generated by a $\wrn{28}{10}$ which we trained on the CIFAR-10 training set only. The training parameters were exactly like those of the baseline in~\citep{cubuk2019autoaugment} (with standard augmentation). That is to say, training parameters were as above, except we used batch size 128 and ran 200 epochs. The resulting model had 96.0\% accuracy on the CIFAR-10 test set.

\paragraph{Baselines.}

For fully supervised adversarial training the above hyperparameter configuration fails due to overfitting, as we report in detail in~\Cref{sec:app-hyper-comp}. Stability training did not exhibit overfitting, but longer training did not improve results: we trained  $\baseline{st}{50K}$  for 19.5K gradient steps, i.e.\ 100 epochs at batch size 256. We also tried training for 200 and 400 epochs, but saw no improvement when comparing certification results over 10\% of the data.

\subsection{Implementation and running times}

We implement our experiments in PyTorch~\citep{paszke2017automatic}, using open source code from~\citep{zhang2019theoretically,
  cohen2019certified}. We ran all our experiments on Titan Xp GPU's. 
  Training $\rstat{50K+500K}$ took 28 hours on 4 GPU's. Running 
  $\pgdours$ on a $\wrn{28}{10}$ took 30 minutes on 4 GPU's. Training 
  $\rstst{50K+500K}$ took 40 hours on 2 GPU's. Running randomized 
  smoothing certification on a $\wrn{28}{10}$ took 19 hours on 2 GPU's.
\subsection{Tuning attacks}\label{sec:app-attack}
In this section, we provide details on how we tuned the parameters for the projected gradient attack to evaluate adversarially trained models. Let $\pgd{\eta}{\tau}{\rho}$ to denote an attack with step-size $\eta$ performing $\tau$ steps, with $\rho$ restarts.
We tune the attacks to maximally reduce the accuracy of our model $\rstat{50K+500K}$.

For every restart, we start with a different random initialization within the $\ell_\infty$ ball of radius $\epsilon$. At every step of every restart, we check if an adversarial example (i.e.\ an input that causes misclassification in the model) was found. We report the final accuracy as the percentage of examples where no successful adversarial example was found across all steps and restarts. 

We first focus on $\epsilon=8/255$ which is the main size of perturbation of interest in this paper. We report all the numbers with $1$ significant figure and observe around $0.05\%$ variation across multiple runs of the same attack on the same model. 

\paragraph{Step-size.} We use $\rho=5$ restarts and experiment tune the number of steps and the steps size. We tried $20, 40, 60$ steps and step sizes $0.005$, $0.01$ and $0.02$. Table~\ref{table:ss} summarizes the results. 
\begin{table}
  \centering
  \begin{tabular}{cccc}
  	\toprule
    Number of steps $\tau$ & $\eta=0.005$ & $\eta=0.01$ & $\eta=0.02$ \\
    \midrule
    $\tau=20$ & 63.4 & \textbf{62.9} & 62.9 \\
    $\tau=40$ & 62.8 & \textbf{62.5} & 63 \\
    $\tau=60$ & 62.6 & \textbf{62.5} & 62.8 \\
    \bottomrule
  \end{tabular}
  ~
  \caption{Tuning the step-size $\eta$ for the PG attack. Over the range of step numbers considered, $\eta=0.01$ was the most effective against our model.}
  \label{table:ss}
\end{table}
We see that the mid step-size $\eta=0.01$ provided the best accuracy, across the range of steps. We also chose $\tau=40$ for computational benefit, since larger $\tau$ did not seem to provide much gain. We thus obtained the $\pgdours$ configuration with $\eta=0.01, \tau=40, \rho=5$, that we used to test multiple reruns of our model and other models from the literature. 

To compare to previous attacks typically used in the literature: $\pgdmadry$~\citep{madry2018towards} corresponds to $\eta=0.007$,  $\rho=1$ and $\tau=20$, and $\pgdtrades$~\citep{zhang2019theoretically} corresponds to $\eta=0.003$, $\rho=1$ and $\tau=20$, \emph{without random initializations} (attack is always initialized at the test input $x$). We use both more steps and more restarts than used typically, and also tune the step-size to observe that $\eta=0.01$ was worse for our model than the smaller step-size of $0.007$.

\paragraph{Number of restarts.}
PG attack which performs projected gradient method on a non-convex objective is typically sensitive to the exact initialization point. Therefore, multiple restarts are generally effective in bringing down robust accuracy.
We now experiment with the effect of number of restarts. We perform upto $20$ restarts and see a very gradual decrease in robust accuracy, with around $0.2\%$ drop from just using $5$ restarts. See Table~\ref{table:restarts} for the robust accuracies at different number of restarts. 
\begin{table}
  \centering
  \begin{tabular}{cc}
  	\toprule
    Number of restarts & Robust accuracy \\
    \midrule
    $\rho=2$ &62.7 \\
    $\rho=4$ &62.5 \\
    $\rho=6$ &62.5 \\
    $\rho=8$ &62.4 \\
    $\rho=10$ &62.4 \\
    $\rho=12$ &62.4 \\
    $\rho=14$ &62.3 \\
    $\rho=16$ &62.3 \\
    $\rho=18$ &62.3 \\
    $\rho=20$ &62.3 \\
    \bottomrule
  \end{tabular}
  ~
  \caption{Effect of number of restarts on robust accuracy with fixed step-size $\eta=0.01$ and number of steps $\tau=40$.}
  \label{table:restarts}
\end{table}
We remark that using a much larger number of steps or restarts could cause additional degradation (as in~\cite{hendrycks2019pretraining}), but we stick to the order of steps and restarts that are typically reported in the literature.

\paragraph{Fine-tuning for different $\epsilon$.}
Figure~\ref{fig:main-at} reports the accuracies for different $\epsilon$.
A smaller $\epsilon$ would typically require a smaller step-size. We fine-tune the step-size for each $\epsilon$ separately, by fixing the number of restarts to $5$ and number of steps to $40$.
\begin{itemize}
\item For $\epsilon = 0.008$, we span $\eta \in \{ 0.001, 0.002, 0.005 \}$.
\item For $\epsilon = 0.016$, we span $\eta \in \{ 0.002, 0.005, 0.01 \}$.
\item For $\epsilon = 0.024$, we span $\eta \in \{ 0.005, 0.01, 0.02 \}$.
\item For $\epsilon = 0.039$, we span $\eta \in \{ 0.005, 0.01, 0.02 \}$.
\end{itemize}
\subsection{Comparison with hyperparameters in~\citep{zhang2019theoretically}}\label{sec:app-hyper-comp}

As a baseline for adversarial robust self-training, we attempted to reproduce the results of~\citep{zhang2019theoretically}, whose publicly-released model has 56.6\% robust accuracy against $\pgdtrades$, 55.3\% robust accuracy against $\pgdours$, and 84.9\% standard accuracy. However, performing adversarial training with the hyper-parameters described in \Cref{sec:app-hyper} produces a poor result: the resulting model has only 50.8\% robust accuracy against $\pgdours$, and slightly better 85.8\% standard accuracy. We then changed all the hyper-parameters to be the same as in~\citep{zhang2019theoretically}, with the exception of the model architecture, which we kept at $\wrn{28}{10}$. The resulting model performed somewhat better, but still not on par with the numbers reported in~\citep{zhang2019theoretically}. 

Examining the training traces (\Cref{fig:trades-train}) reveals that without unlabeled data, both hyper-parameter configurations suffer from overfitting. More precisely, the robust accuracy degrades towards the end of the training, while standard accuracy gradually improves. In contrast, we see no such overfitting with $\rstat{50K+500K}$, directly showing how unlabeled data aids generalization. 

Finally, we perform ``early-stopping'' with the model trained according to~\citep{zhang2019theoretically}, selecting the model with highest validation robust accuracy. This model has 55.5\% robust accuracy against $\pgdtrades$, 54.1\% robust accuracy against $\pgdours$, and 84.5\% standard accuracy. This is reasonably close to the result of~\citep{zhang2019theoretically}, considering we used a slightly lower-capacity model.

\begin{figure}
	\centering
	\includegraphics[height=6cm]{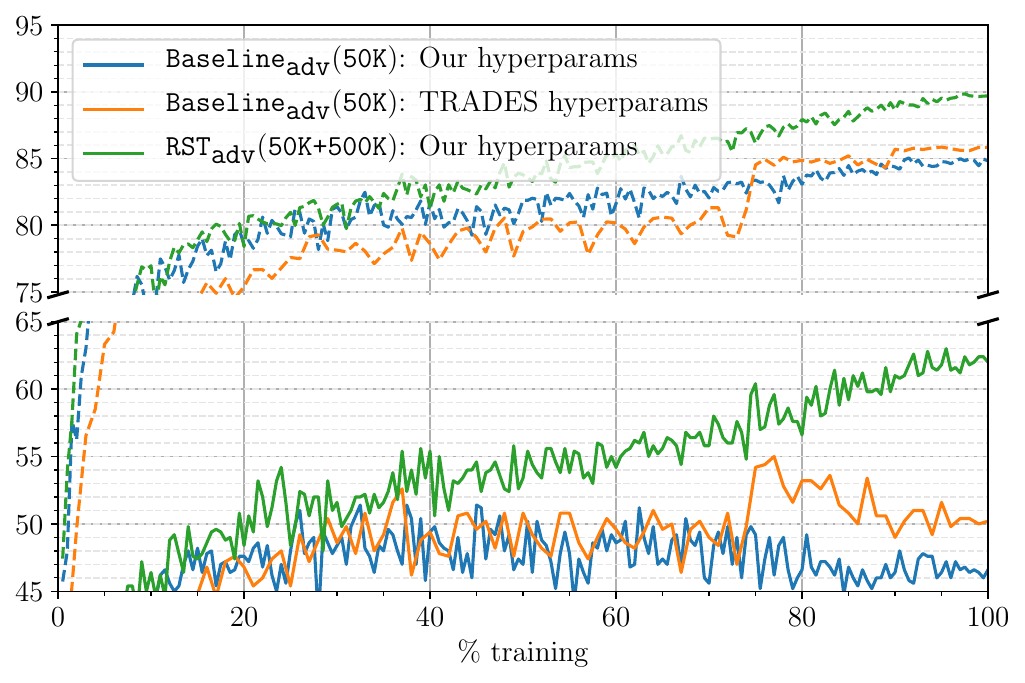}
	\caption{Comparison of training traces for different hyperpameters of adversarial training. Dashed lines show standard accuracy on the entire CIFAR-10 test set, and whole lines show robust accuracy against $\pgdtrades$ evaluated on the first 500 images in the CIFAR-10 test set.}
	\label{fig:trades-train} 
\end{figure}
\subsection{Comparison between stability and noise training}\label{sec:app-compare-stab}
As we report in \Cref{sec:benefit}, our fully supervised baseline for stability  training, i.e.\ $\baseline{stab}{50K}$, 
significantly outperforms the model trained in~\citep{cohen2019certified} and 
available online. There are three a-priori reasons for that: (i) our \wrn{28}{10} has higher capacity than the ResNet-110 used in~\citep{cohen2019certified}, (ii) we employ a different training configuration (e.g. a cosine instead of step learning rate schedule) and (iii) we use stability training while~\citet{cohen2019certified} use a different training objective. Namely, ~\citet{cohen2019certified} add $\sN(0,\sigma^2 I)$ to the input during training, but treat the noise as data augmentation, minimizing the loss
\[
\E_{\xadv\sim\sN(x, \sigma^2 I)} \loss(\theta, \xadv, y).
\]
We refer to the training method of ~\citep{cohen2019certified} as \emph{noise training}.

To test which of these three differences causes the gap in performance we train the following additional models. First, we perform noise training with ResNet-110, but otherwise use the same configuration used to train $\baseline{stab}{50K}$. Second, we keep the ResNet-110 architecture and our training configuration, and use stability training instead. Finally, we perform noise training on $\wrn{28}{10}$ with all other parameters the same as $\baseline{stab}{50K}$. As our goal in this section is to compare supervised training techniques for randomized smoothing, these experiments only use the CIFAR-10 training set (and no unlabeled data).

We plot the performance of all of these models, as well as the model of~\citep{cohen2019certified} and $\baseline{stab}{50K}$, in \Cref{fig:compare-stab}. Starting with the model of~\citep{cohen2019certified} and using our training configuration increases accuracy by 2--3\% across most perturbation radii. Switching to stability training reduces clean accuracy by roughly 3\%, but dramatically increases robustness at large radii, with a 13\% improvement at radius 0.5. Using the larger $\wrn{28}{10}$ model further improves performance by roughly 2\%. We also see that noise training on $\wrn{28}{10}$  performs better at radii below 0.25, and worse on larger radii. With stability training it is possible to further explore the tradeoff between accuracy at low and high radii by tuning the parameter $\beta$ in \eqref{eq:trades}, but we did not pursue this (all of our experiments are with $\beta=6$).

\begin{figure}
	\centering
	\begin{minipage}[t]{0.49\textwidth}
		\centering
		\includegraphics[height=4.0cm]{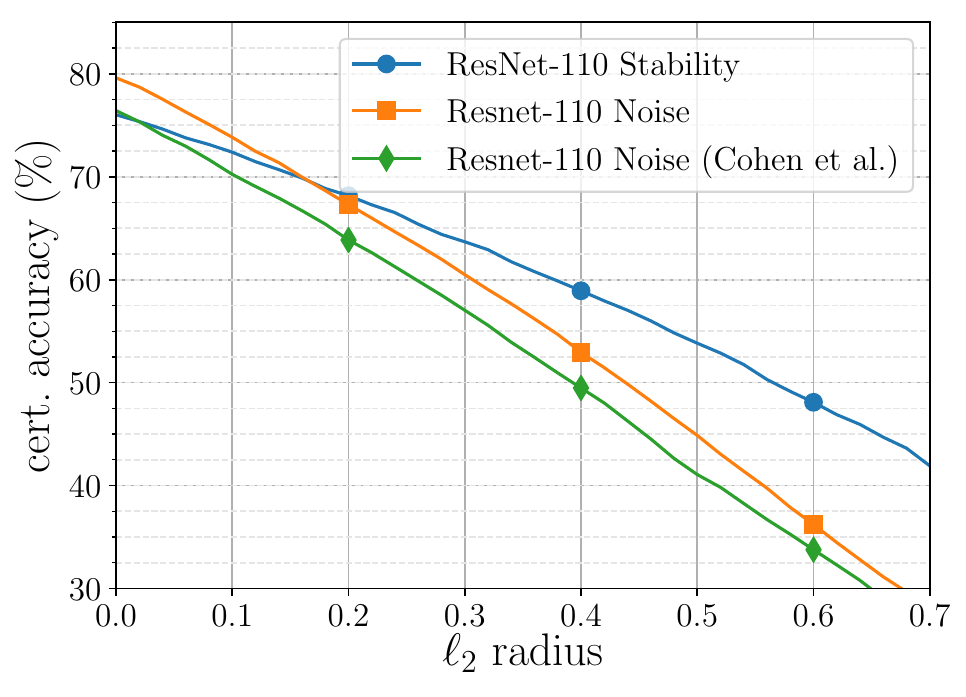}
		\\[0pt]
		\footnotesize\textbf{(a)}
	\end{minipage}
	\begin{minipage}[t]{0.49\textwidth}
		\centering
		\includegraphics[height=4.0cm]{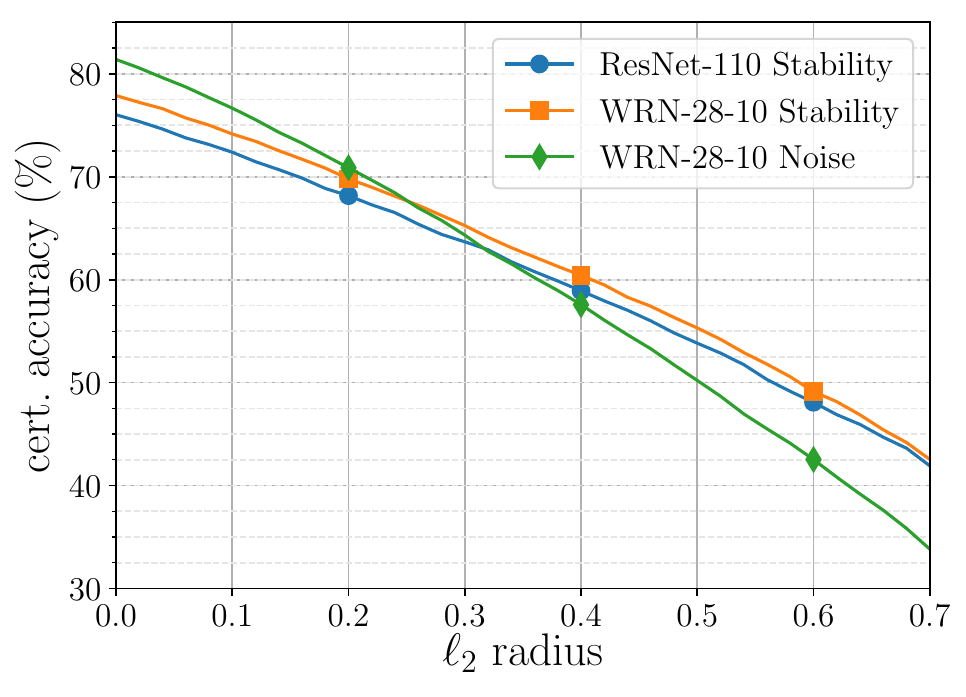}
		\\[0pt]
		\footnotesize\textbf{(b)}
	\end{minipage}
	\caption{Certified accuracy as a function of $\ell_2$ perturbation radius, comparing two architectures, two hyperparameter sets and two training objectives. \textbf{(a)} Both stability training and our hyperparameter choice improve performance. \textbf{(b)} Increasing model capacity improves performance further, and stability training remains beneficial.}
	\label{fig:compare-stab} 
\end{figure}
\subsection{Sourcing unlabeled data for CIFAR-10}\label{app:sourcing} Here we provide a detailed description of the sourcing process we describe in \Cref{sec:sourcing}. To obtain unlabeled data distributed similarly to the CIFAR-10 images, we use the 80 Million Tiny Images (80M-TI) dataset~\citep{torralba2008million}. This dataset contains 79,302,017 color images that were obtained via querying various keywords in a number of image search engines, and resizing the results to a 32x32 resolution. CIFAR-10 is a manually-labeled subset of 80M-TI. However, most of the 80M-TI do not fall into any of the CIFAR-10 categories (see \Cref{fig:80mti-display}a) and the query keywords constitute very weak labels (\Cref{fig:80mti-display}b).  To select relevant images, we train a classifier to classify TI data as relevant or not, in the following steps.

\paragraph{Training data for selection model.}
We create an 11-class training set consisting of the CIFAR-10 training set and 1M images sampled at random from the 78,712,306 images in 80M-TI with keywords that did not appear in CIFAR-10. We similarly sample an additional 10K images for validation.

\paragraph{Training the selection model.}
We train an 11-way classifier on this dataset, with the same $\wrn{28}{10}$ architecture employed in the rest of our experiments. We use the hyperparmeters described in~\Cref{sec:app-hyper}, except we run for 117K gadient steps. We use batch size 256 and comprise each batch from 128 CIFAR-10 images and 128 80M-TI images, and we also weight the loss of the ``80M-TI'' class by 0.1 to balance its higher number of examples. During training we evaluate the model on its accuracy of discriminating between CIFAR-10 and 80M-TI on a combination of the CIFAR-10 test and the 10K 80M-TI validation images, and show the training trace in \Cref{fig:cifar-vs-80mti}. Towards the end of the training, the CIFAR-10 vs. 80M-TI accuracy started to degrade, and we therefore chose and earlier checkpoint (marked in the figure) to use as the data selection model. This model achieves 93.8\%  CIFAR-10 vs. 80M-TI test accuracy.

\paragraph{Removing CIFAR-10 test set.}
To ensure that there is no leakage of the CIFAR-10 test set to the unlabeled data we source, we remove from 80M-TI all near-duplicates of the CIFAR-10 test set. Following~\citep{recht2018cifar}, we define a near-duplicate as an image with $\ell_2$ distance below $2000/255$ to any of the CIFAR-10 test images. We visually confirm that images with distance greater than $2000/255$ are substantially different. Our near-duplicate removal procedure leaves 65,807,640 valid candidates.

\paragraph{Selecting the unlabeled data.}
We apply our classifier on 80M-TI, with all images close to the CIFAR-10 test set excluded as described above. For each CIFAR-10 class, we select the 50,000 images which our classifier predicts with the highest confidence as belonging to that class. This is our unlabeled dataset, depicted in \Cref{fig:80mti-pred-display}, which is 10x the original CIFAR-10 training set and approximately class balanced.

Examining \Cref{fig:80mti-pred-display}, it is clear that our unlabeled dataset is not entirely relevant to the CIFAR-10 classification task. In particular, many of the ``frog'' and ``deer'' images are not actually frogs and deer. It is likely possible to obtain higher quality data by tuning the selection model (and particularly its training) more carefully. We chose not to do so for two reasons. First, allowing some amount of irrelevant unlabeled data more realistically simulates robust self-training in other contexts. Second, for a totally fair comparison against~\citep{zhang2019theoretically}, we chose not to use  state-of-the-art architectures or training techniques for the data selection model, and instead make it as close as possible to the final robust model.

\begin{figure}
	\centering
	\begin{minipage}[t]{0.49\textwidth}
		\centering
		\includegraphics[width=0.8\columnwidth]{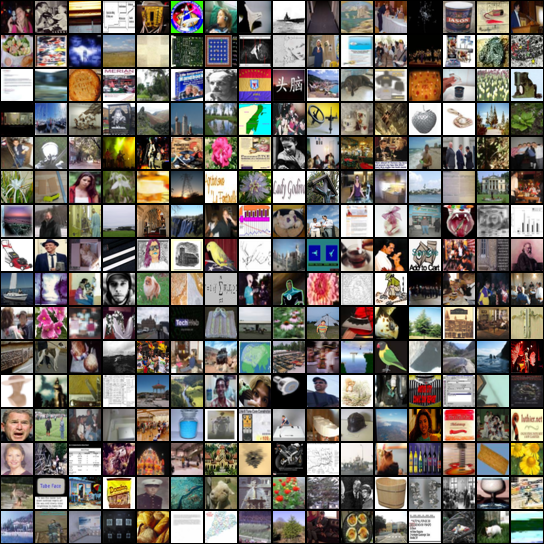}
		\\[0pt]
		\footnotesize\textbf{(a)}
	\end{minipage}
	\begin{minipage}[t]{0.49\textwidth}
		\centering
		\includegraphics[width=0.8\columnwidth]{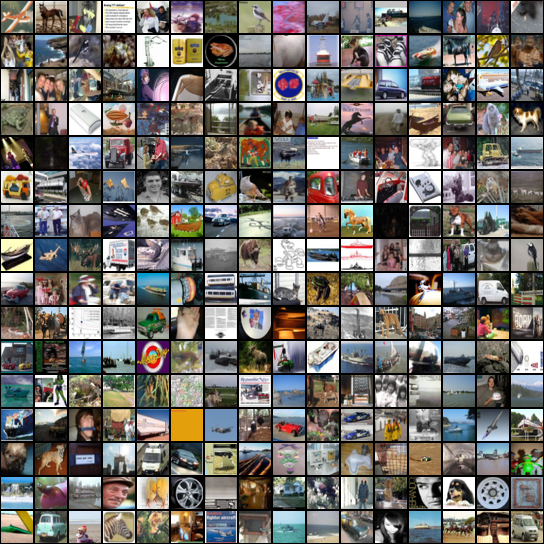}
		\\[0pt]
		\footnotesize\textbf{(b)}
	\end{minipage}
	\caption{Random images from the 80 Million Tiny Images data. 
	\textbf{(a)} Images drawn from the entire dataset. \textbf{(b)} Images 
	drawn for the subset with keywords that appeared in CIFAR-10; matching 
	keywords correlate only weakly with membership in one of the CIFAR-10 
	classes.}
	\label{fig:80mti-display} 
\end{figure}

\begin{figure}
	\centering
	\includegraphics[width=0.9\textwidth]{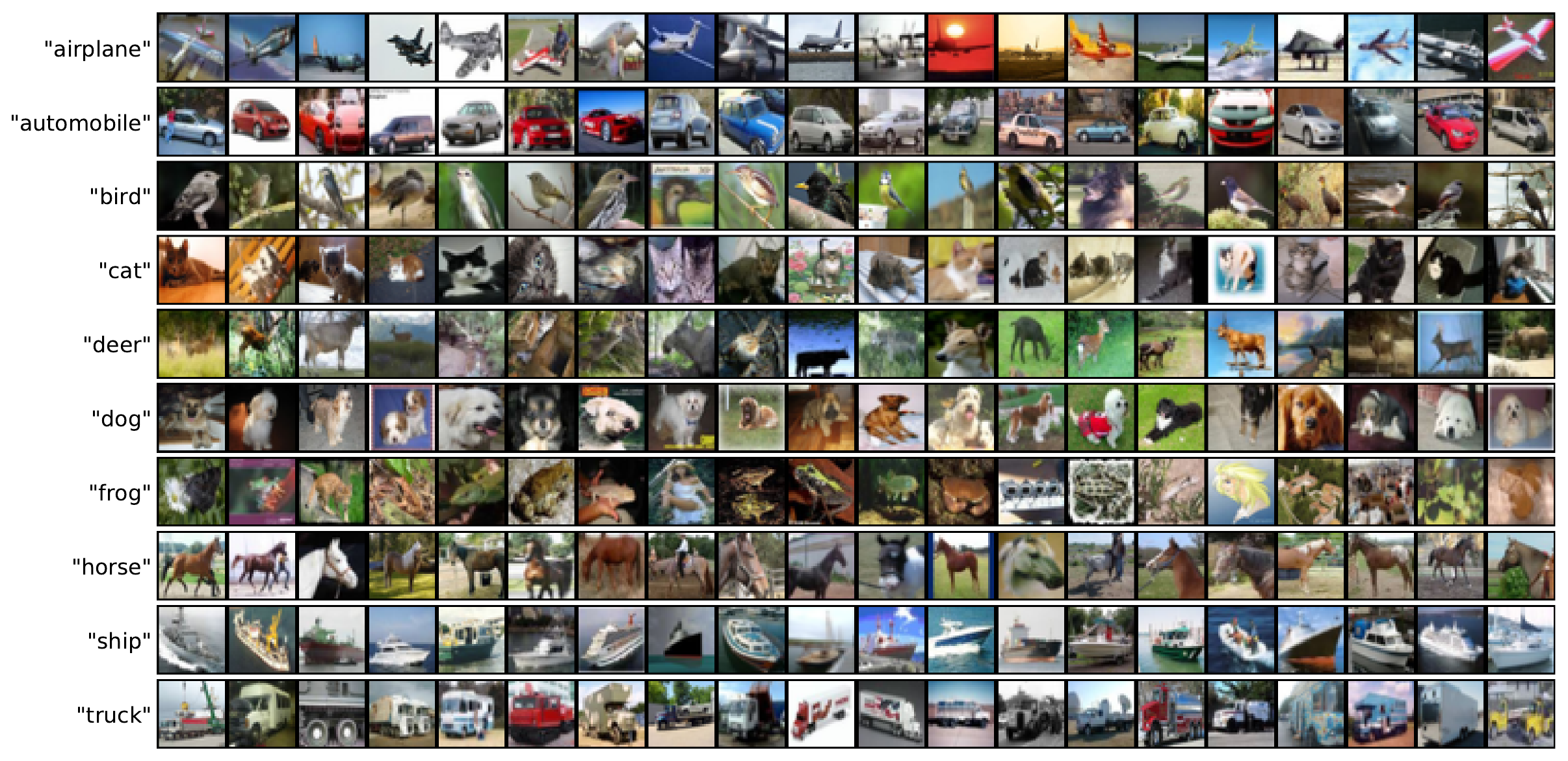}
	\caption{A random sample of our approximately class-balanced 500K 
	auxiliary images, rows correspond to class predictions made by our data 
	selection model. Note the multiple errors on ``frog'' and ``deer.''}
	\label{fig:80mti-pred-display} 
\end{figure}

\begin{figure}
	\centering
	\includegraphics[height=5cm]{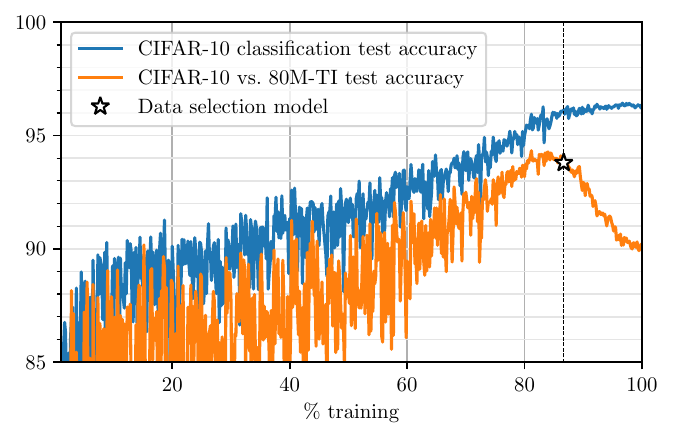}
	\caption{Test accuracy through training for the unlabeled data selection model.}
	\label{fig:cifar-vs-80mti} 
\end{figure}

 \section{Additional CIFAR-10 experiments}\label{sec:app-cifar}
\subsection{Alternative semisupervised training method}\label{sec:app-semisup}
In this section, we consider the straightforward adaptation
of \emph{virtual adversarial training} (VAT) to the (real) adversarial robustness setting of interest in this paper.

\paragraph{Training objective.}
Recall that we consider robust losses on the following form.
\begin{align*}
&\lossrob(\theta, x, y) = 
\loss(\theta, x, y) + \beta \lossreg(\theta, x),\\
&\quad\quad~\mbox{where}~~\lossreg(\theta, x)\defeq
\max_{\xadv \in\ballp(x)} \Dkl{p_\theta(\cdot \mid x)}{p_\theta(\cdot \mid \xadv)}. \nonumber
\end{align*}

The $\lossreg$ term does not require labels, and hence could be evaluated on the unlabeled data directly---exactly as done in Virtual Adversarial Training. As is commonly done in standard semisupervised learning, we also consider an additional entropy regularization term, to discourage the model from reaching the degenerate solution of mapping unlabeled inputs to uniform distributions. The total training objective is
\begin{align}
  \label{eq:alt-semisup}
  \sum\limits_{i=1}^{\nlab} \loss(\theta, x_i, y_i) + \beta \Big(\sum\limits_{i=1}^{\nlab} \lossreg(\theta, x_i) + \wun \sum\limits_{i=1}^{\nunlab} \lossreg(\theta, \xun_i)\Big)  + \lambda_\text{ent} \sum\limits_{i=1}^{\nunlab} h(p_\theta(\cdot \mid \xun_i)),
\end{align} 
where
$h(p_\theta(\cdot \mid \xun_i)) = -\sum\limits_{y \in \sY} p_\theta(y \mid \xun_i) \log p_\theta(y \mid \xun_i)$ is the entropy of the probability distribution over the class labels. We denote models trained according to this objective by $\semisupno$.

There are two differences between VAT and~\eqref{eq:alt-semisup}. First, to minimize the loss, VAT takes gradients of $\Dkl{p_\theta(\cdot \mid x)}{p_\theta(\cdot \mid \xadv)}$ w.r.t. $\theta$ only through $p_\theta(\cdot \mid \xadv)$, treating $p_\theta(\cdot \mid x)$ as a constant. Second, VAT computes perturbations $\xadv\in\balltwo(x)$ that are somewhere between random and adversarial, using a small number of (approximate) power iterations on $\hess_{\xadv}\Dkl{p_\theta(\cdot \mid x)}{p_\theta(\cdot \mid \xadv)}$.

We experiment with both the adversarial training based regularization~\eqref{eq:trades-adv}
and stability training based regularizer~\eqref{eq:trades-noise}.

\paragraph{Training details.}
We consider the alternative semisupervised objetive~\eqref{eq:alt-semisup} and compare with robust self-training. We use the same hyperparmeters as the rest of the main experiments, described in \Cref{sec:app-hyper}, and tune the additional hyperparameter $\lambda_\text{ent}$. Setting $\lambda_\text{ent} = \beta = 6$ would correspond to the entropy weight suggested in VAT~\citep{miyato2018virtual}, since they use a logarithmic loss and not KL-divergence.
 We experiment
with $\lambda_\text{ent}$ between $0$ and $6$. Due to computational  constraints, we use the smaller model $\wrn{40}{2}$ for tuning this hyperparameter in the adversarially trained models. We also perform certification via randomized smoothing on $1000$ random examples from the test set, rather than the entire test set. 

\begin{table}
  \centering
  \begin{tabular}{cccc}
  	\toprule
    Model architecture & Training algorithm & $\pgdours$ & No attack \\
    \midrule
    \wrn{40}{2} & $\baseline{adv}{50K}$ 
     &  52.1 &  81.3 \\
    \wrn{40}{2} & $\semisup{adv}{50K+500K}, \lambda_\texttt{ent}=0$ 
    &  53.6 &  81.5 \\
    \wrn{40}{2} & $\semisup{adv}{50K+500K}, \lambda_\texttt{ent}=0.1$ 
	 &  52.9 &  77.9 \\
    \wrn{40}{2} & $\semisup{adv}{50K+500K}, \lambda_\texttt{ent}=0.6$ 
	 &  43.4 &  60.0 \\
    \wrn{40}{2} & $\semisup{adv}{50K+500K}, \lambda_\texttt{ent}=3.0$ 
	 &  15.4 &  19.1 \\
    \midrule
    \wrn{28}{10} & TRADES~\citep{zhang2019theoretically} %
     &  55.4 &  84.9 \\
    \wrn{28}{10} & $\semisup{adv}{50K+500K}, \lambda_\texttt{ent}=0$ 
	&  56.5 &  83.2 \\
    \wrn{28}{10} & $\rstat{50K+500K}$ 
	 &  62.5 &  89.7 \\
    \bottomrule
  \end{tabular}
  ~
  \caption{Accuracy of adversarially trained models against our PG attack. We see that VAT-like consistency-based regularization produces only minor gains over a baseline without unlabeled data, significantly underperforming robust self-training.}
  \label{table:semisup-adv}
\end{table}

\paragraph{Results.} 
Figures~\ref{fig:semisup-l2-radius} and~\ref{fig:semisup-l2-ew} summarize the results for stability training. We see that the alternative approach also yields gains (albeit much smaller than robust self-training) over the baseline supervised setting, due to the additional unlabeled data.

Table~\ref{table:semisup-adv} presents the accuracies against $\pgdours$ for the smaller $\wrn{40}{2}$ models with different settings of $\lambda_\text{ent}$. We see a steady degradation in the performance with increasing $\lambda_\text{ent}$. For $\lambda_\text{ent}=0$, we also train a large $\wrn{28}{10}$ to compare to our state-of-the-art self-trained $\rstat{50K+500K}$ that has the same architecture. 

We see that in both adversarial training (heuristic defense) and stability training (certified defense), robust self-training significantly outperforms the alternative approach suggesting that ``locking in'' the pseudo-labels is important. This is possibly due to the fact that we are in a regime where the pseudo-labels are quite accurate and hence provide good direct signal to the model. Another possible explanation for the comparative weak performance of $\semisupno$ is that since the robustly-trained  never reaches good clean accuracy, it cannot effectively bootstrap the model's own prediction as training progresses.

\begin{figure}
	\centering
	\includegraphics[height=4.0cm]{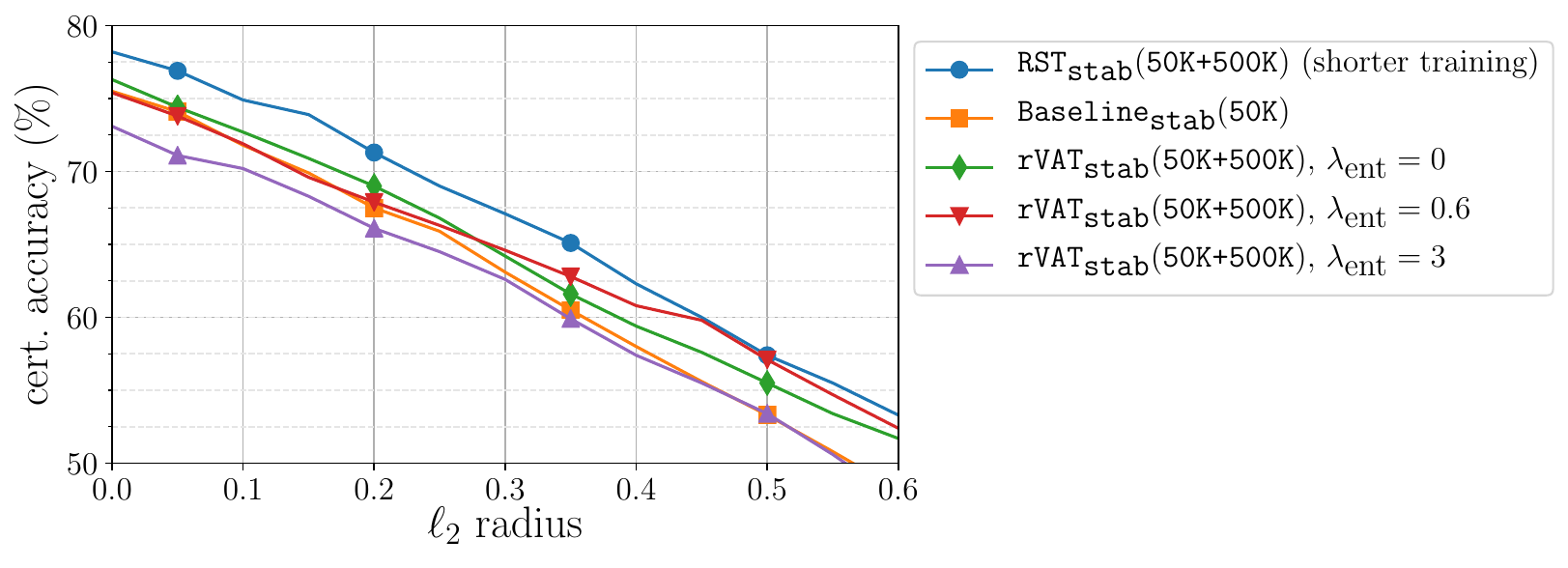}
        \caption{Certified $\ell_2$ accuracy as a function of the radius, for $\wrn{28}{10}$ trained using different semisupervised approaches on the augmented CIFAR10 dataset.}
        \label{fig:semisup-l2-radius}
\end{figure}

\begin{figure}
    \centering
    \includegraphics[height=4.0cm]{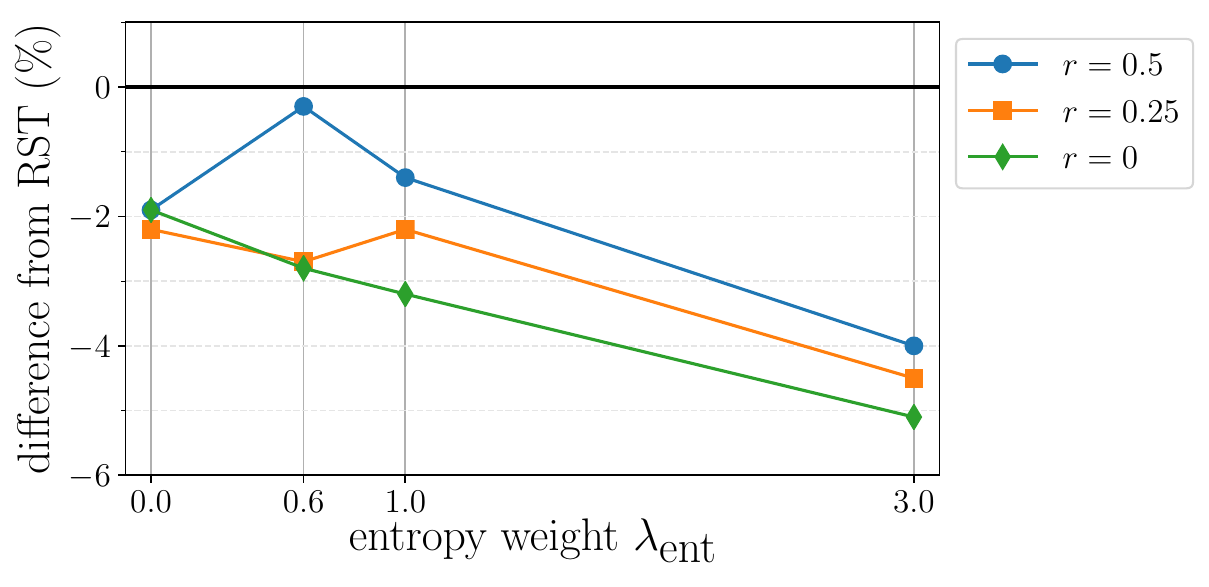}
    \caption{Effect of entropy weight on performance of stability training with alternate semisupervised approach. Larger entropy weight leads to gains in robustness at larger radii
      at a cost in robustness at smaller radii. However, robust self-training outperforms the alternative semisupervised approach over all radii, for all settings of the entropy weight hyperparameter.}
    \label{fig:semisup-l2-ew}
\end{figure}

\paragraph{Additional experiments.} VAT combined with entropy minimization is one of the most succesful semisupervised learning algorithm for standard accuracy~\citep{oliver2018realistic}. As we mentioned earlier, the objective in~\eqref{eq:alt-semisup} differs from VAT by not treating $p_\theta(\cdot \mid \xun)$ as a constant. From the open source implementation\footnote{\url{https://github.com/lyakaap/VAT-pytorch}}, we also note that batch normalization parameters are not updated on the ``adversarial'' perturbations during training.
We tried both variants on the smaller model $\wrn{40}{2}$ and observed that neither modification affected final performance in our setting. Further, we also experimented with different values of the unlabeled data weighting parameter $\wun$. We observed no noticeable improvement in final performance by changing this parameter.
\subsection{Comparison with data augmentation}\label{sec:app-augment}

Advanced data augmentation techniques such as cutout~\citep{devries2017improved} and AutoAugment policies~\citep{cubuk2019autoaugment} provide additional inductive bias about the task at hand. Can data augmentation provide gains in CIFAR-10 robust training, similar to those we observe by augmenting the CIFAR-10 dataset with extra unlabeled images?

\paragraph{Implementation details.}We use open source implementations for cutout\footnote{\url{https://github.com/uoguelph-mlrg/Cutout}} and AutoAugment\footnote{\url{https://github.com/DeepVoltaire/AutoAugment }}, and use the same training hyperparameters from the papers introducing these techniques. We first train $\wrn{28}{10}$ via standard training and reproduce the test accuracies reported in the papers~\citep{devries2017improved, cubuk2019autoaugment}. 

\paragraph{Training details.}
We perform robust supervised training, where the each batch contains \emph{augmented} (cutout/autoaugment in additional to random crops and flips) of the CIFAR10 training set.

We use the same training setup with which we performed robust self-training, as described in \Cref{sec:app-hyper}, except we applied augmentation instead of adding unlabeled data, and that we perform stability training for only 39K steps.
Note that this model and optimization configuration are identical to those used in the AutoAugment paper~\citep{cubuk2019autoaugment}, except increasing batch size to $256$.

\begin{table}
  \centering
  \begin{tabular}{ccc}
  	\toprule
    Model & $\pgdours$ & No attack \\
    \midrule
    TRADES~\citep{zhang2019theoretically}& 55.4 &  84.9\\
    $\baseline{adv}{50K}$       &  47.5 &  84.8 \\
    + Cutout~\citep{devries2017improved} &  51.2 &  85.8 \\
    + AutoAugment~\citep{cubuk2019autoaugment} &  47.4 &  84.5 \\
    $\rstat{50K+500K}$  &  62.5 &  89.7 \\
    \bottomrule
  \end{tabular}
  ~
  \caption{Accuracy of adversarially trained models against $\pgdours$.
    We see that cutout provides marginal gains and AutoAugment leads to slightly worse performance than our baseline that only uses the standard crops and flips.}
  \label{table:aug-adversarial}
\end{table}

\begin{figure}
  \centering
  \includegraphics[height=4cm]{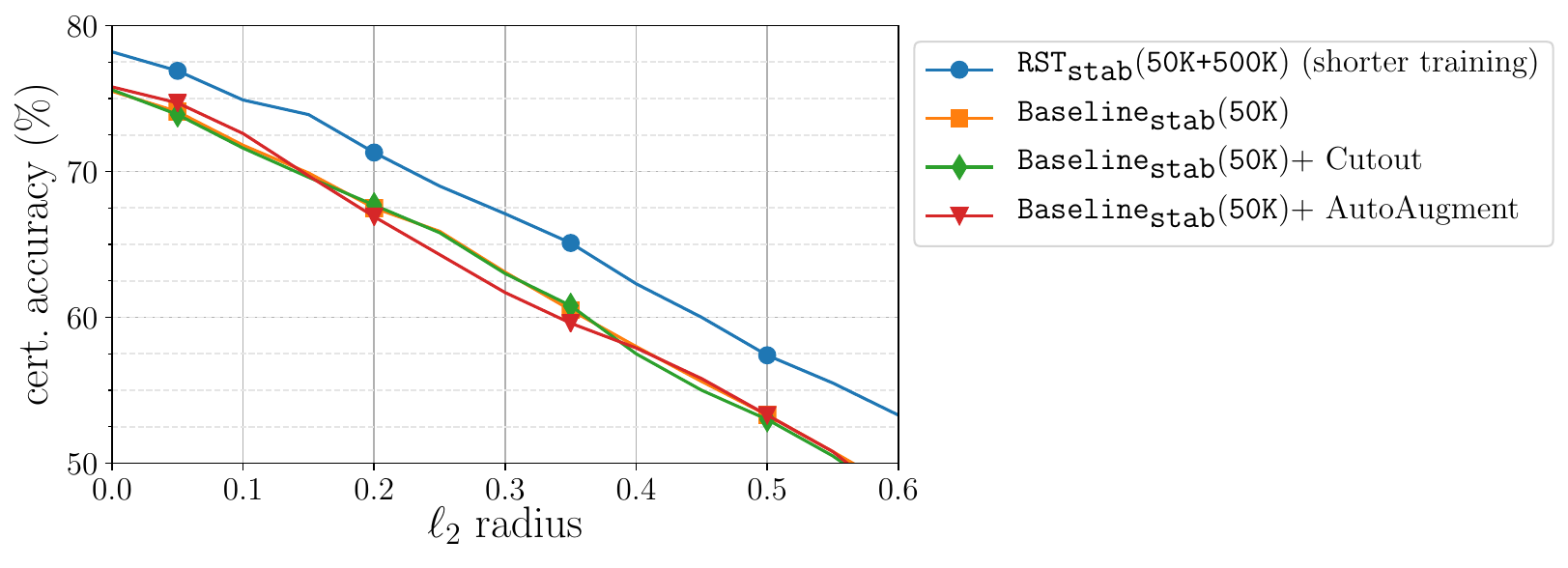}
  \caption{Comparing certified $\ell_2$ accuracy as a function of certified radius, computed via randomized smoothing for different data augmentation methods. }
  \label{fig:aug-smoothing}
\end{figure}

\paragraph{Results.}
Table~\ref{table:aug-adversarial} presents the results on accuracy of the heuristic adversarially trained models against our PG variant $\pgdours$ for $\epsilon=8/255$. We also tabulate the performance of $\baseline{adv}{50K}$, which doesn't use any unlabeled data or augmentation---as we show in \Cref{sec:app-hyper-comp}, this configuration produces poor results due to overfitting. We see that AutoAugment offers no improvement over $\baseline{adv}{50K}$, while cutout ofer a 4\% improvement that is still far from the performance~\citep{zhang2019theoretically} attains with just early stopping. In  \Cref{fig:aug-training} we plot training training traces, and show that AutoAugment fails to prevent overffiting, while cutout provides some improvement, but is till far away from the effect of robust self-training. 

We also perform certification using randomized smoothing on the stability trained models.
We use the same certification parameters as our main experiments $(N_0=100, N=10^4, \alpha=10^{-3}, \sigma=0.25)$ and compute the certificate over every 10th image in the CIFAR-10 test set (1000 images in all) for all the models. As we only train for 31K steps, we compare the results to an $\rstst{50K+500K}$ model trained for the same amount of steps (and evaluated on the same subset of images). We plot the results  to obtain Figure~\ref{fig:aug-smoothing}. As can be seen, data augmentation provides performance essentially identical to that of a baseline without augmentation.  Overall, it is interesting to note that the gains provided by data augmentation for standard training do not manifest in robust training (both adversarial and stability training).

\begin{figure}
  \centering
    \centering
    \includegraphics[height=6cm]{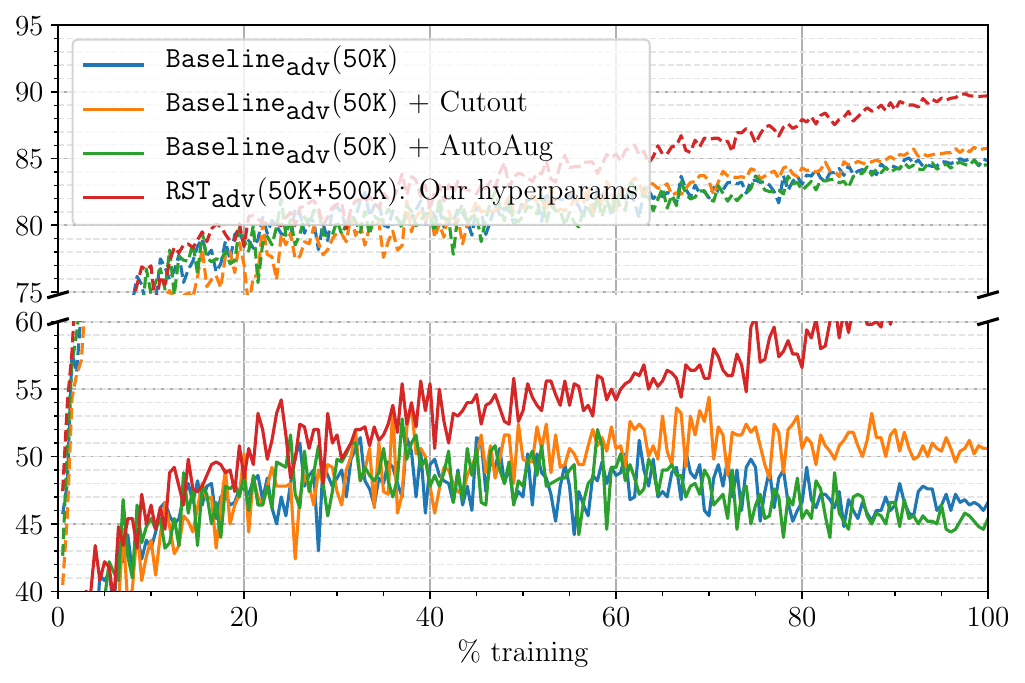}
    ~
  \caption{Training traces adversarial training with robust self-training and various forms of data augmentation. Dashed lines show standard accuracy on the entire CIFAR-10 test set, and whole lines show robust accuracy against $\pgdtrades$ evaluated on the first 500 images in the CIFAR-10 test set. AutoAugment shows the same overfitting as standard augmentation, while cutout mitigates overfitting, but does not provide the gains of self-training.}
  \label{fig:aug-training}
\end{figure}
\subsection{Effect of unlabeled data relevance}
\label{sec:app-relevance}

Semisupervised learning algorithms are often sensitive to the relevance of the unlabeled data. That is, if too much of the unlabeled data does not correspond to any label, the learning algorithm might fail to benefit from the unlabeled dataset, and may even produce worse results simply using labeled data only. This sensitivity was recently demonstrated in a number of semisupervised neural network training techniques~\cite{oliver2018realistic}. In the simple Gaussian model of \Cref{sec:theory}, our analysis in \Cref{sec:app-irrelevant} shows that any fixed fraction of relevant unlabeled data will allow self-training to attain robustness, but that the overall sample complexity grows inversely with the square of the relevant fraction.

To test the practical effect of data relevance on robust self-training, we conduct the following experiment. We draw 500K random images from the 80M-TI with CIFAR-10 test set near-duplicates removed (see \Cref{app:sourcing}), which rarely portray one of the CIFAR-10 classes (see \Cref{fig:80mti-display}a); this is our proxy for irrelevant data. We mix the irrelevant dataset and our main unlabeled image dataset with different proportions, creating a sequence of unlabeled datasets with a growing degree of image relevance, each with 500K images.

We then perform adversarial- and stability- robust self-training on each of those datasets as described in \Cref{sec:experiments} and \Cref{sec:app-hyper}, except here we use a smaller \wrn{40}{2} model to conserve computation resources. We evaluate each of those models as in~\Cref{sec:benefit}, and compare them to a fully supervised baseline. For stability training we train the baseline as in~\Cref{sec:benefit}, except with \wrn{40}{2}. For adversarial training, in view of our findings in~\Cref{sec:app-hyper-comp}, we train our baseline with \wrn{40}{2} and the training hyperparameters of~\citep{zhang2019theoretically} ($\beta$ is still 6). Here, we did not observe overfitting in the training trace, and therefore used the model from the final epoch of training. This baseline achieves 81.3\% standard accuracy and 52.1\% robust accuracy against $\pgdours$. 

In \Cref{fig:relevance} we plot the difference in $\pgdours$/certified accuracy as a function of the relevant data fraction. We see the performance grows essentially monotonically with data relevance, as expected. We also observe that smoothing seems to be more sensitive to data relevance, and we see performance degradation for completely irrelevant and 20\% relevant data. For adversarial training we see no performance degradation, but the gain with completely irrelevant data is negligible as can be expected. Finally, at around 80\% relevant data performance seems close to that of 100\% relevant data. This is perhaps not too surprising, considering that even our ``100\% relevant'' data is likely somewhere between 90-95\% relevant (see \Cref{app:sourcing}). This is also roughly in line with our theoretical model, where there is a small degradation in performance for relevant fraction $\alpha=0.8$.

\begin{figure}
	\centering
	\begin{minipage}[t]{0.495\textwidth}
		\centering
		\includegraphics[height=4.0cm]{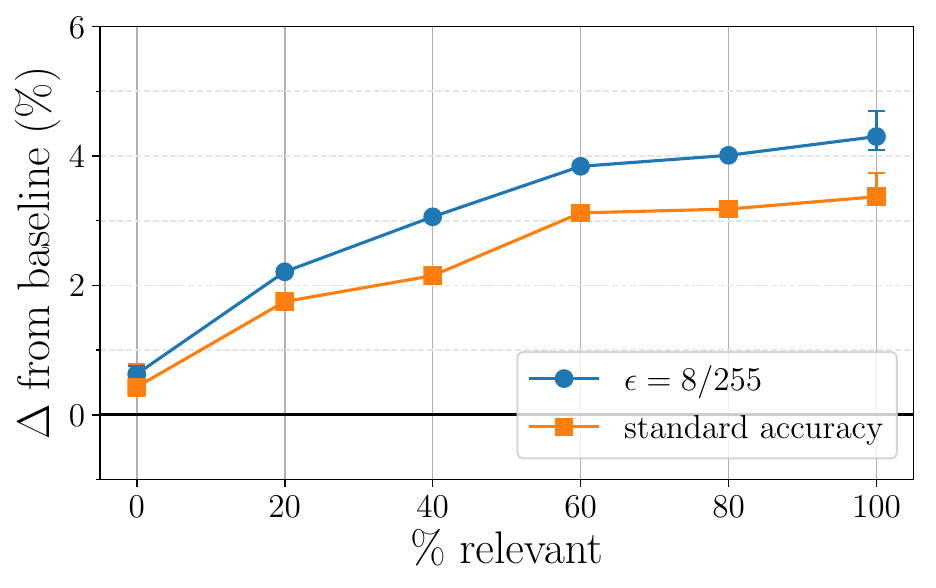}
		\\[0pt]
		\footnotesize\textbf{(a)}
	\end{minipage}
	\begin{minipage}[t]{0.495\textwidth}
		\centering
		\includegraphics[height=4.0cm]{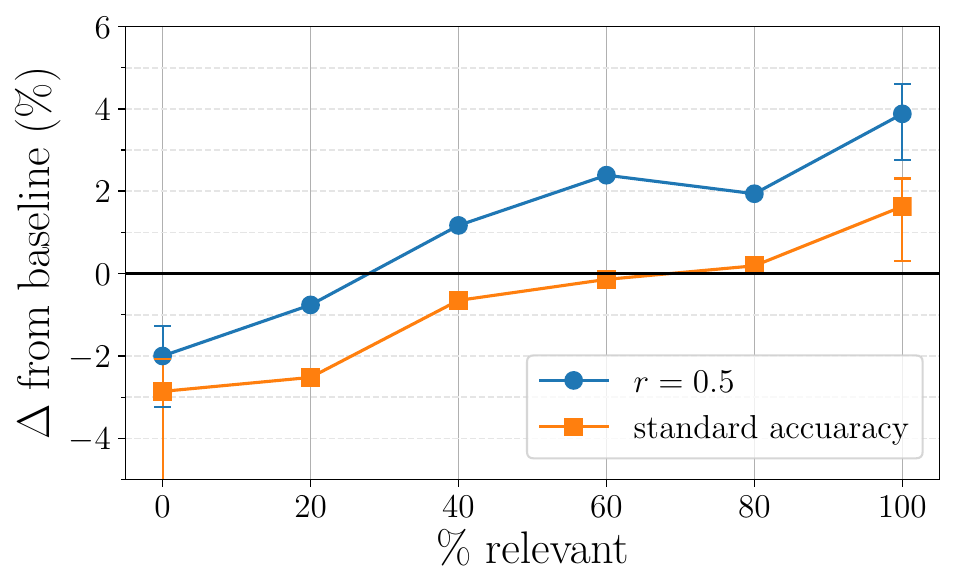}
		\\[0pt]
		\footnotesize\textbf{(b)}
	\end{minipage}
	\caption{Performance improvement as a function of unlabeled data relevance, for \wrn{40}{2}  models, relative to training without unlabeled data. When present, error bars indicate the  range of variation (minimum to maximum) over 5 independent runs. \textbf{(a)} Difference in standard accuracy and in accuracy under the $\ell_\infty$ attack $\pgdours$, for adversarially-trained models. \textbf{(b)} Difference in standard accuracy and in certified $\ell_2$ accuracy for smoothed stability-trained models.}
	\label{fig:relevance} 
\end{figure}

\subsection{Effect of unlabeled data amount}\label{sec:app-amount}

We have conducted all of our previous experiments on an unlabeled dataset 
of fixed size (500K images). Here we test the effect of the size of the 
unlabeled dataset on robust self-training performance. The main question 
we would like to answer here is how much further performance gain can we 
expect from further increasing the data set: should we expect another 7\% 
improvement over~\citep{zhang2019theoretically} by adding 500K more 
images (of similar relevance), or perhaps we would need to add 5M images, 
or is it the case that the benefit plateaus after a certain amount of unlabeled 
data?

There is a number of complicating factors in answering this question. First, 80M-TI does not provide us much more than 500K relevant unlabeled images, certainly not if we wish it to be approximately class-balanced; the more images we take from 80M-TI the lower their relevance. Second, as the amount of data changes, some training hyperparameters might have to change as well. In particular, with more unlabeled data, we expect to require more gradient steps before achieving convergence. Even for 500K images we haven't  completely exhausted the possible gains from longer training. Finally, it might be the case that higher capacity models are crucial for extracting benefit from larger amounts of unlabeled data. For these reasons, and considering our computational constraints, we decided that attempting to rerun our experiment with more unlabeled data will likely yield inconclusive results, and did not attempt it.

Instead, we sub-sampled our 500K images set randomly into nested subsets of varying size, and repeated our experiment of \Cref{sec:benefit} with these smaller datasets. With this experiment we hope to get some understanding of the trend with which performance improves. Since we expect model capacity to be very important for leveraging more data, we perform this experiment using the same high-capacity \wrn{28}{10} used in our main experiment. For adversarial training, we use exactly the same training configuration for all unlabeled dataset sizes. For stability training, we also use the same configuration except we attempted to tune the number of gradient steps. For each dataset size, we started by running 19.5K gradient steps, and doubled the number of gradient steps until we no longer saw improvement in performance. For 40K extra data, we saw the best results with 39K gradient steps, and for 100K and 240K extra data we saw the best result with 78K gradient steps.
Similarly to the data relevance experiment in \Cref{sec:app-relevance}, we compare each training result with a baseline. Here the baselines are the same as those reported in \Cref{sec:benefit}: for adversarial training we compare to the publicly available model of~\citep{zhang2019theoretically}, and for stability training we use $\baseline{stab}{50K}$. 

We plot the results in \Cref{fig:amount}. As the figure shows, accuracy always grows with the data size, except for one errant data point at 40K unlabeled data with adversarial training, which performs worse than the baseline. While we haven't seen overfitting in this setting as we have when attempting to reproduce~\citep{zhang2019theoretically}, we suspect that the reason for the apparent drop in performance is that our training configuration was not well suited to so few unlabeled data. We also see in  the plot that the higher the robustness, the larger the benefit from additional data. 

The experiment shows that 100K unlabeled data come about halfway to 
achieving the gain of 500K unlabeled data. Moreover, for the most part the 
plots appear to be concave, suggesting that increase in data amount 
provides diminishing returns---even on a logarithmic scale. Extrapolating 
this trend to higher amounts of data would suggest we are likely to require 
very large amount of data to see another 7\% improvement. However, the 
negative value at 40K unlabeled data hints at the danger of trying to 
extrapolate this figure---since we haven't carefully tuned the training at 
each data amount  (including the one at 500K), we cannot describe any 
trend with confidence. At most, we can say that under our computation 
budget, model architecture and training configuration, it seems likely the 
benefit of unlabeled data plateaus at around 500K examples. It also seems 
likely that as computation capabilities increase and robust training 
improves, the point of diminishing returns will move towards higher data 
amounts.

\begin{figure}
	\centering
	\begin{minipage}[t]{0.495\textwidth}
		\centering
		\includegraphics[height=4.0cm]{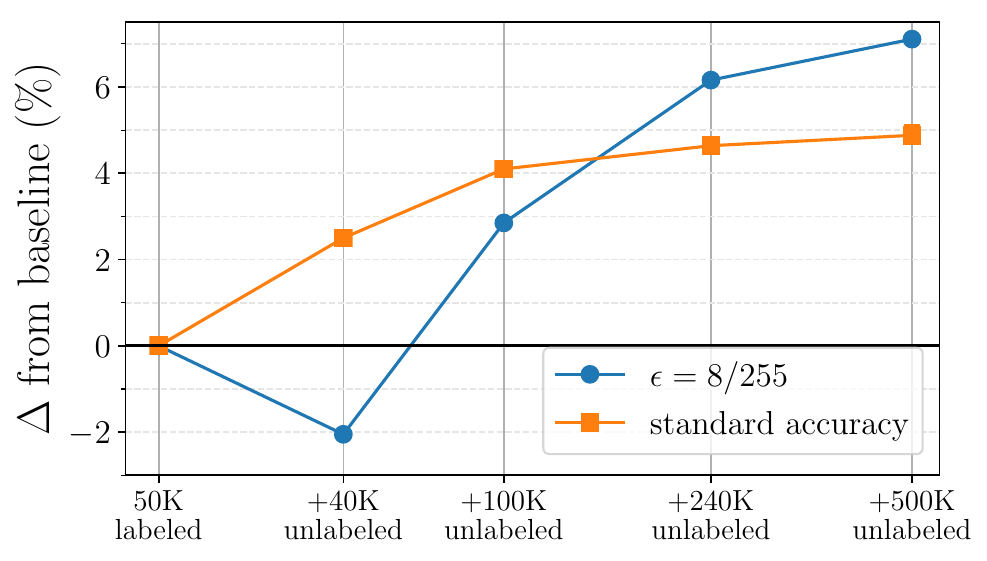}
		\\[0pt]
		\footnotesize\textbf{(a)}
	\end{minipage}
	\begin{minipage}[t]{0.495\textwidth}
		\centering
		\includegraphics[height=4.0cm]{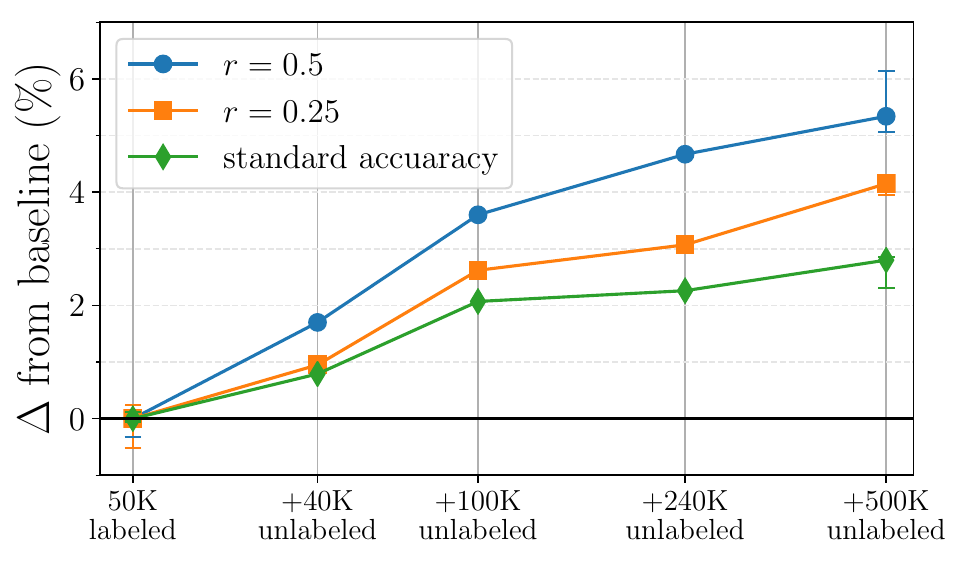}
		\\[0pt]
		\footnotesize\textbf{(b)}
	\end{minipage}
	\caption{Performance improvement as a function of unlabeled data 
	amount, relative to training without unlabeled data. When present, error 
	bars indicate the  range of variation (minimum to maximum) over 3 
	independent runs. \textbf{(a)} Difference in standard accuracy and in 
	accuracy under the $\ell_\infty$ attack $\pgdours$, for 
	adversarially-trained models. \textbf{(b)} Difference in standard accuracy 
	and in certified $\ell_2$ accuracy for smoothed stability-trained models, 
	at different radii.}
	\label{fig:amount} 
\end{figure}
\subsection{Effect of label amount}\label{sec:app-label-amount}

To complement the unlabeled data amount experiment 
of~\Cref{sec:app-amount}, we study the effect of the \emph{labeled} 
data amount. Here too, we only consider smaller amounts of labeled data 
than in our main experiment, since no additional CIFAR-10 labels are 
readily available. The main effect of removing some of the labels is that the 
accuracy of pseudo-labels decreases. However, since the main motivation 
behind robust self-training is that the pseudo-labels only need to be 
``good enough,'' we expect fewer labels to still allow significant gains from 
unlabeled data.

To test this prediction, we repeat the experiments of~\Cref{sec:benefit} with 
the following modification. For each desired label amount 
$\nn\in\{2, 4, 8, 16, 32\}\text{K}$, we pick the first $\nn$  images from a 
random permutation of 
CIFAR-10 to be our labeled data, and consider an unlabeled dataset of 
$50\text{K}-\nn+500{K}$ images comprised of the remaining CIFAR-10 
images and our 500K unlabeled images.
We  
train a classification model on only the labeled subset, using the same 
configuration as in the ``pseudo-label generation'' paragraph 
of~\Cref{sec:app-hyper}, and apply that model on the unlabeled data to 
form pseudo-labels. We then repeat the final step of robust self-training 
using 
these pseudo-labels, with parameters as in~\Cref{sec:app-hyper}. Adding 
the remainder of CIFAR-10 without labels to our unlabeled data keeps the 
total dataset fixed. This allows us to isolate the 
effect of the quality of the pseudo-labels and hence the label amount on robust self-training. 

\Cref{fig:label-amount} summarizes the results of this experiment, which 
are consistent with our expectation that robust self-training remains 
effective even with a much smaller number of labels. In particular, for both 
adversarial robust self-training (\Cref{fig:label-amount}a) and 
stability-based certified robust self-training (\Cref{fig:label-amount}b), 8K 
labels combined with the unlabeled data allow us to obtain comparable  
robust accuracy to that of the state-of-the-art fully supervised method. Moreover, for 
adversarial-training we obtain robust accuracy only 2\% lower than the 
supervised state-of-the-art with as few as 2K labels. In this small labeled 
data regime, we also see that the standard accuracy of the resulting robust 
model is slightly higher than the 
pseudo-label generator accuracy (the dotted black line in 
\Cref{fig:label-amount}a).

Two remarks are in order. First, in the low-label regime we can likely attain 
significantly better results by improving the accuracy of the pseudo-labels using standard semisupervised learning. The results in \Cref{fig:label-amount} 
therefore constitute only a crude lower bound on the benefit of unlabeled 
data when fewer labels are available. Second, we note that the creation of 
our 500K unlabeled dataset involved a ``data selection'' classifier trained 
on all of the CIFAR-10 labels, and we did not account for that in the 
experiment above. Nevertheless, as the data selection model essentially simulates a situation where unlabeled data is generally relevant, we believe that our experiment 
faithfully represents the effect of label amount (mediated through 
pseudo-label quality) on robust self-training. Further, our experiments on the effect of the relevance of unlabeled data (described in Appendix~\ref{sec:app-relevance}) suggest that using a slightly worse data selection model by training only on a subset of CIFAR-10 labels should not change results much. 

\begin{figure}
	\centering
	\begin{minipage}[t]{0.495\textwidth}
		\centering
		\includegraphics[height=5.0cm]{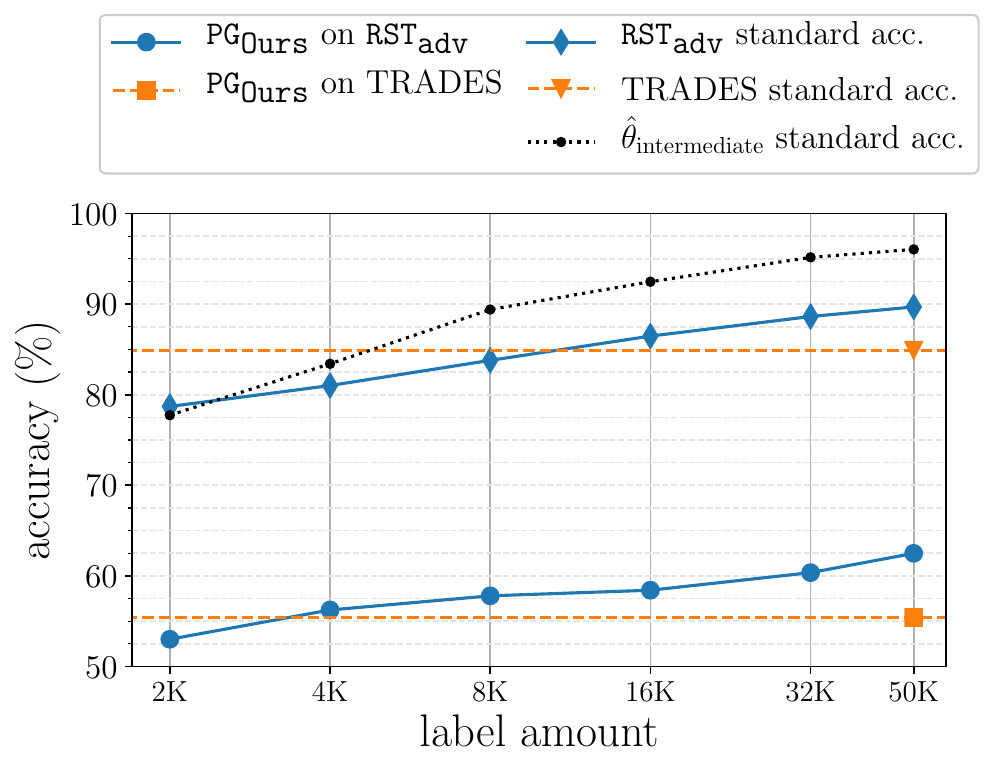}
		\\[0pt]
		\footnotesize\textbf{(a)}
	\end{minipage}
	\begin{minipage}[t]{0.495\textwidth}
		\centering
		\includegraphics[height=5.0cm]{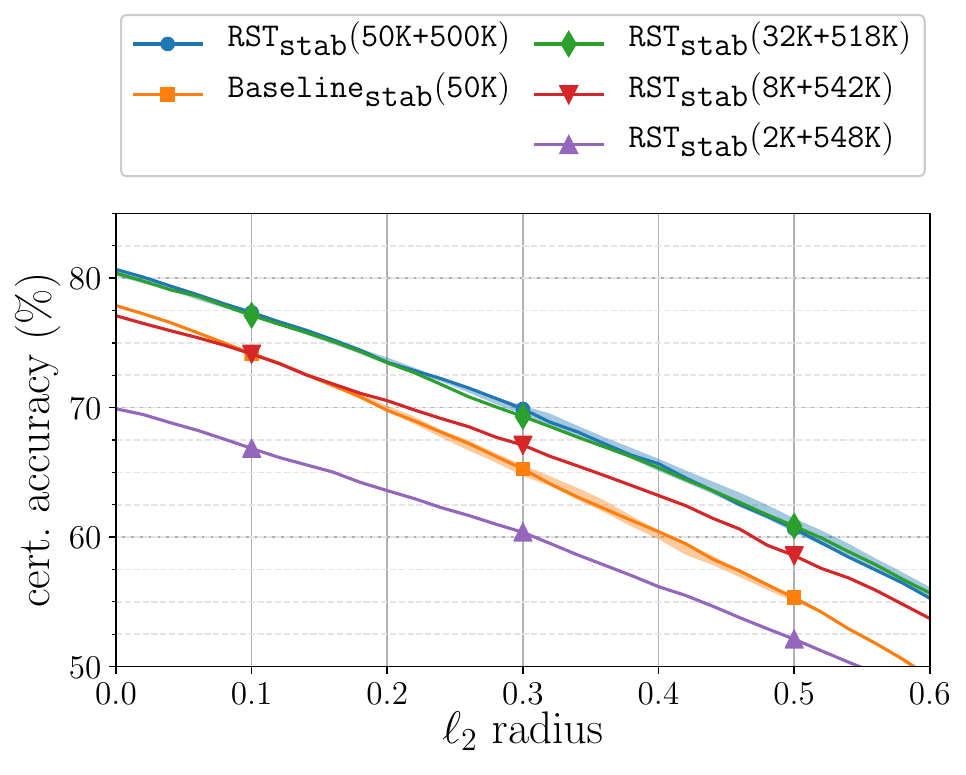}
		\\[0pt]
		\footnotesize\textbf{(b)}
	\end{minipage}
	\caption{Performance for varying labeled data 
		amount, compared to the fully supervised setting. \textbf{(a)}
		Standard accuracy and  
		accuracy under the $\ell_\infty$ attack $\pgdours$, for 
		adversarially-trained models, as a function of label amount. The 
		dotted line indicates the accuracy of the pseudo-label generation 
		model, which we train with the labeled data only. \textbf{(b)}
		Certified robust accuracy as a function of $\ell_2$ radius for different 
		label amounts, and our fully supervised baseline; 
		$\rstst{{\it{a}}+\it{b}}$ 
		denotes the stability trained model with $a$ labeled data and $b$ 
		unlabeled data, which consists of the $50\text{K}-a$ de-labeled 
		CIFAR-10 
		images 
		and our 500K unlabeled dataset.}
	\label{fig:label-amount} 
\end{figure}
\subsection{Standard self-training}\label{sec:app-sst}
We also test whether our unlabeled data can improve standard accuracy. We perform standard self-training to obtain the model $\sst{50K+500K}$; all training parameters are identical to $\rstst{50K+500K}$, except we do not add noise. $\sst{50K+500K}$ attains test accuracy 96.4\%, a 0.4\% improvement over standard supervised learning with identical training parameters on 50K labeled dataset. 
This difference is above the training variability of around 0.1$\%~$\citep{zagoruyko2016wide, cubuk2019autoaugment}, but approaches like aggressive data augmentation~\citep{cubuk2019autoaugment} provide much larger gains for the same model (to 97.4\%).

\subsection{Performance against different $\pgdno$ attack radii}\label{sec:app-radii}

In Figure~\ref{fig:cifar-radii}, we evaluate robustness of our state-of-the-art $\rstat{50K+500K}$ on a range of different values of $\epsilon$ and compare to TRADES~\cite{zhang2019theoretically}, where we fine-tune attacks for each value of $\epsilon$ separately for each model (see \Cref{sec:app-attack}). We see a consistent gain in robustness across the different values of $\epsilon$. 

\begin{figure}
  \centering
    \includegraphics[width=0.5\textwidth]{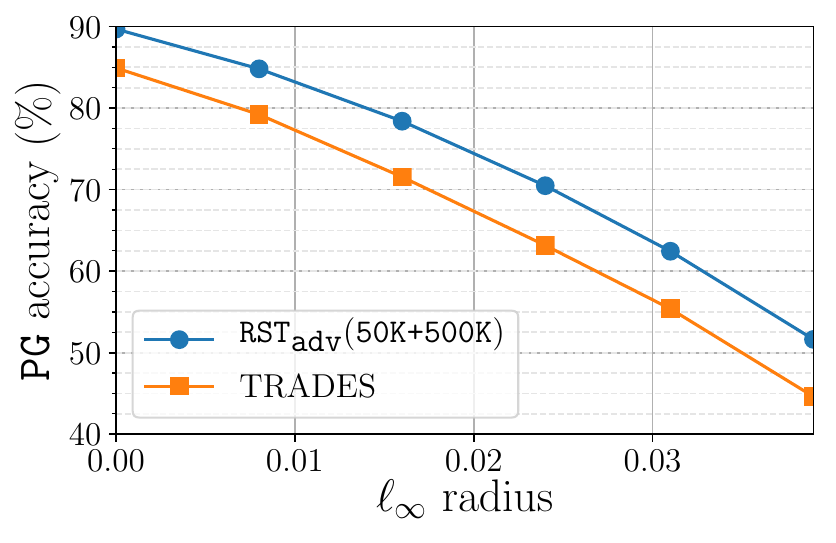}
      \caption{Comparing $\ell_\infty$ accuracy to tuned $\pgdno$ for various $\epsilon$. $\rstat{50K+500K}$ has higher accuracy than a state-of-the-art supervised model, across the range of $\epsilon$.}
      \label{fig:main-at} 
\label{fig:cifar-radii}
\end{figure}
%
\section{SVHN experiment details}\label{sec:app-svhn}

Here we give a detailed description on our SVHN experiments reported in \Cref{sec:svhn}.

\subsection{Experimental setup}
Most of our experimental setup for SVHN is the same as in CIFAR-10, except for the following differences. 

\begin{itemize}[leftmargin=12pt]
	\item \textbf{Architecture.} Throughout the SVHN experiments we use a slightly smaller \wrn{16}{8} model.
	
	\item \textbf{Robust self-training.} Since here the unlabeled data is 100\% relevant, we set $\wun=1$ by sampling every batch from a single pool containing the labeled and unlabeled data. We still use $\beta=6$ throughout; we performed brief tuning over $\beta$ to make sure that 6 is a reasonable value for SVHN.
	
	\item \textbf{Adversarial training.} We perform adversarial training with the same parameters as before and as in~\citep{zhang2019theoretically}, with one exception: we set the perturbation magnitude $\epsilon=4/255$ (we use step size 0.007 and 10 steps as in CIFAR-10). As we discuss below, adversarial training with the current configuration and $\epsilon=8/255$ produced weak and inconsistent results.
	
	\item \textbf{Stability training.} No change.
	
	\item \textbf{Input normalization.}  No change.
	
	\item \textbf{Data augmentation.}  Following~\cite{zagoruyko2016wide} we do not perform any data augmentation.
	
	\item \textbf{Optimizer configuration.} We use the same parameters as 
	before, except here we use batch size 128.
	
	\item \textbf{Number of gradient steps.}
	We run 98K gradient steps in every SVHN experiment. For stability training we attempted to double the number of gradient steps and did not observe improved results. 
\end{itemize}

\subsection{Pseudo-label generation and standard self-training.} 
To generate pseudo-labels we perform standard training of an \wrn{16}{8} model as described above on the core SVHN training set only. This model attains 96.6\% test accuracy, and we use it to generate all the pseudo-labels in our SVHN experiments. 

For comparison we repeat this procedure on the entire SVHN data (with all the labels). The resulting model has 98.2\% test accuracy. Finally, we apply standard self-training using the configuration described above, i.e. we replace the SVHN extra labels with the pseudo-labels---this corrupts 1.6\% of the extra labels (the extra data is easier to classify than the test set). Self-training produces a model with 97.1\% accuracy, similar to the 0.4\% improvement we observed on CIFAR-10, and 1.1\% short of using true labels.

\subsection{Evaluation and attack details}

We perform randomizes smoothing certification exactly as in the CIFAR-10 experiments in \Cref{sec:benefit}. 
For evaluating heuristic defenses, we fine-tune the PG attack to maximally break $\rstat{73K+531K}$ to obtain $\pgdours$ with the following parameters: step-size $\eta=0.005$, number of steps $\tau=100$ and number of restarts $\rho=10$. We evaluate models at $\epsilon=4/255$, which is the same as the value we used during training. 

Compared to CIFAR-10, we find that we require larger number of steps for 
SVHN attacks. Interestingly, for $\epsilon=8/255$ (which is what we 
evaluate our CIFAR-10 models on), we find that even after $1000$ steps, we 
see a steady decrease in accuracy (when evaluated over a small random 
subset of the test set). However, for a smaller value of $\eps=4/255$ 
(which is what we finally report on), we see that the accuracies seem to 
saturate after $100$ steps of the attack. 

\newcolumntype{C}[1]{>{\Centering\arraybackslash}p{#1\linewidth}}
\newcolumntype{L}[1]{>{\RaggedRight\arraybackslash}p{#1\linewidth}}
\newcolumntype{R}[1]{>{\RaggedLeft\arraybackslash}p{#1\linewidth}}
\newcommand\mR[2]{\multicolumn{1}{R{#1}}{#2}} %

\begin{table}
	\centering
	\begin{tabular}{L{0.22}L{0.08}L{0.08}L{0.08}L{0.08}}
		\toprule
		\mR{0.3}{$\ell_2$ radius:} &  0 &   0.22 &   0.435 &   0.555 \\
		\mR{0.3}{Enclosed $\ell_\infty$ radius:} &  0 &   1/255 &   2/255 &   
		0.01 \\
		Model &&&&  \\ 
		\midrule
		$\baseline{stab}{604K}$ &  93.6 &  84.9 &  70.0 &  59.8 \\
		$\rstst{73K+531K}$  &  93.2 &  84.5 &  69.7 &  59.5 \\
		$\baseline{stab}{73K}$ &  90.1 &  80.2 &  65.0 &  55.0 \\
		\bottomrule
	\end{tabular}

	\vspace{6pt}
	\caption{SVHN certified test accuracy (\%) for different $\ell_2$ 
	perturbations radii, and the $\ell_\infty$ certified robustness they imply.}
	\label{table:svhn-additional}
\end{table}

\subsection{Comparison with results in the literature}
In the context of adversarial robustness, SVHN was not studied extensively in the literature, and in most cases there are no clear benchmarks to compare to. Unlike CIFAR-10, there is no agreed-upon benchmark perturbation radius for heuristic $\ell_\infty$ defenses. Moreover, we are not aware of a heuristic SVHN defense that withstood significant scrutiny. A previous heuristic defense~\citep{kannan2018adversarial} against attacks with $\epsilon=12/255$ was subsequently broken~\citep{engstrom2018evaluating}. In~\citep{schmidt2018adversarially}, the authors study SVHN attacks with $\epsilon=4/255$ but do not tabulate their results. Visual inspection of their figures indicates that we get better robust accuracies than~\citep{schmidt2018adversarially} by over 7\%, likely  due to using a higher capacity model, better training objectives and better hyperparameters.

Two recent works constitute the state-of-the-art for certified robustness in 
SVHN. \citet{cohen2019certified} study randomized smoothing certification 
of $\ell_2$ robustness and report some results for SVHN, but do not 
tabulate them and did not release a model. Their figure shows a sharp 
cutoff at radius 0.3, suggesting a different input normalization than the one 
we used. In view of our comparison in \Cref{sec:app-compare-stab}, it 
seems likely that our model attains higher certified accuracy. 
\citet{gowal2018effectiveness} propose interval bound propagation for 
certified $\ell_\infty$ robustness. They report a model with 85.2\% 
standard accuracy and 62.4\% certified robust accuracy against 
$\ell_\infty$ attacks with $\epsilon=0.01$.

In \Cref{table:svhn-additional} we list selected point from \Cref{fig:svhn}, showing certified accuracy as a function of $\ell_2$ perturbation radius. For each $\ell_2$ radius we also list the radius of the largest $\ell_\infty$ ball contained within in it, allowing comparison between our results and~\citep{gowal2018effectiveness}. At the $\ell_2$ radius that contains an $\ell_\infty$ ball of radius $0.01$ we certify accuracy of 59.8\%, less than 3\% below than the result of~\citet{gowal2018effectiveness}. This number is likely easy to improve by tuning $\sigma$ and $\beta$ used in stability training, situating it as a viable alternative to interval bound propagation in SVHN as well as CIFAR-10.

\section{Comparison to \citet{uesato2019are}}
\label{sec:app-comparison}

Independently from our work, \citet{uesato2019are} also study semisupervised adversarial learning theoretically in the Gaussian model of~\citep{schmidt2018adversarially} and empirically via experiments on CIFAR-10, Tiny Images, and SVHN.
Overall, \citet{uesato2019are} reach conclusions similar to ours.
Here, we summarize the main differences between our works. 

We can understand the algorithms that~\citep{uesato2019are} propose as instances of \Cref{alg:robustself} with different choices of $\lossrob$.
In particular, their most successful algorithm (UAT++) corresponds to
\begin{equation*}
\lossrob^{\text{UAT++}}(\theta, x, y) = 
\max_{\xadv \in\ballp(x)} \loss(\theta, \xadv, y) + \lambda \lossreg(\theta, x),
\end{equation*}
where $\lossreg(\theta, x)=\max_{\xadv \in\ballp(x)} \Dkl{p_\theta(\cdot \mid x)}{p_\theta(\cdot \mid \xadv)}$ as in~\eqref{eq:trades}. In contrast, we do not maximize over $\loss$, i.e.\ we use 
\begin{equation*}
\lossrob(\theta, x, y) = 
\loss(\theta, x, y) + \lambda \lossreg(\theta, x).
\end{equation*}

\newcommand{\swrn}[2]{WRN-#1-#2}

Both this work and~\citep{uesato2019are} perform adversarial training on CIFAR-10 with additional unlabeled data from Tiny Images. Both works consider the benchmark of $\ell_\infty$ perturbations with radius $\epsilon=8/255$ and report results against a range of attacks, which in both papers includes $\pgdtrades$~\cite{zhang2019theoretically}. 
For this attack, our best-performing models have robust accuracies within 1\% of each other (their \swrn{106}{8} is 1.1\% higher than our \swrn{28}{10}, and their \swrn{34}{8} is 1\% lower), and we obtain about 3\% higher standard accuracy. 

Beyond the algorithmic and model size difference, \citet{uesato2019are} use different training hyperparameters, most notably larger batch sizes. Additionally, to source unlabeled data from 80 Million Tiny images they use a combination of keyword filtering and predictions from a standard CIFAR-10 model. In contrast, we do not use keywords at all and train a classifier to distinguish CIFAR-10 from Tiny Images. Both works remove the CIFAR-10 dataset prior to selecting images from TI; we also remove an $\ell_2$ ball around the CIFAR-10 test set. Due to these multiple differences and similar final accuracies, we cannot determine which robust loss provides better performance.

\citet{uesato2019are} perform a number of experiments that complement ours. First, they show strong improvements in the low labeled data regimes by removing most labels from CIFAR-10 and SVHN (similar findings appear also in~\citep{zhai2019adversarially,najafi2019robustness}). Second, they demonstrate that their method is tolerant to inaccurate pseudo-labels via a controlled study. Finally, they propose a new ``MultiTargeted'' attack that reduces the reported accuracies of the state-of-the-art robust models by 3-8\%. Contributions unique to our work include showing that unlabeled data improves certified robustness via randomized smoothing and studying the effect of irrelevant data theoretically and experimentally.

\end{document}